\documentclass[letterpaper]{article}
\usepackage[utf8]{inputenc}
\usepackage[preprint]{aaai2027}  

\usepackage[table]{xcolor}
\usepackage[hyphens]{url}  
\usepackage{graphicx} 
\usepackage{natbib}  
\usepackage{caption} 
\usepackage{amsmath}
\usepackage{amssymb}
\usepackage{newunicodechar}
\newunicodechar{（}{(}
\newunicodechar{）}{)}
\usepackage{multirow}
\usepackage{booktabs}
\usepackage{algorithm}
\usepackage{algpseudocode}
\usepackage{booktabs}
\usepackage{threeparttable}

\usepackage{tabularx}
\usepackage{amsthm}
\usepackage[most]{tcolorbox}
\usepackage{enumitem}

\theoremstyle{remark}
\newtheorem{remark}{Remark}
\newtheorem{definition}{Definition}
\newtheorem{theorem}{Theorem}
\newtheorem{corollary}{Corollary}
\newtheorem{lemma}{Lemma}
\newtheorem{proposition}{Proposition}

\definecolor{PromptBlue}{RGB}{25,70,180}
\definecolor{PromptPurple}{RGB}{120,35,145}
\definecolor{PromptRed}{RGB}{190,35,35}
\definecolor{PromptGreen}{RGB}{0,110,70}
\definecolor{PromptFrame}{RGB}{75,75,75}
\newtcolorbox{promptbox}[1][]{
  enhanced,
  breakable,
  colback=black!2,
  colframe=PromptFrame,
  boxrule=0.8pt,
  arc=0pt,
  left=5mm,
  right=5mm,
  top=3mm,
  bottom=3mm,
  fontupper=\small,
  #1
}
\newcommand{\promptrole}[1]{\textcolor{PromptBlue}{#1}}
\newcommand{\promptinput}[1]{\textcolor{PromptPurple}{\texttt{#1}}}
\newcommand{\promptconstraint}[1]{\textcolor{PromptRed}{#1}}
\newcommand{\promptoutput}[1]{\textcolor{PromptGreen}{#1}}

\newcommand{\promptfield}[2]{\textbf{#1}\newline\noindent#2\par\smallskip}

\makeatletter
\renewcommand\subsubsection{%
  \@startsection{subsubsection}{3}{\z@}%
  {-2.0ex plus -0.5ex minus -.2ex}%
  {3pt plus 2pt minus 1pt}%
  {\normalsize\bfseries\raggedright}%
}
\makeatother

\usepackage{newfloat}
\usepackage{listings}
\usepackage{soul}
\DeclareCaptionStyle{ruled}{labelfont=normalfont,labelsep=colon,strut=off} 
\floatstyle{ruled}
\newfloat{listing}{tb}{lst}{}
\floatname{listing}{Listing}

\title{DGA$_2$D: Directed Graph-Guided Automated Algorithm Design with Large Language Models}

\author {
    Jiale Zhao\equalcontrib\textsuperscript{\rm 1},
    Zimu Chen\equalcontrib\textsuperscript{\rm 1},
    Sirui Mao\textsuperscript{\rm 2},
    Wentao Yang\textsuperscript{\rm 2},
    Yuxiang Bai\textsuperscript{\rm 1},
    Liyuanjun Lai\textsuperscript{\rm 1}\corresponding
}
\affiliations {
    \textsuperscript{\rm 1}School of Automation Science and Electrical Engineering, Beihang University, Beijing, China\\
    \textsuperscript{\rm 2}School of Computer Science and Engineering, Beihang University, Beijing, China\\
    \{lailiyuanjun\}@buaa.edu.cn
}

\begin{document}

\maketitle

\begin{abstract}
The rapid development of Large Language Models (LLMs) has opened new avenues for Automated Heuristic Design (AHD) for solving NP-hard combinatorial optimization problems (COPs). 
However, existing LLM-driven AHD methods are largely confined to rigid solver templates, relegating the search process to isolated module tuning. 
Transitioning to fully autonomous, system-level algorithm design is essential but fraught with low reliability of generated operators, extremely large search spaces, and ineffective credit assignment. 
To overcome these drawbacks, this paper proposes a Directed Graph-Guided Automated Algorithm Design framework, termed $\text{DGA}_2\text{D}$. It structures the open-ended program space as a directed graph, where each node represents a functional operator that can be instantiated using one of multiple candidate code implementations, while directed walks constitute complete algorithmic pipelines. 
A first-order path-dependent credit assignment mechanism is introduced to evaluate code variations strictly based on their topological context. 
Extensive experiments across 12 distinct COPs—ranging from complex scheduling to routing—demonstrate the consistent empirical advantages of $\text{DGA}_2\text{D}$. 
It reduces the average normalized gap by up to 10.96 percentage points compared to state-of-the-art LLM baselines. The source code is available at https://github.com/BeihangAILab/DGA2D.
\end{abstract}


\section{1 Introduction}

Numerous practical combinatorial optimization problems (COPs), such as flexible job-shop scheduling and open vehicle routing, are NP-hard. Unless $\mathrm{P}=\mathrm{NP}$, no polynomial-time exact algorithms are known for these problems. 
As these problems are key determinants of the efficiency of many complex systems, different kinds of algorithms, such as mathematical programming algorithms, evolutionary algorithms, and reinforcement learning algorithms are designed to find high-quality feasible solutions in a limited time.

The no-free-lunch theorems suggest that no optimization algorithm is universally superior across all possible problems, motivating the development of methods that exploit problem-specific structure~\cite{wolpert2002no}. Despite the development of numerous problem-specific optimization algorithms, designing effective and extensible algorithms remains
difficult and labor-intensive. Moreover, the optima of many classical benchmark problems remain unknown.

\begin{figure}[t]
    \centering 
    \includegraphics[width=1.0\columnwidth]{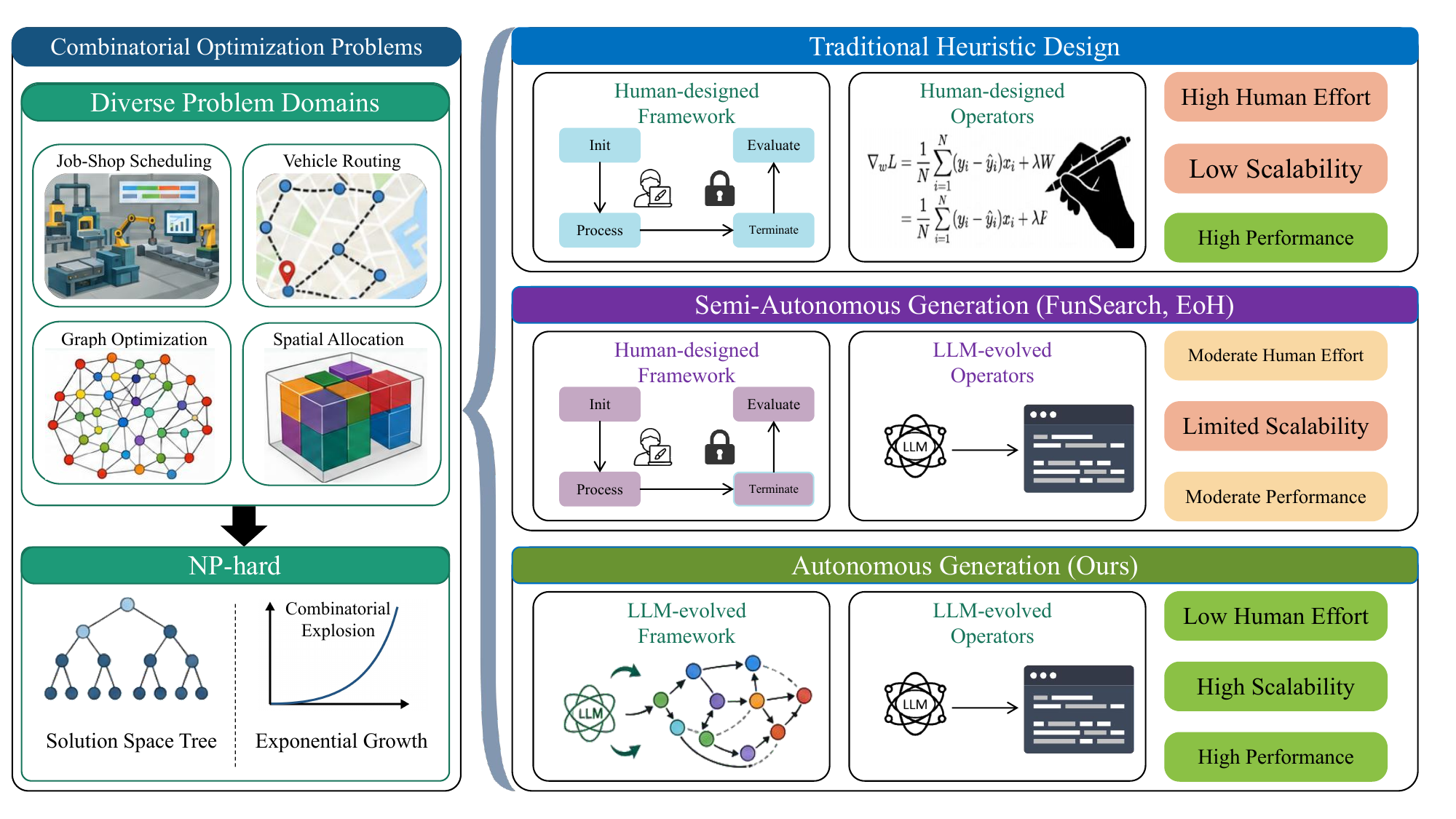}
    \caption{Evolution of Automated Heuristic Design (AHD). \textbf{Left:} NP-hard problem domains and exponentially growing solution spaces. \textbf{Right:} Comparison of AHD paradigms. Unlike \textbf{Traditional Heuristic Design} and \textbf{Semi-Autonomous Generation} approaches, our \textbf{Autonomous Generation} evolves both the complete algorithmic structure and operators.}
\label{fig:framework_comparison}
\end{figure}

To achieve a more efficient and convenient heuristic algorithm design process, researchers proposed the concept of AHD, which soon became an active research area \cite{burke2013hyper}. 
The efforts devoted to AHD can mainly be classified into three routes: heuristic selection, heuristic configuration, and heuristic composition. 
Despite substantial progress, traditional heuristic-design approaches still rely heavily on code modules crafted by human experts \cite{drake2020recent}. As summarized in Figure 1, they typically require \textbf{high human effort} and exhibit \textbf{low scalability}, although they can achieve \textbf{high performance} on specific problems.

The rapid development of Large Language Models (LLMs) presents a new opportunity for AHD \cite{liu2026systematic}. 
LLMs have enabled fully automated generation and refinement of heuristics, along with their associated code implementations. Representative methods include FunSearch \cite{romera2024mathematical} and EoH \cite{liu2024evolution}.
Attempts have also been made to couple LLMs with Evolutionary Computation (EC) to produce heuristics in a more automated manner \cite{liu2024evolution, dat2025hsevo, chen2026a2dept}. However, existing LLM-driven AHD methods operate primarily within rigid solver templates. As illustrated in Figure 1, such semi-autonomous approaches generally require \textbf{moderate human effort}, offer \textbf{limited scalability}, and achieve \textbf{moderate performance}.

To overcome the above structural bottlenecks, a natural progression is to use LLMs as both the framework architect and the operator designer. 
As illustrated in Figure 1, allowing the LLM to redesign both the algorithmic framework and its operators opens a richer design space, with the goal of achieving \textbf{low human effort}, \textbf{high scalability}, and \textbf{high performance}.

However, expanding the evolution target from isolated heuristics to complex algorithmic architectures composed of multiple operators introduces three fundamental challenges. 

\begin{enumerate}
\item \textbf{Low reliability of generated algorithms.}
Complete algorithms consist of long, interdependent operator
sequences and are therefore prone to compilation errors,
interface mismatches, and logical failures.

\item \textbf{Extremely large search spaces.}
Valid and high-performing algorithms are sparse in open-ended
program spaces, making efficient exploration difficult.

\item \textbf{Ineffective credit assignment.}
Because performance is observed only after evaluating complete
algorithms, determining which local modifications are responsible
for an improvement is difficult.
\end{enumerate}

To overcome the above challenges, this paper proposes a directed graph-guided automated algorithm design framework, termed DGA\textsubscript{2}D, for open-ended algorithm design that optimizes complete algorithms with multiple interconnected operators rather than isolated heuristics. 
The main contributions of this paper are as follows.

\noindent 1) It structures the algorithm as a Directed Graph (DG), where individual nodes represent operators, each associated with multiple candidate implementations, while directed walks constitute complete algorithmic pipelines. 
The LLM drives this evolution through three straightforward actions. 
First, it edits the actual code inside each step. 
Second, it rearranges how different steps connect to form new topological paths. 
Third, it picks the most suitable code variation for the specific path being taken. 

\noindent 2) A first-order path-dependent credit assignment mechanism is integrated into the algorithm generation process. 
Credit is assigned to each code implementation according to  its context within the evolutionary process, enabling a quantitative assessment of each code implementation’s contribution to the final algorithm. 

\noindent 3) Comprehensive experiments are conducted across 12 representative COPs, with comparisons against classical solvers, problem-specific learning-based methods, and state-of-the-art LLM-driven AHD frameworks. The results show that DGA\textsubscript{2}D consistently outperforms the LLM-driven baselines across four representative benchmarks, reducing the average normalized gap by 9.67 and 10.96 percentage points under DeepSeek-V4-Flash and GPT-5.6 sol, respectively.

\section{2 Related Work}
\subsection{2.1 Heuristics for Combinatorial Optimization}
Existing heuristic approaches for COPs can be broadly summarized into three research lines. Evolutionary algorithms use population-based random search to explore large solution spaces \cite{huang2023evolutionary,wang2023sustainable}. Reinforcement-learning-based methods learn adaptive policies through interaction with optimization environments \cite{ye2023deepaco,zhang2024deep}. 
Machine-learning and deep-learning-based methods use learned models, such as Transformers and diffusion models, to construct or refine solutions \cite{sun2023difusco,li2024fast}. These approaches depend heavily on manually designed structures, operators, and learning objectives, which limits their applicability to specific problem classes \cite{bengio2021machine}.

\subsection{2.2 Automated Heuristic Design}
Driven by the demand for more efficient algorithm design, AHD has emerged as an active research area \cite{burke2013hyper}. 
Classical AHD is largely studied through hyper-heuristics, which generally follow three paradigms: heuristic configuration, heuristic selection, and heuristic composition. 
Heuristic configuration tunes parameters or control policies within a predefined algorithmic framework \cite{brandt2023ac,guo2025configx}. Heuristic selection adaptively chooses low-level heuristics from a fixed portfolio, increasingly using reinforcement learning and combined offline--online experience \cite{zhao2023selection,pei2024learning}. Heuristic composition constructs new procedures with data-driven component selection and modular metaheuristic generation \cite{dokeroglu2024hyper,xue2025automated,contreras2025automated}. Despite reducing manual design effort, these AHD approaches remain constrained by human-defined operators and search spaces \cite{drake2020recent}.

\subsection{2.3 Automated Heuristic Design Using LLMs}
Recently, LLMs have been applied to AHD, leveraging their zero-shot capabilities in code generation to synthesize algorithmic heuristics from natural language \cite{liu2026systematic}. 
Early studies explored LLMs for automated heuristic and algorithm generation \cite{liu2023algorithm}. Related work also demonstrated LLM-driven optimization in reward design and black-box optimization \cite{ma2024eureka,yang2024large}. Standalone prompt-driven approaches often suffer from limited long-context reasoning and lack systematic exploration mechanisms to escape local optima.

To overcome these limitations, recent work has coupled LLM-based generation with iterative search, including population-based evolution \cite{romera2024mathematical,liu2024evolution,van2024llamea}, reflection-guided refinement \cite{ye2024reevo}, multi-objective population management \cite{dat2025hsevo,yao2025multi}, and tree-structured exploration \cite{zheng2025monte}.

In summary, existing methods optimize heuristics within predefined frameworks, which improves reliability but limits their performance. To move beyond these rigid frameworks, recent work explores LLMs as system-level algorithm architects \cite{chen2026a2dept}; however, reliably generating complete and consistently competitive algorithms remains challenging.

\section{3 Methodology}
\subsection{3.1 Overview of the DGA\textsubscript{2}D Framework}
The overall framework of DGA\textsubscript{2}D is illustrated in Figure 2. It represents the algorithm search space as a DG in which both the topology and the operator implementations are iteratively refined through a closed-loop LLM-driven process.

The algorithm generation process consists of four parts: LLM-based problem initialization, graph-guided candidate algorithm construction, evaluation and credit assignment, and credit-guided dual-level evolution, as detailed in the following subsections.

\subsection{3.2 LLM-based Initialization}
\begin{figure*}[t]
    \centering
    \includegraphics[width=1.0\textwidth]{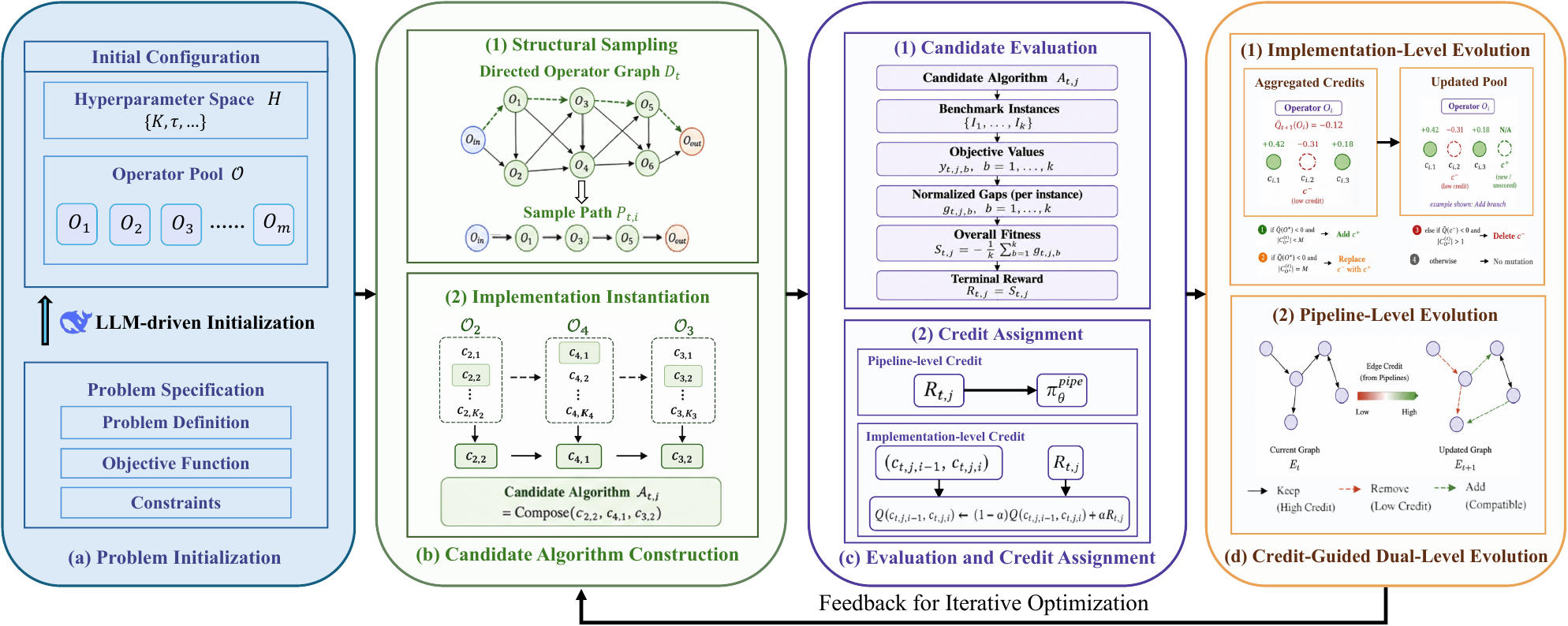}
    \caption{Overview of the DGA$_2$D framework.
(a) LLM-driven initialization of problem configuration and operator pool.
(b) Candidate algorithm construction by sampling operator structures and implementations.
(c) Hierarchical credit assignment based on candidate algorithm performance.
(d) Dual-level evolution of operator implementations and algorithmic pipelines guided by aggregated credits.}
    \label{fig:framework}
\end{figure*}

\label{sec:problem_init}
As many algorithms are highly sensitive to specific problem domains, this paper introduces an LLM-driven initialization phase to take the natural language description of target problem as input, setting up the initial hyperparameter space and the foundational operator pool for algorithm evolution, as demonstrated in Figure~2a.

The initial hyperparameter space is defined as $\mathcal{H}=\{K,\tau,\dots\}$, where $K\in\mathbb{N}^+$ specifies the maximum number of functional operators maintained in the operator pool, and $\tau\in(0,1]$ is the sampling temperature used to balance exploration and exploitation.

The foundational operator pool $\mathcal{O}=\{O_1,\dots,O_m\}$ consists of reusable operators selected by the LLM from a fixed repository constructed from existing algorithms and domain-specific literature. Each operator $O_i\in\mathcal{O}$ is associated with multiple code implementations that realize fine-grained heuristic actions, such as local-search moves and greedy assignment rules. Together, $\mathcal{H}$ and $\mathcal{O}$ initialize the evolutionary search in a problem-relevant region of the algorithm space.

\subsection{3.3 Graph-Guided Algorithm Construction}
\label{sec:workflow}

As shown in Figure~2b, the proposed DGA\textsubscript{2}D evolves both operator
implementations and operator connectivity. The resulting DG defines a
shared search space in which a fixed number of candidate
algorithms are generated and subsequently evaluated according to the procedure defined in
Section~3.4.

\subsubsection{Directed-Graph Search-Space Representation.} 

Suppose the maximum generation of the DGA\textsubscript{2}D process be $t_{max}$. 
At each generation $t\in[1, t_{max}]$, the search space is modeled as a directed graph $\mathcal{D}_t=(\mathcal{O}_t,\mathcal{E}_t)$. 
Each node ${O}_i\in\mathcal{O}_t$ represents a functional operator. A directed edge $(O_i,O_j)\in\mathcal{E}_t$ indicates that
$O_j$ is executed after $O_i$.

An operator \(O_i\) is associated with a pool of alternative code implementations,
denoted by
\[C_i^{(t)}=\{c_{i,1},c_{i,2},\ldots,c_{i,|C_i^{(t)}|}\}.\]

Each operator can be formulated by one implementation in \(C_i^{(t)}\).
The number of candidate implementations maintained by each operator is bounded by \(M\), i.e.,\[|C_i^{(t)}|\leq M.\]

Each code implementation belongs to one and only one operator, so the implementation pools are disjoint:
$$
\mathcal{C}_i^{(t)} \cap \mathcal{C}_{i'}^{(t)} = \varnothing, \quad i \neq i'.
$$

\subsubsection{Candidate Algorithm Construction.} 
For each generation $t\in [1,t_{\max}]$, a candidate algorithm is obtained in two stages: pipeline sampling and algorithm instantiation. First, the algorithmic structure is sampled by a bounded directed walk, which is defined as a pipeline $P_{t,j}$.

\begin{equation}
P_{t,j}
=
\left(
O_{\mathrm{in}},
O_{t,j,1},
\ldots,
O_{t,j,L_{t,j}},
O_{\mathrm{out}}
\right),
\label{eq:pipeline}
\end{equation}
where $O_{\mathrm{in}}$ and $O_{\mathrm{out}}$ are sentinel nodes
representing the input and output of the pipeline, respectively. They do
not correspond to functional operators and therefore require no
code implementations. A valid pipeline satisfies
\begin{equation}
\begin{aligned}
&\left(O_{\mathrm{in}},O_{t,j,1}\right)\in\mathcal{E}_t,\\
&\left(O_{t,j,i},O_{t,j,i+1}\right)\in\mathcal{E}_t,
\quad i=1,\ldots,L_{t,j}-1,\\
&\left(O_{t,j,L_{t,j}},O_{\mathrm{out}}\right)\in\mathcal{E}_t,\\
&1\leq L_{t,j}\leq L_{\max}.
\end{aligned}
\label{eq:pipeline_constraints}
\end{equation}

It is noted that repeated operators are permitted in a pipeline to enable iterative algorithmic structure. 
At each generation, in total of $N$ pipelines $\mathcal{P}_t=\{P_{t,1},\ldots,P_{t,N}\}$ are sampled according to a credit-guided policy introduced in Section 3.4. 

Each pipeline is instantiated as a candidate algorithm by selecting a code implementation for each operator involved in the pipeline.
Specifically, for each operator $O_{t,j,i}$, one implementation \(\hat{c}_{t,j,i}\) is sampled according to an implementation-selection probability \(p_t\), which is defined in Section~3.4, too.

Then, a candidate algorithm is represented as a sequence of code implementations $A_{t,j}=\left( \hat{c}_{t,j,1}, \ldots, \hat{c}_{t,j,L_{t,j}} \right)$,
where $L_{t,j}$ denotes the length of the sampled pipeline.
Repeating this procedure for $j=1,\ldots,N$ yields a population of algorithms
$\mathcal{A}_t=\left\{ A_{t,1},\ldots,A_{t,N} \right\}$.

\subsection{3.4 Problem Evaluation and Credit Assignment}

To support the pipeline sampling and algorithm instantiation, a credit assignment mechanism is defined for algorithm-level, pipeline-level, and code-implementation-level evaluation.

\paragraph{Algorithm-level evaluation.}

An algorithm $A_{t,j}$ is evaluated on $k$ benchmark instances. For an instance $b\in[1, k]$, let $y_{t,j,b}$ denote the objective value of the solution obtained by $A_{t,j}$, and $y_b^{\mathrm{ref}}$ denote the reference value, which can be either the optimum or the best known objective value of the instance. 
A normalization scale is defined as $s_b = \max\left\{\left|y_b^\mathrm{ref}\right|,1\right\}$ to prevent numerical instability when the reference value is
zero or close to zero.

Then, the direction-aware normalized gap is defined as
\begin{equation}
g_{t,j,b}
=
\begin{cases}
\displaystyle
\frac{y_{t,j,b}-y_b^{\mathrm{ref}}}{s_b},
& \text{for minimization},\\[9pt]
\displaystyle
\frac{y_b^{\mathrm{ref}}-y_{t,j,b}}{s_b},
& \text{for maximization}.
\end{cases}
\label{eq:directional_gap}
\end{equation}
where \(g_{t,j,b}>0\) indicates a worse performance than the reference, whereas \(g_{t,j,b}<0\) indicates an improvement upon the reference.

Suppose that all benchmark instances are equally important, the
overall fitness of an algorithm is defined as
\begin{equation}
S_{t,j}
=
-\frac{1}{k}
\sum_{b=1}^{k}
g_{t,j,b}.
\label{eq:candidate_fitness}
\end{equation}

A larger $S_{t,j}$ indicates better overall performance.

\paragraph{Terminal reward.}
As \(S_{t,j}\) is a normalized value and follows a larger-is-better
convention, it is applied as the terminal reward of an algorithm:
\begin{equation}
\widetilde{R}_{t,j}=\frac{S_{t,j}-\mu_t}{\sigma_t+\varepsilon},\label{eq:terminal_reward}
\end{equation}
where \(\mu_t\) and \(\sigma_t\) are the mean and standard deviation of
the rewards obtained in generation \(t\), and $\varepsilon > 0$ is a small constant introduced for numerical stability. Thus, \(\widetilde{R}_{t,j}>0\) indicates the algorithm performs above-average, whereas
\(\widetilde{R}_{t,j}<0\) indicates it performs below-average.

\paragraph{Pipeline-level credit assignment.}
The standardized terminal reward
$\widetilde{R}_{t,j}$ is then assigned to the sampled pipeline as its credit and used for updating the pipeline sampling policy (defined as $\pi_{\theta}^{\mathrm{pipe}}$):

\begin{equation}
\mathcal{L}_{\mathrm{pipe}}(\theta)
=
-\frac{1}{|\mathcal{P}_t|}
\sum_{j\in[1,N]}
\widetilde{R}_{t,j}
\log
\pi_{\theta}^{\mathrm{pipe}}
\left(P_{t,j}\mid \mathcal{D}_t\right).
\end{equation}

Minimizing this loss encourages
higher sampling probabilities for pipelines associated with positive
pipeline-credit values. The factorization and optimization details of
$\pi_{\theta}^{\mathrm{pipe}}$ are provided in Appendix~D.1.

\paragraph{Code-implementation-level credit assignment.}

The utility of an algorithm $A_{t,j} = (c_{t,j,1}, \ldots, c_{t,j,L_{t,j}})$ largely depends on its code implementation. 
Suppose $c_{t,j,0}=c_{\mathrm{start}}$ and the \(i\)-th activated
transition is $e_{t,j,i} = \left( c_{t,j,i-1}, c_{t,j,i} \right)$. We define the credit of a first-order implementation transition $e$ as
$Q_t(e)$, which estimates the expected standardized terminal reward
associated with that transition. Initially, $Q_0(e)=0$ for unseen transitions. 

For every transition activated in the sampled code implementations,
its credit is updated using the standardized terminal reward $\widetilde{R}_{t,j}$:
\begin{equation}
Q_{t+1}(e_{t,j,i})
=
(1-\alpha_t)Q_t(e_{t,j,i})
+
\alpha_t\widetilde{R}_{t,j},
\quad
i=1,\ldots,L_{t,j},
\end{equation}
where \(\alpha_t\in(0,1)\) is the learning rate.

Therefore, we introduce a first-order Markov approximation in defining the code implementation sampling policy as \[
\pi_t^{\mathrm{code}}\!\left(c_{t,j,i} \mid c_{t,j,i-1}\right)
=
\operatorname{Softmax}\!\left(
\frac{Q_t\!\left(e_{t,j,i}\right)}{\tau}
\right).
\] The sampling of $c_{t,j,i}$ is only determined according to its immediate predecessor $c_{t,j,i-1}$,  $\tau>0$ is the temperature parameter which controls the
sharpness of the code-implementation sampling distribution.


\subsection{3.5 Credit-Guided Dual-Level Evolution}

Suppose all of the activated
transitions that involve a code implementation $c$ selected in generation $t$ formulates a transition set $\mathcal{E}$, the empirical credit of $c$ is defined as
\begin{equation}
\bar{Q}_{t}(c)=
\frac{\sum_{e\in\mathcal{E}} n_t(e)Q_{t}(e)}
{\sum_{e\in\mathcal{E}} n_t(e)+\epsilon},
\end{equation}
where $e$ is a transition in $\mathcal{E}$ and $n_t(e)$ is the activation count of this transition and $\epsilon>0$ is a small constant used to avoid division by zero.
Similarly, the empirical credit of operator $O_i$ is defined as
\begin{equation}
\bar{Q}_{t}(O_i)=
\frac{\sum_{c\in\mathcal{C}_i^{(t)}}
\bar{Q}_{t}(c)}
{|\mathcal{C}_i^{(t)}|},
\end{equation}
where $c$ denotes all possible code implementation that can be selected for $O_i$.

The target operator is selected as
\begin{equation}
O^\star=\arg\min_{O_i\in\mathcal{O}}\bar{Q}_{t}(O_i),
\end{equation}

Defines the lowest-credit implementation of $O^\star$ as $c^-$, which can be calculated by
\begin{equation}
c^-=\arg\min_{c\in\mathcal{C}_{O^\star}^{(t)}}
\bar{Q}_{t}(c).
\end{equation}

If $\bar{Q}_{t}(O^{\star}) \geq 0$ and $|\mathcal{C}_{O^\star}^{(t)}| > 1$, $c^-$ is deleted.
If $\bar{Q}_{t}(O^{\star})<0$ and $|\mathcal{C}_{O^\star}^{(t)}| < M$, a newly generated implementation
$c^{+}$ is added to the candidate pool of $O^{\star}$.
If $\bar{Q}_{t}(O^{\star})<0$ and $|\mathcal{C}_{O^\star}^{(t)}| = M$, $c^{-}$ is replaced by $c^{+}$.

Similarly, suppose a transition between a candidate code $u\in\mathcal{C}_i^{(t)}$ for $O_i$ and a candidate code $c\in\mathcal{C}_j^{(t)}$ for $O_j$ as $e=(u,c)$, the transition credit of $O_i$ and $O_j$ can be 
\begin{equation}
\bar{Q}_{t}(O_i,O_j)
=
\frac{
\sum_{u\in\mathcal{C}_i^{(t)}}
\sum_{c\in\mathcal{C}_j^{(t)}}
n_t(e)Q_{t}(e)
}{
\sum_{u\in\mathcal{C}_i^{(t)}}
\sum_{c\in\mathcal{C}_j^{(t)}}
n_t(e)+\epsilon}.
\end{equation}

At each generation, we identify the removable edge with the
lowest aggregated credit:
\begin{equation}
(O_i^-,O_j^-)
=
\arg\min_{(O_i,O_j)\in\mathcal{E}_t}
\bar{Q}_t(O_i,O_j).
\end{equation}

If $\bar{Q}_{t}(O_i^-,O_j^-)\geq 0$, all existing edges are retained. Otherwise, the edge $(O_i^-,O_j^-)$ is removed and replaced by a new
edge, which is randomly generated by the LLM. The new edge should be currently absent and interface-compatible, so as to ensure a valid pipeline.
If no valid edge exists, the DG remains unchanged.

In summary, contextual credits guide implementation sampling, while their aggregated statistics drive both implementation-level and pipeline-level evolution.

\section{4 Experiments}
\label{sec:experiments}

\subsection{4.1 Experimental Setup}
\label{subsec:setup}
\begin{table*}[t!]
\centering
\caption{Performance of $\text{DGA}_2\text{D}$ compared with baselines across combinatorial optimization benchmarks, grouped by LLM backbones. Obj.: Objective value ($\downarrow$ lower is better, $\uparrow$ higher is better). Gap (\%): normalized gap to the corresponding reference solution. Results are reported as mean $\pm$ standard deviation (10 trials). \textbf{Bold} signifies the best overall result and the best LLM-based result under each backbone.\colorbox{gray!20}{Gray} indicates the best performance under the same LLM API.
}
\label{tab:main_results}
\small
\setlength{\tabcolsep}{4pt}
\begin{tabular}{lcccccccc}
\toprule
\multirow{2}{*}{\textbf{Method}} & \multicolumn{2}{c}{\textbf{TSP}} & \multicolumn{2}{c}{\textbf{FJSP}} & \multicolumn{2}{c}{\textbf{MIS}} & \multicolumn{2}{c}{\textbf{3D-CLP}} \\
\cmidrule(lr){2-3} \cmidrule(lr){4-5} \cmidrule(lr){6-7} \cmidrule(lr){8-9}
 & Obj. ($10^6$) $\downarrow$ & Gap (\%) ± std $\downarrow$ & Obj.($10^5$) $\downarrow$ & Gap (\%) ± std $\downarrow$ & Obj.($10^3$) $\uparrow$  & Gap (\%) ± std $\downarrow$  & Obj.($10^1$) $\uparrow$ & Gap (\%) ± std $\downarrow$\\
\midrule
\textit{\textbf{Baselines}} & & & & & & & & \\
Reference & 1.237 & - & 1.009 & - & 3.244 & - & 2.696 & - \\
ILS & \textbf{1.240} & \textbf{0.48} & 1.123 & 11.31 & 2.997 & 8.41 & 2.577 & 5.51 \\
Neural* & 1.301 & 5.13 & 1.935 & 71.47 & 2.695 & 17.97 & 2.571 & 5.82 \\
OR-Tools & 1.247 & 0.76 & 1.076 & 8.82 & 3.044 & 7.45 & 2.563 & 6.10 \\
\midrule
\multicolumn{9}{l}{\textit{\textbf{LLM: DeepSeek-V4-Flash}}} \\
EoH & 1.496 & 19.73$\pm$0.49 & 1.352 & 28.76$\pm$4.45 & 2.759 & 16.31$\pm$0.53 & 2.670 & 0.88$\pm$0.48 \\
MCTS-AHD & 1.527 & 21.22$\pm$2.24 & 1.210 & 24.59$\pm$5.25 & 2.733 & 18.16$\pm$1.59 & 2.671 & 0.90$\pm$0.28 \\
MEoH & 1.484 & 19.17$\pm$0.91 & 1.330 & 26.37$\pm$2.33 & 2.787 & 16.48$\pm$0.49 & 2.672 & 0.73$\pm$0.02 \\
ReEvo & 1.467 & 18.20$\pm$3.91 & 1.311 & 24.91$\pm$5.47 & 2.614 & 21.89$\pm$3.15 & 2.668 & 1.07$\pm$0.26 \\
\rowcolor{gray!20}$\text{DGA}_2\text{D}$ (Ours) & 1.307 & 5.24$\pm$2.79 & 1.107 & 9.94$\pm$1.85 & 2.870 & 10.40$\pm$1.80 & \textbf{2.687} & \textbf{0.51$\pm$0.76} \\
\midrule
\multicolumn{9}{l}{\textit{\textbf{LLM: GPT-5.6 Sol}}} \\
EoH & 1.465 & 20.10$\pm$1.08 & 1.327 & 26.80$\pm$1.30 & 2.769 & 16.68$\pm$0.29 & 2.677 & 0.71$\pm$0.02 \\
MCTS-AHD & 1.558 & 18.50$\pm$4.66 & 1.297 & 23.32$\pm$6.58 & 2.758 & 15.16$\pm$1.16 & 2.668 & 1.05$\pm$0.54 \\
MEoH & 1.461 & 20.73$\pm$3.34 & 1.309 & 26.63$\pm$0.60 & 2.787 & 16.41$\pm$0.41 & 2.669 & 1.06$\pm$0.53 \\
ReEvo & 1.482 & 22.60$\pm$2.20 & 1.358 & 21.60$\pm$7.06 & 2.783 & 16.78$\pm$0.33 & 2.662 & 1.26$\pm$0.01 \\
\rowcolor{gray!20}$\text{DGA}_2\text{D}$ (Ours) & 1.298 & 4.70$\pm$1.71 & \textbf{1.075} & \textbf{6.54$\pm$2.52} & \textbf{3.057} & \textbf{6.61$\pm$1.52} & 2.686 & 0.53$\pm$0.37 \\
\bottomrule
\end{tabular}
\begin{tablenotes}
    \small
    \item * “Neural” denotes learning-based methods.
\end{tablenotes}
\end{table*}

\paragraph{Benchmarks.}
This paper evaluates $\text{DGA}_2\text{D}$ on a comprehensive benchmark
suite covering 12 NP-hard combinatorial optimization domains:
\textbf{resource scheduling} (Job-Shop Scheduling (JSP),
Flexible Job-Shop Scheduling (FJSP), Flow-Shop Scheduling (FSSP),
Open-Shop Scheduling (OSSP), Resource-Constrained Project Scheduling
(RCPSP), and Simple Assembly Line Balancing (S-ALBP));
\textbf{vehicle routing} (the Traveling Salesman Problem (TSP)
and the Capacitated Vehicle Routing Problem (CVRP));
\textbf{spatial allocation} (the 3D Bin Packing Problem (3D-BPP)
and the Three-Dimensional Container Loading Problem (3D-CLP));
and \textbf{graph optimization} (the Maximum Independent Set (MIS)
and Maximum Cut (Max-Cut) problems).

Detailed definitions of all 12 problems are provided in
Appendix~A.2, while the dataset sources, numbers of instances,
and problem-size ranges are reported in Appendix~E.
For detailed analysis in the main text, we select one representative
domain from each category: \textbf{FJSP, TSP, 3D-CLP, and MIS}.
Results on the remaining eight domains are reported in Appendix~F.2, which gives full evaluation on the cross-domain generalizability of
$\text{DGA}_2\text{D}$.

For each domain, the instances are divided into an evolution set and a held-out test set.
For each domain, the evolution and test sets are mutually exclusive. 
The test set is used only for final evaluation after the algorithm generation process is completed; detailed split sizes are provided in Appendix~E.

\paragraph{Baselines.} 
$\text{DGA}_2\text{D}$ is compared against three categories of baselines. 
(1) Classical solvers: OR-Tools CP-SAT solver is adopted as a general-purpose optimization baseline, together with an Iterated Local Search (ILS) baseline \cite{lourencco2018iterated}. (2) Learning-based methods: Representative neural solvers are involved, including POMO \cite{kwon2020pomo} for TSP, ReSched \cite{xiao2026resched} for FJSP, SDDS \cite{sanokowski2025scalable} for MIS and BPP-DRL \cite{hu2017solving} for 3D-CLP.(3) LLM-based AHD methods: Recent LLM-based AHD frameworks, including EoH \cite{liu2024evolution}, MCTS-AHD \cite{zheng2025monte}, MEoH \cite{yao2025multi}, and ReEvo \cite{ye2024reevo}, are also introduced for comparison.

\paragraph{Implementation and Evaluation Protocol.}
$\text{DGA}_2\text{D}$ is implemented in Python~3.13, and all
experiments are conducted on a workstation equipped with an Intel
Core i7-14650HX CPU, 16\,GB of RAM, and an NVIDIA RTX~4090 GPU. 

The sampling temperature $\tau$ of the initialization process is set to 0.7. 
For each problem domain,
the iterative algorithm evolution workflow is limited to 500 LLM calls. 
During the algorithm generation process, all generated algorithms are evaluated on the same benchmark instances under an identical runtime protocol. Each of them is given a wall-clock budget of 90 seconds per instance. 
The ones that fail to compile, raise runtime errors, produce infeasible solutions, or exceed the time limit are assigned the lowest fitness. 

Unless otherwise specified, baseline implementations and hyperparameters follow the official code and the recommended settings reported in the corresponding publications.

To account for the stochasticity of LLM generation and evolutionary search, we perform 10 independent runs with different random seeds and report the mean and standard deviation across these runs.

\paragraph{Evaluation Metrics.}
We report the objective value (Obj.) and the normalized gap relative to the reference solution.
The gap is computed for each test instance after normalizing the optimization direction, so that smaller values consistently indicate better performance. 
Hence, the reported gap is the average over all test instances.
It is also noted that the EP+FF solution is used as the reference for 3D-CLP, while the best-known solutions (BKSs) are used as the references for the rest benchmarks. 

\subsection{4.2 Comparative Studies}
\label{subsec:Comparison_results}

\subsubsection{4.2.1 Convergence Analysis}

Figure \ref{fig:convergence} illustrates the convergence trajectories of $\text{DGA}_2\text{D}$ compared to the LLM-based AHD frameworks. 
First, it can be seen that $\text{DGA}_2\text{D}$ exhibits a pattern of frequent, multi-step descents, ultimately converging to the smallest normalized gap over all instances. 
Second, the continued improvements observed over a long search horizon suggest that increased structural flexibility contributes to escaping local optima and identifying better solutions in later generations.

\begin{figure}[htbp]
    \centering
    \includegraphics[width=0.95\columnwidth]{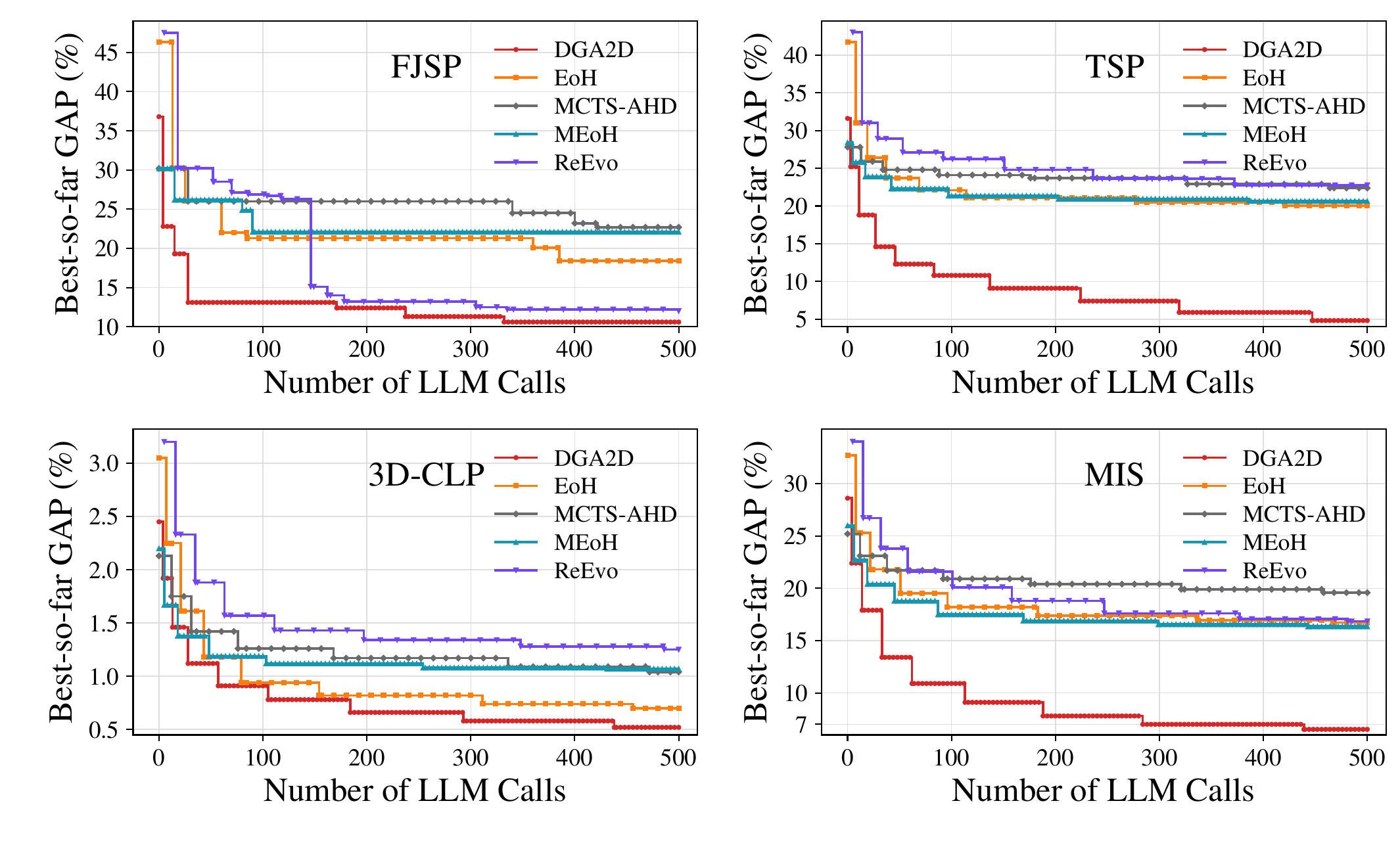}
    
    \caption{Convergence analysis of $\text{DGA}_2\text{D}$ compared with baseline frameworks.}
    \label{fig:convergence}
\end{figure}

\subsubsection{4.2.2 Experiments on Benchmarks}
As shown in Table 1, $\mathrm{DGA}_{2}\mathrm{D}$ outperforms all LLM-based AHD baselines across the four benchmarks under both DeepSeek-V4-Flash and GPT-5.6 Sol, using the same prompt template and evaluation budget. 
Averaged across the four tasks, $\mathrm{DGA}_{2}\mathrm{D}$ reduces the normalized gap from 15.69\% to 6.02\% under DeepSeek-V4-Flash, corresponding to an absolute improvement of 9.67 percentage points over the strongest baseline. 
Under GPT-5.6 Sol, the average gap decreases from 15.56\% to 4.60\%, corresponding to a reduction of 10.96 percentage points. 
The consistent gains across both backbones suggest that $\mathrm{DGA}_{2}\mathrm{D}$ is not tied to a specific LLM. 

Compared to conventional baselines, it achieves the smallest gaps on FJSP and 3D-CLP and, with GPT-5.6 Sol, also outperforms ILS and OR-Tools on MIS. Although specialized solvers remain stronger on TSP, $\mathrm{DGA}_{2}\mathrm{D}$ demonstrates robust performance across diverse combinatorial optimization problems without relying on manually designed, problem-specific search strategies.

\subsection{4.3 Ablation Studies}

\label{subsec:ablation}

\subsubsection{4.3.1 Structural Flexibility}
\label{sec:ablation_structure}

To examine the structural flexibility, we compare four
structures: Fixed, Linear, Directed Acyclic Graph (DAG), and
DG. These variants progressively relax the
constraints on operator composition, with DG additionally allowing
cyclic transitions. All variants are evaluated under the same
search and evaluation budgets.

\begin{table}[htbp]
\centering
\caption{Effect of structural representations on algorithmic convergence. Results are reported as mean $\pm$ standard deviation (10 trials); lower values are better. }
\label{tab:ablation_structure}
\small
\renewcommand{\arraystretch}{1.3} 
\begin{tabular*}{\linewidth}{@{\extracolsep{\fill}} l c c c c @{}}
\toprule
\multirow{2}{*}{\textbf{Structure}} & \multicolumn{4}{c}{\textbf{Gap (\%) }} \\
\cmidrule(lr){2-5}
 & \textbf{CVRP} & \textbf{FJSP} & \textbf{MIS} & \textbf{3D-CLP} \\
\midrule
Fixed  & 8.62$\pm$1.34 & 12.07$\pm$3.73 & 9.06$\pm$1.83 & 1.32$\pm$0.22 \\
Linear & 0.99$\pm$0.44 & 8.45$\pm$2.61  & 8.86$\pm$1.79 & 0.98$\pm$0.17 \\
DAG    & 1.09$\pm$0.49  & 5.34$\pm$0.65  & 6.95$\pm$0.41 & 0.85$\pm$0.19 \\
DG     & \textbf{0.89$\pm$0.40} & \textbf{5.07$\pm$0.20}  & \textbf{6.46$\pm$0.30} & \textbf{0.71$\pm$0.14} \\
\bottomrule
\end{tabular*}
\end{table}

Table~\ref{tab:ablation_structure} shows that increasing structural
flexibility generally improves algorithm-generation performance.
Across all four benchmarks, the DG representation achieves the
lowest test gap, suggesting that cyclic operator transitions
provide useful flexibility for expressing iterative refinement,
repeated local search, and other complex algorithmic behaviors.
Overall, the results support DG as an effective representation for
system-level automated algorithm design.

\subsubsection{4.3.2 Effect of Credit Assignment Mechanisms}

To evaluate the proposed credit assignment mechanism, we experimentally compare different levels of path dependence in credit assignment, including zero-order, first-order, second-order, and full-path settings.

\begin{table}[htbp]
\centering
\caption{Effect of credit assignment mechanisms on algorithmic
convergence. Results are reported as mean $\pm$ standard deviation (10 trials).}
\label{tab:ablation_scoring}
\small
\renewcommand{\arraystretch}{1.3}

\begin{tabular*}{\linewidth}
{@{\extracolsep{\fill}} l c c c c @{}}
\toprule
\multirow{2}{*}{\textbf{Context}}
& \multicolumn{4}{c}{\textbf{Gap (\%)}} \\
\cmidrule(lr){2-5}
& \textbf{TSP}
& \textbf{FJSP}
& \textbf{MIS}
& \textbf{3D-CLP} \\
\midrule
Zero
& 10.73$\pm$3.82
& 15.63$\pm$7.56
& 19.25$\pm$0.53
& 0.82$\pm$0.12 \\

First
& \textbf{7.07$\pm$0.75}
& \textbf{9.20$\pm$0.96}
& \textbf{16.31$\pm$4.07}
& \textbf{0.65$\pm$0.11} \\

Second
& 8.50$\pm$1.59
& 11.44$\pm$2.21
& 17.69$\pm$3.43
& 2.07$\pm$0.63 \\

Full
& 19.80$\pm$5.21
& 19.11$\pm$7.02
& 23.74$\pm$7.36
& 9.72$\pm$3.06 \\
\bottomrule
\end{tabular*}

\vspace{2pt}
\begin{minipage}{\linewidth}
\footnotesize
\textit{Note.} Contexts zero, first, and second denote zero-order,
first-order, and second-order Markov credit assignment,
respectively; Full denotes full-path credit assignment.

\end{minipage}
\end{table}

As shown in Table~\ref{tab:ablation_scoring}, two observations can be made. 
First, path-dependent credit is more effective than path-independent credit. In particular, first-order credit consistently outperforms zero-order credit across all four tasks, indicating that transitions between consecutive code implementations provides a more informative basis for implementation selection.

Second, incorporating longer path histories does not necessarily improve performance. Both second-order and full-path credits perform worse than first-order credit, with full-path credit producing particularly large gaps. This suggests that excessive path dependence may introduce noisy or sparse scoring signals, leading to unstable search and poor convergence.

Overall, first-order path dependence provides an effective balance between informative implementation credit and search stability.

\subsubsection{4.3.3 Sensitivity to Operator Pool Capacity} 
We evaluate how the heuristic operator pool size ($K\in\{4,8,15,20\}$) affects performance across different problem domains. As shown in Table~\ref{tab:sensitivity_vlevel}, $K$ controls the degree of flexibility in operator selection. A small $K$ restricts search diversity and may limit the achievable performance, as observed for $K=4$ across all tasks. Increasing $K$ provides more candidate operators and improves the potential solution quality.

However, an excessively large pool also enlarges the search space and may introduce redundant or competing operators. Consequently, different problems favor different Operator-pool capacities: FJSP performs best at $K=15$, whereas CVRP, MIS, and 3D-CLP obtain their best results at $K=20$. Overall, $K$ should balance search flexibility and optimization stability according to the characteristics of each problem.

\begin{table}[htbp]
\centering
\caption{Sensitivity analysis of the heuristic operator pool size ($K$-Level). Results are reported as mean $\pm$ standard deviation (10 trials).}
\label{tab:sensitivity_vlevel}
\small
\renewcommand{\arraystretch}{1.3} 
\begin{tabular*}{\linewidth}{@{\extracolsep{\fill}} l c c c c @{}}
\toprule
\multirow{2}{*}{\textit{K}\textbf{-Level}} & \multicolumn{4}{c}{\textbf{Gap (\%)}} \\
\cmidrule(lr){2-5}
 & \textbf{CVRP} & \textbf{FJSP} & \textbf{MIS} & \textbf{3D-CLP} \\
\midrule
$K = 4$ & 16.19$\pm$3.42 & 27.01$\pm$4.01 & 8.93$\pm$1.12 & 1.61$\pm$0.27 \\
$K = 8$ & 2.07$\pm$0.51 & 17.79$\pm$2.86 & 7.54$\pm$0.73 & 1.63$\pm$0.19 \\
$K = 15$ & 5.31$\pm$1.38 & \textbf{12.29$\pm$1.71} & 5.50$\pm$0.91 & 1.50$\pm$0.12 \\
$K = 20$ & \textbf{1.07$\pm$0.39} & 12.34$\pm$1.51 & \textbf{4.19$\pm$0.59} & \textbf{1.41$\pm$0.18} \\
\bottomrule
\end{tabular*}
\end{table}

\section {5 Discussion and Conclusion} \label{sec:discussions}
DGA$_2$D extends LLM-driven AHD from isolated heuristic optimization to system-level algorithm evolution. Its directed-graph representation jointly evolves operator implementations and connectivity, while first-order context-dependent credit assignment enables feedback sharing across pipelines without the sparsity of full-path conditioning. Experiments across 12 combinatorial optimization problems demonstrate broad applicability, and results on four representative benchmarks show consistent improvements over LLM-based AHD baselines. Ablation studies further confirm the benefits of directed-graph representations and first-order implementation-transition scoring.

DGA$_2$D remains limited by the computational cost of some generated algorithms, sensitivity to the initial operator pool, and potential cross-module inconsistencies. Future work will focus on more efficient evaluation, adaptive operator-pool construction, and stronger program-invariant
checking.

\section{Acknowledgments}
This paper is supported by National Key Research and Development Program of China (Grant No. 2024YFB3311902), Beijing Natural Science Foundation (Grant No. L241018) and Beijing Nova Program (Grant No. 2024085).

\bibliography{aaai2027}

\onecolumn
\setcounter{secnumdepth}{3}
\appendix
\begin{center}
  {\LARGE\bfseries Supplementary Material}
\end{center}
\vspace{1em}

\section{Baseline Frameworks and Selection Rationale}
This section reviews four representative LLM-based automated heuristic
design baselines and summarizes the mathematical formulations of the
twelve combinatorial optimization problems considered in our experiments.
\subsection{Review of LLM-based AHD Frameworks}

We consider four representative LLM-based automated heuristic design frameworks: EoH~(Liu et al.2024),
MCTS-AHD~(Zheng et al.2025),
MEoH~(Yao et al.2025), and
ReEvo~(Ye et al.2024).
As summarized in Table~\ref{tab:baseline_comparison}, these methods
represent four complementary directions in heuristic evolution:
thought code co-evolution, tree-structured search, multi-objective
diversity-aware optimization, and reflection-guided evolution.

\begin{table}[htbp]
    \centering
    \caption{Characteristics of the selected LLM-based AHD baselines.}
    \label{tab:baseline_comparison}
    \small
    \renewcommand{\arraystretch}{1.15}
    \setlength{\tabcolsep}{5pt}

    \begin{tabularx}{\linewidth}{
        @{}
        >{\raggedright\arraybackslash}p{2.0cm}
        >{\raggedright\arraybackslash}p{2.4cm}
        >{\raggedright\arraybackslash}X
        >{\raggedright\arraybackslash}X
        @{}
    }
        \toprule
        Method &
        Search structure &
        Guidance mechanism &
        Key characteristic \\
        \midrule

        EoH &
        Evolutionary search &
        Fitness-guided variation &
        Thought--code co-evolution \\

        MCTS-AHD &
        Tree-structured search &
        Visit and value statistics &
        Preservation and revisiting of search branches \\

        MEoH &
        Multi-objective search &
        Pareto dominance and AST diversity &
        Diverse non-dominated heuristic set \\

        ReEvo &
        Evolutionary search &
        Short- and long-term reflection &
        Reflection-guided crossover and mutation \\

        \bottomrule
    \end{tabularx}
\end{table}

We select these methods based on methodological representativeness,
experimental comparability, and recency. They generate executable
heuristics and can be evaluated under a common benchmark set, objective
function, and computational budget. Together, they cover complementary
search structures, feedback mechanisms, representations, and
diversity-management strategies, providing a broad basis for evaluating
the proposed method.

Concurrent with our work, $\mathrm{A}_2\mathrm{DEPT}$ also targets
system-level algorithm design through tree-structured evolution of
complete solver programs. In contrast, $\mathrm{DGA}_{2}\mathrm{D}$
represents the solver space as a directed operator graph with
operator-specific implementation pools, constructs complete algorithms
as directed walks, and assigns context-dependent credit to local
implementation transitions. This design supports component reuse and
localized evolution of both graph connectivity and operator
implementations. We do not include $\mathrm{A}_{2}\mathrm{DEPT}$ in the
empirical comparison because, to the best of our knowledge, no official
implementation was publicly available at the time of submission; therefore,
we restrict the unified comparison to methods with publicly available
implementations.

\subsection{Mathematical Formulations of the 12 COPs}

Throughout this subsection, let $\mathbb{B}:=\{0,1\}$.
For the disjunctive scheduling constraints,
$M^{\mathrm{big}}>0$ denotes a sufficiently large constant.

\subsubsection{Scheduling Problem}
\begin{definition}[JSP]
Given a set of jobs $J=\{1,2,\ldots,n\}$ and a set of machines
$\mathcal{M}=\{1,2,\ldots,m\}$, each job $i\in J$ consists of a fixed
sequence of $m$ operations. Let $\sigma_{i,r}\in\mathcal{M}$ denote the
machine required by the $r$-th operation of job $i$, and let
$p_{ij}$ denote the processing time of job $i$ on machine $j$.
Let $S_{ij}\geq 0$ denote the start time of job $i$'s operation on
machine $j$.

For each pair of jobs $i,k\in J$ with $i<k$ and each machine
$j\in\mathcal{M}$, let $x_{ikj}\in\mathbb{B}$ equal 1 if job $i$ is
processed before job $k$ on machine $j$, and 0 otherwise. The JSP is
formulated as
\begin{align}
\min \quad & C_{\max} \\
\text{s.t.}\quad
& S_{i,\sigma_{i,r}} + p_{i,\sigma_{i,r}}
  \leq S_{i,\sigma_{i,r+1}},
&& \forall i\in J,\; r=1,\ldots,m-1, \\
& S_{ij}+p_{ij}
  \leq S_{kj}+M^{\mathrm{big}}(1-x_{ikj}),
&& \forall i,k\in J,\; i<k,\; j\in\mathcal{M}, \\
& S_{kj}+p_{kj}
  \leq S_{ij}+M^{\mathrm{big}}x_{ikj},
&& \forall i,k\in J,\; i<k,\; j\in\mathcal{M}, \\
& S_{i,\sigma_{i,m}}+p_{i,\sigma_{i,m}}
  \leq C_{\max},
&& \forall i\in J, \\
& S_{ij}\geq 0,
&& \forall i\in J,\; j\in\mathcal{M}, \\
& x_{ikj}\in\mathbb{B},
&& \forall i,k\in J,\; i<k,\; j\in\mathcal{M}.
\end{align}
\end{definition}
\begin{definition}[FSSP]
Given a set of jobs $J=\{1,2,\ldots,n\}$ and a set of machines
$\mathcal{M}=\{1,2,\ldots,m\}$, every job must be processed on all
machines in the common order $1,2,\ldots,m$. Let $p_{ij}$ denote the
processing time of job $i$ on machine $j$, and let $S_{ij}\geq 0$
denote its start time.

For each pair of jobs $i,k\in J$ with $i<k$ and each machine
$j\in\mathcal{M}$, let $x_{ikj}\in\mathbb{B}$ equal 1 if job $i$
precedes job $k$ on machine $j$, and 0 otherwise. The FSSP is
formulated as
\begin{align}
\min \quad & C_{\max} \\
\text{s.t.}\quad
& S_{i,j+1}\geq S_{ij}+p_{ij},
&& \forall i\in J,\; j=1,\ldots,m-1, \\
& S_{ij}+p_{ij}
  \leq S_{kj}+M^{\mathrm{big}}(1-x_{ikj}),
&& \forall i,k\in J,\; i<k,\; j\in\mathcal{M}, \\
& S_{kj}+p_{kj}
  \leq S_{ij}+M^{\mathrm{big}}x_{ikj},
&& \forall i,k\in J,\; i<k,\; j\in\mathcal{M}, \\
& S_{i,m}+p_{i,m}\leq C_{\max},
&& \forall i\in J, \\
& S_{ij}\geq 0,
&& \forall i\in J,\; j\in\mathcal{M}, \\
& x_{ikj}\in\mathbb{B},
&& \forall i,k\in J,\; i<k,\; j\in\mathcal{M}.
\end{align}
\end{definition}
\begin{definition}[FJSP]
  Given a set of jobs $J=\{1,2,\ldots,n\}$ and a set of machines
$\mathcal{M}=\{1,2,\ldots,m\}$, each job $i\in J$ consists of an
ordered sequence of operations
$O_{i1},O_{i2},\ldots,O_{i n_i}$.
Each operation $O_{ij}$ can be processed on any machine in its eligible
machine set $\mathcal{M}_{ij}\subseteq\mathcal{M}$, with processing time
$p_{ijk}$ if assigned to machine $k\in\mathcal{M}_{ij}$.
Let $S_{ij}\geq 0$ denote the start time of operation $O_{ij}$.

For every $k\in\mathcal{M}_{ij}$, let $y_{ijk}\in\mathbb{B}$ equal 1
if operation $O_{ij}$ is assigned to machine $k$, and 0 otherwise.
For every $i,h\in J$ with $i<h$, operations $O_{ij}$ and $O_{h\ell}$,
and machine
$k\in\mathcal{M}_{ij}\cap\mathcal{M}_{h\ell}$, let
$x_{ijh\ell k}\in\mathbb{B}$ equal 1 if $O_{ij}$ precedes
$O_{h\ell}$ on machine $k$, and 0 otherwise.
The FJSP is formulated as
\begin{align}
\min \quad
& C_{\max}                                                     \\
\text{s.t.}\quad
& \sum_{k\in\mathcal{M}_{ij}}y_{ijk}=1,
&& \forall i\in J,\; j=1,\ldots,n_i,                           \\
& S_{i,j+1}
  \geq S_{ij}
  +\sum_{k\in\mathcal{M}_{ij}}p_{ijk}y_{ijk},
&& \forall i\in J,\; j=1,\ldots,n_i-1,                         \\
& S_{ij}+p_{ijk}
  \leq S_{h\ell}
  +M^{\mathrm{big}}
   (3-y_{ijk}-y_{h\ell k}-x_{ijh\ell k}),
&& \substack{\forall i,h\in J,\; i<h,\\
              j=1,\ldots,n_i,\;
              \ell=1,\ldots,n_h,\\
              k\in\mathcal{M}_{ij}\cap\mathcal{M}_{h\ell}},     \\
& S_{h\ell}+p_{h\ell k}
  \leq S_{ij}
  +M^{\mathrm{big}}
   (2-y_{ijk}-y_{h\ell k}+x_{ijh\ell k}),
&& \substack{\forall i,h\in J,\; i<h,\\
              j=1,\ldots,n_i,\;
              \ell=1,\ldots,n_h,\\
              k\in\mathcal{M}_{ij}\cap\mathcal{M}_{h\ell}},     \\
& S_{ij}
  +\sum_{k\in\mathcal{M}_{ij}}p_{ijk}y_{ijk}
  \leq C_{\max},
&& \forall i\in J,\; j=1,\ldots,n_i,                           \\
& S_{ij}\geq 0,
&& \forall i\in J,\; j=1,\ldots,n_i,                           \\
& y_{ijk}\in\mathbb{B},
&& \forall i\in J,\; j=1,\ldots,n_i,\;
   k\in\mathcal{M}_{ij},                                       \\
& x_{ijh\ell k}\in\mathbb{B},
&& \substack{\forall i,h\in J,\; i<h,\\
              j=1,\ldots,n_i,\;
              \ell=1,\ldots,n_h,\\
              k\in\mathcal{M}_{ij}\cap\mathcal{M}_{h\ell}}.
\end{align}
\end{definition}

\begin{definition}[OSSP]
Given a set of jobs $J=\{1,2,\ldots,n\}$ and a set of machines
$\mathcal{M}=\{1,2,\ldots,m\}$, each job $i\in J$ has one operation
on every machine $j\in\mathcal{M}$, with processing time $p_{ij}$.
Unlike JSP and FSSP, the machine-processing order of each job is not
predefined. Let $S_{ij}\geq 0$ denote the start time of job $i$'s
operation on machine $j$.

For every $i,k\in J$ with $i<k$ and $j\in\mathcal{M}$, let
$x_{ikj}\in\mathbb{B}$ equal 1 if job $i$ precedes job $k$ on machine
$j$. For every $i\in J$ and $j,\ell\in\mathcal{M}$ with $j<\ell$,
let $y_{ij\ell}\in\mathbb{B}$ equal 1 if job $i$'s operation on
machine $j$ precedes its operation on machine $\ell$. The OSSP is
formulated as
\begin{align}
\min \quad & C_{\max} \\
\text{s.t.}\quad
& S_{ij}+p_{ij}\leq C_{\max},
&& \forall i\in J,\; j\in\mathcal{M}, \\
& S_{ij}+p_{ij}
  \leq S_{kj}+M^{\mathrm{big}}(1-x_{ikj}),
&& \forall i,k\in J,\; i<k,\; j\in\mathcal{M}, \\
& S_{kj}+p_{kj}
  \leq S_{ij}+M^{\mathrm{big}}x_{ikj},
&& \forall i,k\in J,\; i<k,\; j\in\mathcal{M}, \\
& S_{ij}+p_{ij}
  \leq S_{i\ell}+M^{\mathrm{big}}(1-y_{ij\ell}),
&& \forall i\in J,\; j,\ell\in\mathcal{M},\; j<\ell, \\
& S_{i\ell}+p_{i\ell}
  \leq S_{ij}+M^{\mathrm{big}}y_{ij\ell},
&& \forall i\in J,\; j,\ell\in\mathcal{M},\; j<\ell, \\
& S_{ij}\geq 0,
&& \forall i\in J,\; j\in\mathcal{M}, \\
& x_{ikj}\in\mathbb{B},
&& \forall i,k\in J,\; i<k,\; j\in\mathcal{M}, \\
& y_{ij\ell}\in\mathbb{B},
&& \forall i\in J,\; j,\ell\in\mathcal{M},\; j<\ell.
\end{align}
\end{definition}
\begin{definition}[SALBP-1]
    Given a set of tasks $V = \{1, 2, \dots, n\}$, where each task $i \in V$ has a deterministic task time $t_i$. There is a precedence graph $G = (V, E)$, where a directed edge $(i, j) \in E$ indicates that task $i$ must be completed before task $j$ can start. The tasks are to be assigned to a sequence of ordered workstations $k \in \{1, 2, \dots, m_{\max}\}$. Let \(m_{\max}\) be an upper bound on the number of workstations.
Let \(x_{ik}\in\mathbb{B}\) equal 1 if task \(i\) is assigned to
workstation \(k\), and 0 otherwise.
Let \(y_k\in\mathbb{B}\) equal 1 if workstation \(k\) is used, and
0 otherwise.
The number of utilized workstations is
\(m=\sum_{k=1}^{m_{\max}}y_k\).Given a fixed cycle time $C$ (the maximum available time at each station), the Simple Assembly Line Balancing Problem of type 1 (SALBP-1) aims to find an assignment of tasks to workstations that minimizes the total number of utilized workstations $m$. The mathematical model is:$$ \min \sum_{k=1}^{m_{\max}} y_k $$$$ s.t. \quad \sum_{k=1}^{m_{\max}} x_{ik} = 1, \quad \forall i \in V $$$$ \sum_{i \in V} t_i \cdot x_{ik} \leq C \cdot y_k, \quad \forall k \in \{1, 2, \dots, m_{\max}\} $$$$ \sum_{k=1}^{m_{\max}} k \cdot x_{ik} \leq \sum_{k=1}^{m_{\max}} k \cdot x_{jk}, \quad \forall (i, j) \in E $$$$ y_{k+1} \leq y_k, \quad \forall k \in {1, 2, \dots, m_{\max}-1} $$$$ x_{ik} \in \{0, 1\}, \quad y_k \in \{0, 1\} $$
\end{definition}
\begin{definition}[RCPSP]
    Given a project consisting of a set of activities $V = \{0, 1, \dots, n, n+1\}$, where $0$ and $n+1$ represent dummy start and end activities with zero duration. Each actual activity $i \in \{1, \dots, n\}$ has a deterministic processing time $d_i$. The precedence constraints among activities are defined by a set of directed edges $E$, where $(i, j) \in E$ indicates that activity $j$ cannot start until activity $i$ is fully completed. The project requires a set of renewable resources $R$, where each resource $k \in R$ has a constant maximum capacity $K_k$ per time period. Processing activity $i$ requires $r_{ik}$ units of resource $k$ at every time instance during its execution. Let $\mathcal{T}:=\{0,1,\ldots,T_{\max}\}$ denote the discrete
time horizon.
Let \(x_{it}\in\mathbb{B}\) equal 1 if activity \(i\) starts at
time \(t\), and 0 otherwise.
For the dummy activities, set
\(d_0=d_{n+1}=0\) and
\(r_{0k}=r_{n+1,k}=0\) for every \(k\in R\). Given a planning horizon $T_{max}$ (an upper bound of the makespan), the Resource-Constrained Project Scheduling Problem (RCPSP) aims to determine the start time for each activity to minimize the total project duration while satisfying both precedence and resource constraints. The mathematical model is:$$ \min \sum_{t=0}^{T} t \cdot x_{n+1, t} $$$$ s.t. \quad \sum_{t=0}^{T} x_{it} = 1, \quad \forall i \in V $$$$ \sum_{t=0}^{T} t \cdot x_{jt} \geq \sum_{t=0}^{T} (t + d_i) \cdot x_{it}, \quad \forall (i, j) \in E $$$$ \sum_{i \in V} \sum_{\tau = \max\{0, t - d_i + 1\}}^{t} r_{ik} \cdot x_{i\tau} \leq K_k, \quad \forall k \in R, t \in \{0, 1, \dots, T_{max}\} $$$$ x_{it} \in \{0, 1\}, \quad \forall i \in V, t \in \{0, 1, \dots, T_{max}\} $$
\end{definition}
\subsubsection{Vehicle Routing}
\begin{definition}[CVRP]
 Given a complete graph $G = (V, E)$, where $V = \{0, 1, \dots, n\}$ is the set of nodes. Node $0$ represents the central depot, and $C = \{1, 2, \dots, n\}$ represents a set of customers. Each customer $i \in C$ has a deterministic non-negative demand $q_i$. A fleet of $K$ homogeneous vehicles, each with a maximum capacity $Q$, is stationed at the depot. The travel cost (or distance) between node $i$ and node $j$ is given as $c_{ij}$. The Capacitated Vehicle Routing Problem (CVRP) aims to find a set of vehicle routes that minimizes the total travel cost, such that every customer is visited exactly once by exactly one vehicle, all routes start and end at the depot, and the total demand on any route does not exceed the vehicle capacity $Q$. The mathematical model is:$$\min \sum_{i \in V} \sum_{j \in V} c_{ij} \cdot x_{ij}$$$$s.t. \quad \sum_{i \in V, i \neq j} x_{ij} = 1, \quad \forall j \in C$$$$\sum_{j \in V, j \neq i} x_{ij} = 1, \quad \forall i \in C$$$$\sum_{j \in C} x_{0j} \leq K$$$$\sum_{i \in C} x_{i0} \leq K$$$$u_i - u_j + Q \cdot x_{ij} \leq Q - q_j, \quad \forall i, j \in C (i \neq j)$$$$q_i \leq u_i \leq Q, \quad \forall i \in C$$$$x_{ij} \in \{0, 1\}, \quad \forall i, j \in V$$
    
\end{definition}
\begin{definition}[TSP]
     Given a complete graph $G = (V, E)$, where $V = \{1, 2, \dots, n\}$ is a set of nodes (or cities) and $E$ is the set of directed edges connecting them. A travel cost or distance $c_{ij}$ is associated with each edge $(i, j) \in E$. The Traveling Salesman Problem (TSP) aims to find a closed tour of minimal total cost that visits every node exactly once and returns to the starting node. The mathematical model (MTZ formulation) is:$$ \min \sum_{i \in V} \sum_{j \in V, j \neq i} c_{ij} \cdot x_{ij} $$$$ s.t. \quad \sum_{j \in V, j \neq i} x_{ij} = 1, \quad \forall i \in V $$$$ \sum_{i \in V, i \neq j} x_{ij} = 1, \quad \forall j \in V $$$$ u_i - u_j + n \cdot x_{ij} \leq n - 1, \quad \forall i, j \in {2, 3, \dots, n} (i \neq j) $$$$ x_{ij} \in \{0, 1\}, \quad \forall i, j \in V $$$$ 1 \leq u_i \leq n - 1, \quad \forall i \in {2, 3, \dots, n} $$
     
\end{definition}

\subsubsection{Bin Packing}
\begin{definition}[3D-BPP]
Given a set of rectangular items $I=\{1,2,\ldots,n\}$ and a set of
candidate bins $\mathcal{K}=\{1,2,\ldots,m\}$, each item $i\in I$ has
fixed length $l_i$, width $w_i$, height $h_i$, and volume
$v_i=l_iw_ih_i$. All bins are identical, with length $L$, width $W$,
and height $H$. The Three-Dimensional Bin Packing Problem (3D-BPP)
aims to orthogonally pack all items into the minimum number of bins.

Let $s_{ik}\in\mathbb{B}$ equal 1 if item $i$ is assigned to bin $k$,
and 0 otherwise, and let $u_k\in\mathbb{B}$ equal 1 if bin $k$ is used.
Let $(X_i,Y_i,Z_i)$ denote the coordinates of the lower-left-back
corner of item $i$ within its assigned bin. For every pair of items
$i,j\in I$ with $i<j$, let
$a_{ij},b_{ij},c_{ij},d_{ij},e_{ij},f_{ij}\in\mathbb{B}$ represent
the six possible relative positions of the two items along the three
coordinate axes. The 3D-BPP is formulated as
\begin{align}
\min \quad
& \sum_{k\in\mathcal{K}}u_k                                  \\
\text{s.t.}\quad
& \sum_{k\in\mathcal{K}}s_{ik}=1,
&& \forall i\in I,                                            \\
& s_{ik}\leq u_k,
&& \forall i\in I,\; k\in\mathcal{K},                         \\
& u_{k+1}\leq u_k,
&& k=1,\ldots,m-1,                                            \\
& X_i+l_i\leq L,
&& \forall i\in I,                                            \\
& Y_i+w_i\leq W,
&& \forall i\in I,                                            \\
& Z_i+h_i\leq H,
&& \forall i\in I,                                            \\
& X_i+l_i
  \leq X_j+M^{\mathrm{big}}(1-a_{ij}),
&& \forall i,j\in I,\; i<j,                                  \\
& X_j+l_j
  \leq X_i+M^{\mathrm{big}}(1-b_{ij}),
&& \forall i,j\in I,\; i<j,                                  \\
& Y_i+w_i
  \leq Y_j+M^{\mathrm{big}}(1-c_{ij}),
&& \forall i,j\in I,\; i<j,                                  \\
& Y_j+w_j
  \leq Y_i+M^{\mathrm{big}}(1-d_{ij}),
&& \forall i,j\in I,\; i<j,                                  \\
& Z_i+h_i
  \leq Z_j+M^{\mathrm{big}}(1-e_{ij}),
&& \forall i,j\in I,\; i<j,                                  \\
& Z_j+h_j
  \leq Z_i+M^{\mathrm{big}}(1-f_{ij}),
&& \forall i,j\in I,\; i<j,                                  \\
& a_{ij}+b_{ij}+c_{ij}+d_{ij}+e_{ij}+f_{ij}
  \geq s_{ik}+s_{jk}-1,
&& \forall i,j\in I,\; i<j,\; k\in\mathcal{K},                \\
& X_i,Y_i,Z_i\geq 0,
&& \forall i\in I,                                            \\
& u_k\in\mathbb{B},
&& \forall k\in\mathcal{K},                                   \\
& s_{ik}\in\mathbb{B},
&& \forall i\in I,\; k\in\mathcal{K},                         \\
& a_{ij},b_{ij},c_{ij},d_{ij},e_{ij},f_{ij}\in\mathbb{B},
&& \forall i,j\in I,\; i<j.
\end{align}
\end{definition}
\paragraph{Evaluator convention.}
Let
\begin{equation}
B:=\sum_{k\in\mathcal{K}}u_k
\end{equation}
denote the number of used bins. Because the symmetry-breaking
constraints enforce consecutive bin usage, the last used bin is
well-defined as
\begin{equation}
k^{\star}:=\max\{k\in\mathcal{K}\mid u_k=1\}.
\end{equation}
Its volume-utilization ratio is
\begin{equation}
\rho_{\mathrm{last}}
=
\frac{\sum_{i\in I}v_i s_{ik^{\star}}}{LWH}.
\end{equation}
The executable evaluator minimizes
\begin{equation}
F_{\mathrm{3D\text{-}BPP}}
=
B+\lambda\left(1-\rho_{\mathrm{last}}\right),
\qquad
\lambda=0.1.
\end{equation}
The first term is the canonical 3D-BPP objective, whereas the second
term distinguishes solutions using the same number of bins by preferring
a more highly utilized last bin. Since
$0\leq\lambda(1-\rho_{\mathrm{last}})\leq 0.1<1$, reducing the number
of used bins by one always dominates any difference in the secondary
term. Therefore, the evaluator is lexicographically consistent with
first minimizing the number of bins and then maximizing the utilization
of the last bin.
\begin{definition}[3D-CLP]
    Given a set of rectangular items $I=\{1,2,\ldots,n\}$ and a single
rectangular container with length $L$, width $W$, and height $H$, each
item $i\in I$ has length $l_i$, width $w_i$, height $h_i$, and volume
$v_i=l_iw_ih_i$. The Three-Dimensional Container Loading Problem
(3D-CLP) aims to pack a feasible subset of the items into the container
so as to maximize the total packed volume.

Let $s_i\in\mathbb{B}$ equal 1 if item $i$ is packed into the container,
and 0 otherwise. Let $(X_i,Y_i,Z_i)$ denote the coordinates of the
lower-left-back corner of item $i$. For every pair of items $i,j\in I$
with $i<j$, let
$a_{ij},b_{ij},c_{ij},d_{ij},e_{ij},f_{ij}\in\mathbb{B}$ represent the
six possible relative positions of the two items along the three
coordinate axes. The 3D-CLP is formulated as
\begin{align}
\max \quad
& \sum_{i\in I} v_i s_i                                      \\
\text{s.t.}\quad
& X_i+l_i
  \leq L+M^{\mathrm{big}}(1-s_i),
&& \forall i\in I,                                            \\
& Y_i+w_i
  \leq W+M^{\mathrm{big}}(1-s_i),
&& \forall i\in I,                                            \\
& Z_i+h_i
  \leq H+M^{\mathrm{big}}(1-s_i),
&& \forall i\in I,                                            \\
& X_i+l_i
  \leq X_j+M^{\mathrm{big}}(1-a_{ij}),
&& \forall i,j\in I,\; i<j,                                  \\
& X_j+l_j
  \leq X_i+M^{\mathrm{big}}(1-b_{ij}),
&& \forall i,j\in I,\; i<j,                                  \\
& Y_i+w_i
  \leq Y_j+M^{\mathrm{big}}(1-c_{ij}),
&& \forall i,j\in I,\; i<j,                                  \\
& Y_j+w_j
  \leq Y_i+M^{\mathrm{big}}(1-d_{ij}),
&& \forall i,j\in I,\; i<j,                                  \\
& Z_i+h_i
  \leq Z_j+M^{\mathrm{big}}(1-e_{ij}),
&& \forall i,j\in I,\; i<j,                                  \\
& Z_j+h_j
  \leq Z_i+M^{\mathrm{big}}(1-f_{ij}),
&& \forall i,j\in I,\; i<j,                                  \\
& a_{ij}+b_{ij}+c_{ij}+d_{ij}+e_{ij}+f_{ij}
  \geq s_i+s_j-1,
&& \forall i,j\in I,\; i<j,                                  \\
& X_i,Y_i,Z_i\geq 0,
&& \forall i\in I,                                            \\
& s_i\in\mathbb{B},
&& \forall i\in I,                                            \\
& a_{ij},b_{ij},c_{ij},d_{ij},e_{ij},f_{ij}\in\mathbb{B},
&& \forall i,j\in I,\; i<j.
\end{align}
\end{definition}
\paragraph{Evaluator convention.}
The packed-volume utilization is defined as
\begin{equation}
\rho_{\mathrm{fill}}
=
\frac{\sum_{i\in I}v_i s_i}{LWH}.
\end{equation}
Since the executable evaluation engine follows a minimization
convention, it evaluates each solution using
\begin{equation}
F_{\mathrm{3D\text{-}CLP}}
=
100\left(1-\rho_{\mathrm{fill}}\right)
=
100-
100\frac{\sum_{i\in I}v_i s_i}{LWH}.
\end{equation}
Thus, $F_{\mathrm{3D\text{-}CLP}}=0$ corresponds to a completely
filled container, and lower values indicate better performance.
Since the container volume $LWH$ is fixed, minimizing
$F_{\mathrm{3D\text{-}CLP}}$ is equivalent to maximizing the total
packed volume.

For compactness, the formulation above assumes fixed item
orientations. The executable evaluator additionally considers the
item-specific axis-aligned orientations permitted by the
\texttt{orientation\_allowed} flags; this extension changes the
feasibility set but not the packed-volume objective.
\subsubsection{Graph Optimization}

\begin{definition}[MIS]
Given an undirected graph $G=(V,E)$, the Maximum Independent Set (MIS)
problem aims to find a vertex subset $S \subseteq V$ with maximum
cardinality such that no two vertices in $S$ are adjacent.
The mathematical model is:

\begin{equation*}
\max f(S)=|S| \end{equation*}
\begin{equation*}
\text{s.t.}\ \{u,v\}\notin E,
\qquad
\forall u,v\in S,\; u\neq v.
\end{equation*}

\end{definition}
\begin{definition}[Max-Cut]
Given an undirected graph $G=(V, E)$ with edge weights $W=\{w_{uv}\}_{\{u,v\} \in E}$, the Maximum Cut (Max-Cut) problem aims to find a vertex subset $S \subseteq V$ that maximizes the total weight of the edges crossing the partition between $S$ and $V \setminus S$. The mathematical model is:
\begin{equation*}
\max_{S\subseteq V}
\quad
f(S)=
\sum_{\substack{\{u,v\}\in E\\
u\in S,\;v\in V\setminus S}}
w_{uv}.
\end{equation*}
\end{definition}

\section{Implementation Details of DGA$_2$D}

\begin{algorithm}[htbp]
\caption{DGA$_2$D Training}
\label{alg:dga2d_training}
\footnotesize
\begin{algorithmic}[1]

\Require Domain evaluator $f$;
evolution instances $\mathcal{X}_{\mathrm{evo}}$;
operator repository $\mathcal{O}_{\mathrm{lib}}$
\Require Generations $t_{max}$; algorithm candidate per generation $N$;
shared instances $k$; temperature $\tau$; maximum implementation-pool size $M$; 
credit rate $\alpha$

\Ensure Best discovered program $\mathcal A^\star$, Evolved graph $\mathcal{D}_{t_{max}+1}$ and
implementation pools $\mathcal{C}_{t_{max}+1}$

\State $(\mathcal{O},\mathcal{D}_1,\mathcal{C}^{(1)})
\gets
\Call{LLMInitialization}
{f,\mathcal{O}_{\mathrm{lib}}}$

\State Initialize pipeline policy
$\pi_{\theta}^{\mathrm{pipe}}$
, transition credits $Q_1\gets 0$ and transition counts
$n_1(e)\gets0$
\State $\mathcal A^\star\gets\varnothing,\quad S^\star\gets-\infty$
\State $\mathcal{I}
    \gets
    \Call{RandomSample}
    {\mathcal{X}_{\mathrm{evo}},k}$
\For{$t=1,\ldots,t_{max}$}

    \For{$j=1,\ldots,N$}

        \State $P_{t,j}
        =
        (O_{\mathrm{in}},
        O_{t,j,1},\ldots,
        O_{t,j,L_{t,j}},
        O_{\mathrm{out}})
        \sim
        \pi_{\theta}^{\mathrm{pipe}}
        (\cdot\mid\mathcal{D}_t)$
        \Comment{bounded directed walk}

        \State $c_{t,j,0}\gets c_{\mathrm{start}}$

        \For{$i=1,\ldots,L_{t,j}$}
\State $c_{t,j,i}\sim
    \operatorname{Softmax}_{c\in\mathcal C^{(t)}(O_{t,j,i})}
    \left(
    \frac{Q_t(e(c_{t,j,i-1},c))}{\tau}
    \right)$
    \State $e_{t,j,i}\gets
    e(c_{t,j,i-1},c_{t,j,i})$

        \EndFor

        \State $\mathcal{A}_{t,j}
        \gets{(c_{t,j,1},\ldots,c_{t,j,L_{t,j}})}$

    \EndFor

    \State $\{y_{t,j,x}\}
    \gets
    \Call{EvaluateCandidates}
    {\{\mathcal{A}_{t,j}\}_{j=1}^{N},\mathcal{I}}$

    \For{$j=1,\ldots,N$}

        \State $S_{t,j}
        \gets
        -\dfrac{1}{k}
        \sum_{x\in\mathcal{I}}
        g(y_{t,j,x};x)$
\If{$S_{t,j}>S^\star$}
        \State $S^\star\gets S_{t,j}$
        \State $\mathcal A^\star\gets\mathcal A_{t,j}$
    \EndIf
    \EndFor

    \State $\widetilde{R}_{t,j}
    \gets
    \Call{Standardize}
    {\{S_{t,j}\}_{j=1}^{N}}$,
    \quad $j=1,\ldots,N$

    \State $\mathcal{L}_{\mathrm{pipe}}
    \gets
    -\dfrac{1}{N}
    \sum_{j=1}^{N}
    \widetilde{R}_{t,j}
    \log
    \pi_{\theta}^{\mathrm{pipe}}
    (P_{t,j}\mid\mathcal{D}_t)$

    \State Update $\theta$ by minimizing
    $\mathcal{L}_{\mathrm{pipe}}$

 \State $Q_{t+1}\gets Q_t,\quad n_{t+1}\gets n_t$

\For{$j=1,\ldots,N$}
    \For{$i=1,\ldots,L_{t,j}$}

        \State $Q_{t+1}(e_{t,j,i})
        \gets
        (1-\alpha)Q_{t+1}(e_{t,j,i})
        +\alpha\widetilde{R}_{t,j}$

        \State $n_{t+1}(e_{t,j,i})
        \gets
        n_{t+1}(e_{t,j,i})+1$
    \EndFor
\EndFor
\State
$(\mathcal D_{t+1},\mathcal C^{(t+1)})
\gets
\Call{DualLevelEvolution}
{\mathcal D_t,\mathcal C^{(t)},Q_{t+1},n_{t+1},M}$
\EndFor
\State \Return
$\mathcal D_{t_{\max}+1}, \mathcal C^{(t_{\max}+1)}$, $\mathcal A^\star$

\end{algorithmic}
\end{algorithm}
This appendix provides the implementation-level details of DGA$_2$D.
Algorithm~\ref{alg:dga2d_training} describes the overall training and
candidate-evaluation procedure, while
Algorithm~\ref{alg:dual_evolution} details the credit-guided evolution
of the operator graph and implementation pools.
Algorithm~\ref{alg:dga2d_training} summarizes the complete training procedure
of DGA$_2$D. Before training, we randomly sample a fixed batch
$\mathcal I$ of $k$ evolution instances from
$\mathcal X_{\mathrm{evo}}$. The same batch is reused across all
generations, ensuring that candidate scores remain directly comparable
when maintaining the best discovered program $\mathcal A^\star$. The pipeline policy
samples a length-bounded directed walk
$P_{t,j}$ from $O_{\mathrm{in}}$ to $O_{\mathrm{out}}$ in the current graph
$\mathcal D_t$. For every operator in the sampled pipeline, a concrete
implementation is selected according to the transition credits:
\[
c_{t,j,i}
\sim
\operatorname{Softmax}_{c\in\mathcal C^{(t)}(O_{t,j,i})}
\left(
\frac{Q_t(e(c_{t,j,i-1},c))}{\tau}
\right),
\]
where $e(c,c')$ denotes the implementation-level transition from $c$ to
$c'$. The temperature $\tau$ controls the balance between exploiting
high-credit transitions and exploring alternative implementations.

Each instantiated candidate $\mathcal A_{t,j}$ is evaluated on the shared
instance batch, and its score is computed as
\[
S_{t,j}
=
-\frac{1}{k}
\sum_{x\in\mathcal I}
g(y_{t,j,x};x).
\]
The best-performing candidate observed throughout training is retained as
$\mathcal A^\star$. The candidate scores are standardized into
$\widetilde R_{t,j}$ and used to optimize the pipeline policy through
\[
\mathcal L_{\mathrm{pipe}}
=
-\frac{1}{N}
\sum_{j=1}^{N}
\widetilde R_{t,j}
\log
\pi_\theta^{\mathrm{pipe}}
(P_{t,j}\mid\mathcal D_t).
\]
Here, $\pi_{\theta}^{\mathrm{pipe}}(P_{t,j}\mid D_t)$ denotes the
probability of sampling pipeline $P_{t,j}$ under the current pipeline
policy. Its autoregressive factorization and the corresponding masked
operator-selection distributions are detailed in Appendix~D.1.

Thus, pipelines with positive standardized rewards become more likely to be
sampled, whereas poorly performing pipelines are suppressed. The same reward
is also assigned to every implementation transition activated by the
candidate:
\[
Q_{t+1}(e_{t,j,i})
=
(1-\alpha)Q_{t+1}(e_{t,j,i})
+
\alpha\widetilde R_{t,j},
\]
while $n_{t+1}(e_{t,j,i})$ records its cumulative activation count. The
updated credits are subsequently used by
Algorithm~\ref{alg:dual_evolution} to evolve the implementation pools
and operator graph.
\begin{algorithm}[ht]
\caption{Credit-Guided Dual-Level Evolution}
\label{alg:dual_evolution}
\footnotesize
\begin{algorithmic}[1]

\Require Current graph
$\mathcal{D}_t=(\mathcal{O}_{t},\mathcal{E}_t)$;
implementation pools $\mathcal{C}^{(t)}$;
transition credits $Q_{t}(e)$; transition counts $n_{t}(e)$; maximum pool size $M$

\Ensure Updated graph $\mathcal{D}_{t+1}$ and
implementation pools $\mathcal{C}^{(t+1)}$

\State $\bar{Q}_t(c),\bar{Q}_t(O_i)
\gets
\Call{AggregateImplementationCredits}
{Q_{t}(e),n_t(e)}$

\State $O^\star
\gets
\arg\min_{O_i\in\mathcal{O}}
\bar{Q}_t(O_i)$

\State $c^{-}
\gets
\arg\min_{c\in\mathcal{C}_{O^\star}^{(t)}}
\bar{Q}_t(c)$

\State $\mathcal{C}^{(t+1)}\gets\mathcal{C}^{(t)}$
\If{$\bar{Q}_t(O^\star)\geq 0$
    \textbf{and}
    $|\mathcal{C}_{O^\star}^{(t)}|>1$}

    \State $\mathcal{C}^{(t+1)}(O^\star)
\leftarrow
\mathcal{C}^{(t)}(O^\star)\setminus\{c^-\}$

\ElsIf{$\bar{Q}_t(O^\star)<0$}

    \State $c^{+}
    \gets
    \Call{LLMGenerateImplementation}
    {O^\star,\mathcal{C}^{(t)}(O^\star)}$

    \If{$|\mathcal{C}^{(t)}(O^\star)| < M$}
    \State $\mathcal{C}^{(t+1)}(O^\star)
    \gets
    \mathcal{C}^{(t)}(O^\star)\cup\{c^+\}$
\Else
    \State $\mathcal{C}^{(t+1)}(O^\star)
    \gets
    \bigl(\mathcal{C}^{(t)}(O^\star)\setminus\{c^-\}\bigr)
    \cup\{c^+\}$
\EndIf
    \EndIf

\State $\bar{Q}_t(O_i,O_j)
\gets
\Call{AggregateEdgeCredits}
{Q_{t}(e),n_t(e)}$

\State $(O_i^{-},O_j^{-})
\gets
\arg\min_{(O_i,O_j)\in\mathcal{E}_t}
\bar{Q}_t(O_i,O_j)$

\State $\mathcal{E}_{t+1}\gets\mathcal{E}_t$

\If{$\bar{Q}_t(O_i^{-},O_j^{-})<0$}

    \State $(O_i^{+},O_j^{+})
    \gets
    \Call{LLMGenerateCompatibleEdge}
    {\mathcal{D}_t}$

    \If{the new edge is valid and preserves an
        input--output route}

        \State $\mathcal{E}_{t+1}
        \gets
        \mathcal{E}_t
        \setminus\{(O_i^{-},O_j^{-})\}
        \cup\{(O_i^{+},O_j^{+})\}$

    \EndIf

\EndIf

\State $\mathcal{D}_{t+1}
\gets
(\mathcal{O}_{t},\mathcal{E}_{t+1})$

\State \Return
$\mathcal{D}_{t+1},\mathcal{C}^{(t+1)}$

\end{algorithmic}
\end{algorithm}
\paragraph{Implementation- and Operator-Level Credit Aggregation.}
The procedure
\textsc{AggregateImplementationCredits} first aggregates transition-level
credits into an empirical credit for each implementation $c$:
\begin{equation}
\bar Q_t(c)
=
\frac{
\sum_{e\in\mathcal T_t(c)} n_t(e)Q_t(e)
}{
\sum_{e\in\mathcal T_t(c)} n_t(e)+\epsilon
},
\label{eq:implementation_credit}
\end{equation}
where $\mathcal T_t(c)$ contains all activated transitions involving $c$,
and $\epsilon>0$ prevents division by zero. The corresponding operator credit
is the average credit of its current implementations:
\begin{equation}
\bar Q_t(O_i)
=
\frac{
\sum_{c\in\mathcal C^{(t)}(O_i)}
\bar Q_t(c)
}{
|\mathcal C^{(t)}(O_i)|
}.
\label{eq:operator_credit}
\end{equation}
The lowest-credit operator $O^\star$ and its lowest-credit implementation
$c^-$ are then selected for pruning or local repair.

\paragraph{Operator-Edge Credit Aggregation.}
The procedure \textsc{AggregateEdgeCredits} aggregates all
implementation level transitions associated with an operator edge
$(O_i,O_j)$:
\begin{equation}
\bar Q_t^{\mathrm{edge}}(O_i,O_j)
=
\frac{
\sum_{e\in\mathcal T_t(O_i,O_j)}
n_t(e)Q_t(e)
}{
\sum_{e\in\mathcal T_t(O_i,O_j)}
n_t(e)+\epsilon
},
\label{eq:edge_credit}
\end{equation}
where $\mathcal T_t(O_i,O_j)$ contains transitions from implementations of
$O_i$ to implementations of $O_j$. The edge with the lowest aggregated
credit is considered for replacement. A proposed replacement is accepted
only if it is valid and the resulting graph preserves at least one route
from $O_{\mathrm{in}}$ to $O_{\mathrm{out}}$.
\paragraph{Credit-Guided Evolution.}
If the lowest operator credit is non-negative, DGA$_2$D removes its
lowest-credit implementation when more than one implementation remains,
thereby controlling pool complexity. If the operator credit is negative, a
new implementation is generated and either added to the pool or used to
replace $c^-$ when the maximum pool size $M$ has been reached. At the graph
level, only negatively credited edges are considered for replacement. \textsc{LLMGenerateImplementation} returns a candidate only after it
passes AST, signature, isolated-runtime, and composed smoke validation.
If validation fails, the implementation pool remains unchanged
\section{Prompt Engineering and Templates}

The implementation begins from a fixed, versioned operator library.  The
external proposer handles constrained graph initialization, scheduled edge
edits, and code generation after a rule-based version decision.  The pipeline
policy itself is a locally loaded Qwen model.  In the
templates below, \promptrole{blue} specifies the role, \promptinput{purple
monospace} marks injected fields, \promptconstraint{red} marks hard constraints,
and \promptoutput{green} gives the required output.

\subsection{Directed-Graph Initialization}
\begin{promptbox}[title={\textbf{Template I: Constrained DG Initialization}}]
\promptrole{Initialize a directed operator graph for an optimization RL system.}\\
\promptinput{Domain contract: \{domain\_prompt\}}\\
\promptinput{Hard YAML controls: \{bounded\_H\_controls\}}\\
\promptinput{Candidate operator IDs: \{operator\_ids\}}\\
\promptinput{Safe baseline: \{initial\_pipeline\}}

\medskip
Return JSON with exactly the keys \texttt{H}, \texttt{operators},
\texttt{entry\_nodes}, \texttt{exit\_nodes}, and \texttt{edges}.
\begin{enumerate}[leftmargin=6mm,itemsep=1pt,topsep=2pt]
  \item Select only listed operator IDs and obey every enabled YAML bound.
  \item Entry nodes must be initialization operators and cannot be revisited.
  \item Exit nodes must terminate a valid solution; at least one
        \texttt{START}-to-\texttt{END} route must exist.
  \item Directed cycles are allowed only among non-initialization nodes.
\end{enumerate}
\promptoutput{Return JSON only.  \texttt{START} and \texttt{END} are implicit.}
\end{promptbox}

\subsection{Local Pipeline-Policy Context}
\begin{promptbox}[title={\textbf{Template II: Constrained Pipeline Sampling}}]
\promptrole{Select a bounded walk through the current directed operator graph.}\\
\promptinput{Domain summary: \{domain\_summary\}}\\
\promptinput{Instance summary: \{scalar\_sizes\_and\_array\_shapes\}}\\
\promptinput{Action aliases: \{fixed\_width\_ASCII\_mapping\}}

\medskip
At each step, generation is restricted by a token trie to aliases of outgoing
edges that can still reach \texttt{END} within the remaining length.
\promptconstraint{The initialization node cannot be revisited.}  Log-probability
includes only generated operator and \texttt{END} tokens; prompt and padding
tokens are excluded.
\promptoutput{Generate one legal alias sequence ending in \texttt{END}.}
\end{promptbox}

\subsection{Transactional Graph Evolution}
\begin{promptbox}[title={\textbf{Template III: GraphDelta Proposal}}]
\promptrole{Propose a small edge-only edit to the current directed graph.}\\
\promptinput{Current nodes and edges: \{graph\_json\}}\\
\promptinput{Recent generation summary: \{performance\_summary\}}\\
\promptinput{Domain contract: \{domain\_prompt\}}

\medskip
Return JSON with \texttt{add\_edges} and \texttt{delete\_edges}, containing at
most two edges each.  Do not add nodes.  Every endpoint must already exist.
The edit is applied to a graph copy and committed only if interface checks,
DG constraints, reachability, and smoke execution all pass.
\promptoutput{Return JSON only; an empty delta is allowed.}
\end{promptbox}

\subsection{Rule-Guided Operator Evolution}
\begin{promptbox}[title={\textbf{Template IV: ADD/REPLACE Code Generation}}]
\promptrole{Generate one implementation for a selected operator.}\\
\promptinput{Domain contract: \{domain\_prompt\}}\\
\promptinput{Target operator: \{category\}|\{name\}}\\
\promptinput{Parent implementation: \{parent\_code\}}\\
\promptinput{Failure context, if any: \{diagnostic\}}

\medskip
The trainer has already selected ADD or REPLACE using fault, dominance, trial
count, and version-cap rules.  Preserve exactly
\texttt{def run(env\_data, state, calc\_makespan\_fn):}, use only documented
environment fields, and update the returned state consistently.  Generated
code is activated only after AST, signature, isolated-runtime, and composed
smoke validation.  DELETE never calls the proposer.
\promptoutput{Return one Python code block and no prose.}
\end{promptbox}

\subsection{Condensed Domain Contracts}

The following problem cards are shared by graph initialization, graph
evolution, and operator-code generation.  They match the executable evaluator
contracts.

\subsubsection{Travelling Salesman Problem (TSP)}
\begin{promptbox}[unbreakable]
\promptfield{Objective:}{\promptrole{Minimize the total length of a Hamiltonian cycle that visits every one of the $N$ cities exactly once and returns to its starting city.}}
\promptfield{Encoding:}{\texttt{state.sequence} is an integer permutation of all city indices $[0,\ldots,N-1]$.  The permutation is interpreted cyclically, so its rotations describe the same tour; city 0 may be fixed as the conventional start.}
\promptfield{Decoder:}{Visit cities in sequence order, close the tour with the edge from the last city back to the first, and sum the corresponding entries of \texttt{distance\_matrix}.  Distances use TSPLIB \texttt{EUC\_2D}: $d_{ij}=\lfloor\sqrt{(x_i-x_j)^2+(y_i-y_j)^2}+0.5\rfloor$.}
\promptfield{Environment~keys:}{\promptinput{num\_nodes, coords, distance\_matrix, upper\_bound}.}
\promptfield{Constraints:}{\promptconstraint{The sequence must have length \texttt{num\_nodes}, contain every index in $[0,\texttt{num\_nodes}-1]$ exactly once, and remain a valid permutation after every operator.}}
\end{promptbox}

\subsubsection{3D Bin Packing Problem (3D-BPP)}
\begin{promptbox}[unbreakable]
\promptfield{Objective:}{\promptrole{Pack all $N$ fixed-orientation 3D items into the fewest fixed-size bins.  Minimize $\textit{bins}+0.1(1-\textit{last-bin utilization})$, where the fractional term prefers greater utilization of the last bin among equal-bin solutions.}}
\promptfield{Encoding:}{\texttt{state.sequence} is an \texttt{int32} priority permutation of item indices $[0,\ldots,N-1]$; items are presented to the packing decoder in this order.}
\promptfield{Decoder:}{Apply Extreme-Points First-Fit.  For each item, scan existing bins and candidate points consisting of the origin and the positive $x$, $y$, and $z$ faces of placed items; sort points by $z$, then $y$, then $x$ (Bottom--Left--Front), place at the first feasible point, or open a new bin at the origin.}
\promptfield{Environment~keys:}{\promptinput{num\_activities, bin\_size, item\_sizes, item\_volumes, upper\_bound}.}
\promptfield{Constraints:}{\promptconstraint{The sequence must contain every item exactly once.  Items cannot rotate or overlap and each item must lie entirely within one bin; all indices and array accesses must remain in range.}}
\end{promptbox}

\subsubsection{3D Container Loading Problem (3D-CLP)}
\begin{promptbox}[unbreakable]
\promptfield{Objective:}{\promptrole{Maximize the volume utilization of a single fixed-size container by packing a feasible subset of the $N$ items.  The minimization engine evaluates $100-\textit{fill-rate percentage}$, so 0 denotes a completely filled container and lower values are better.}}
\promptfield{Encoding:}{\texttt{state.sequence} is an \texttt{int32} priority permutation of all item indices $[0,\ldots,N-1]$; items are presented to the packing decoder in this order, while items that cannot be placed are skipped.}
\promptfield{Decoder:}{Apply Extreme-Points First-Fit in the single container.  For each item, collect the origin and the positive $x$, $y$, and $z$ faces of placed items, sort these points by $z$, then $y$, then $x$ (Bottom--Left--Front), and try only the axis-aligned orientations permitted by that item's \texttt{orientation\_allowed} flags.  Place the item at the first feasible position or skip it if none exists.}
\promptfield{Environment~keys:}{\promptinput{num\_activities, bin\_size, item\_sizes, item\_volumes, orientation\_allowed, fill\_upper\_bound, upper\_bound, lower\_bound, name}.}
\promptfield{Constraints:}{\promptconstraint{The sequence must contain every item index exactly once and all accesses must remain in range.  Only one container is available; every packed orientation must be permitted, lie fully within the container, and not overlap another packed item.}}
\end{promptbox}

\subsubsection{Capacitated Vehicle Routing Problem (CVRP)}
\begin{promptbox}[unbreakable]
\promptfield{Objective:}{\promptrole{Minimize the total distance of all depot-to-depot vehicle routes while serving every customer demand with vehicles of fixed capacity.}}
\promptfield{Encoding:}{\texttt{state.sequence} is a flat list of 1-indexed customer IDs, separated by the depot marker 0; for example, \texttt{[1,2,0,3,4]} represents routes $(1,2)$ and $(3,4)$.  The first and last depot visits are implicit.}
\promptfield{Decoder:}{Split the sequence at depot markers.  For each nonempty route, sum the edge from \texttt{depot\_idx} to its first customer, all consecutive customer edges, and the final edge back to the depot; add these route costs to obtain total distance.}
\promptfield{Environment~keys:}{\promptinput{dist\_matrix, demands, capacity, num\_nodes, depot\_idx}.}
\promptfield{Constraints:}{\promptconstraint{Every customer in $[1,\texttt{num\_nodes}-1]$ must appear exactly once; every route demand must not exceed \texttt{capacity}; and consecutive zeros, empty routes, duplicates, omissions, and out-of-range IDs are invalid.  With $r$ routes, the sequence length is the number of customers plus $r-1$.}}
\end{promptbox}

\subsubsection{Flexible Job Shop Scheduling Problem (FJSP)}
\begin{promptbox}[unbreakable]
\promptfield{Objective:}{\promptrole{Minimize makespan while scheduling every operation of every job on one of its eligible alternative machines, preserving each job's operation order.}}
\promptfield{Encoding:}{A dual chromosome is used.  The OS chromosome \texttt{state.sequence} has length \texttt{total\_ops} and contains job $j$ exactly \texttt{op\_counts[j]} times.  The MS chromosome stored in \texttt{state.metadata["machine\_alloc"]} has shape \texttt{(num\_jobs, max\_ops)}; entry $(j,o)$ stores the actual machine ID selected for operation $o$ of job $j$.}
\promptfield{Decoder:}{Scan OS left to right; the occurrence count of a job identifies its next operation.  Read the actual machine ID from MS and its duration from \texttt{processing\_matrix}, then start it at the maximum of that job's ready time and the machine's available time.  The largest completion time is the makespan.}
\promptfield{Environment~keys:}{\promptinput{num\_jobs, num\_machines, op\_counts, eligible\_counts, machines\_matrix, times\_matrix, processing\_matrix, max\_ops, total\_ops}.}
\promptfield{Constraints:}{\promptconstraint{OS length and job multiplicities must be exact.  Every stored machine ID must occur among the operation's nonnegative eligible IDs, and its processing time must be finite.  Job precedence and single-machine capacity must hold.  An ineligible assignment receives the invalid makespan 999999999.}}
\end{promptbox}

\subsubsection{Flow Shop Scheduling Problem (FSSP)}
\begin{promptbox}[unbreakable]
\promptfield{Objective:}{\promptrole{Minimize makespan while processing every job once on each machine in the common fixed machine order.}}
\promptfield{Encoding:}{\texttt{state.sequence} is an operation-based integer sequence of length \texttt{total\_ops}.  Each job ID appears exactly \texttt{num\_machines} times, and its $k$th occurrence denotes that job's operation on machine $k$.}
\promptfield{Decoder:}{Scan the sequence left to right, infer the current machine from the job's prior occurrence count, and schedule the operation at the maximum of that job's ready time and the machine's available time using \texttt{processing\_times[machine,job]}.  The maximum completion time is $C_{\max}$.}
\promptfield{Environment~keys:}{\promptinput{num\_jobs, num\_machines, processing\_times, total\_ops}.}
\promptfield{Constraints:}{\promptconstraint{The sequence length must equal \texttt{num\_jobs * num\_machines}; each job in $[0,\texttt{num\_jobs}-1]$ must occur exactly \texttt{num\_machines} times.  Extra occurrences, invalid IDs, job-order violations, or machine conflicts are invalid.}}
\end{promptbox}

\subsubsection{Job Shop Scheduling Problem (JSP)}
\begin{promptbox}[unbreakable]
\promptfield{Objective:}{\promptrole{Minimize makespan while scheduling all job operations in their fixed technological order on their prescribed machines.}}
\promptfield{Encoding:}{\texttt{state.sequence} is an operation-based sequence of length \texttt{total\_ops}.  Each job ID appears exactly \texttt{num\_machines} times; its $k$th occurrence represents operation $k$ of that job.}
\promptfield{Decoder:}{Scan the sequence from left to right.  Use each job's occurrence count as its operation index, read the prescribed machine and duration from \texttt{machines\_matrix} and \texttt{times\_matrix}, and start at the maximum of job and machine availability.  Return the largest machine completion time.}
\promptfield{Environment~keys:}{\promptinput{num\_jobs, num\_machines, machines\_matrix, times\_matrix, total\_ops}.}
\promptfield{Constraints:}{\promptconstraint{The sequence must have length \texttt{num\_jobs * num\_machines}, contain every job exactly \texttt{num\_machines} times, and use only IDs in $[0,\texttt{num\_jobs}-1]$.  Job precedence and single-machine capacity must be preserved; violations receive makespan 999999999.}}
\end{promptbox}

\subsubsection{Maximum Independent Set (MIS)}
\begin{promptbox}[unbreakable]
\promptfield{Objective:}{\promptrole{Find a maximum-cardinality set of mutually non-adjacent vertices.  The minimization engine uses score $-|S|$, so a larger feasible independent set has a smaller, better score.}}
\promptfield{Encoding:}{\texttt{state.sequence} is a binary \texttt{int32} vector of length \texttt{num\_vertices}; entry $i=1$ selects vertex $i$, and $i=0$ excludes it.}
\promptfield{Decoder:}{Collect all vertices whose bit is 1, verify that no selected vertex has a selected neighbor in \texttt{adj}, and return the negative selected-vertex count.  An infeasible or malformed vector returns the domain's large positive \texttt{INVALID\_SCORE}.}
\promptfield{Environment~keys:}{\promptinput{num\_vertices (also available as total\_vertices), adj, name}.}
\promptfield{Constraints:}{\promptconstraint{The vector must have the declared length and contain only binary values with valid vertex indices.  For every edge $(u,v)$, at most one of \texttt{sequence[u]} and \texttt{sequence[v]} may equal 1.}}
\end{promptbox}

\subsubsection{Maximum Cut (Max-Cut)}
\begin{promptbox}[unbreakable]
\promptfield{Objective:}{\promptrole{Partition the vertices of a weighted undirected graph into two disjoint sets so that the total weight of edges crossing between them is maximized.  The minimization engine returns the negative cut weight, so a lower value represents a heavier cut.}}
\promptfield{Encoding:}{\texttt{state.sequence} is a binary partition-assignment vector of length \texttt{num\_nodes}; \texttt{sequence[i] = 0} or 1 identifies the partition containing vertex $i$.}
\promptfield{Decoder:}{Traverse the graph stored in CSR form through \texttt{offsets}, \texttt{neighbors}, and \texttt{edge\_weights}; accumulate the weight of every edge whose endpoints have different assignments, and return the negative total cut weight.}
\promptfield{Environment~keys:}{\promptinput{num\_nodes, num\_edges, neighbors, edge\_weights, offsets}.}
\promptfield{Constraints:}{\promptconstraint{The assignment vector must have length \texttt{num\_nodes}, contain only 0 or 1, and assign every vertex to exactly one partition.  Vertex IDs, CSR offsets, neighbor indices, and edge-weight accesses must remain within their valid ranges.}}
\end{promptbox}

\subsubsection{Open Shop Scheduling Problem (OSSP)}
\begin{promptbox}[unbreakable]
\promptfield{Objective:}{\promptrole{Minimize makespan while processing every job exactly once on every machine; machine order is free, but operations sharing a job or machine cannot overlap.}}
\promptfield{Encoding:}{\texttt{state.sequence} is a permutation of all operation IDs $[0,\ldots,NM-1]$, where \texttt{job\_id = op // num\_machines} and \texttt{mach\_id = op \% num\_machines}.}
\promptfield{Decoder:}{Use a greedy list scheduler in permutation order.  Start each operation at the maximum of its job's and machine's availability, add \texttt{ossp\_times[job,machine]}, update both availability values, and return the latest completion time.}
\promptfield{Environment~keys:}{\promptinput{num\_jobs, num\_machines, total\_ops, ossp\_times, upper\_bound, lower\_bound}.}
\promptfield{Constraints:}{\promptconstraint{The sequence must contain each operation ID in $[0,\texttt{total\_ops}-1]$ exactly once.  Operations are non-preemptive, and neither two operations of one job nor two operations on one machine may overlap.}}
\end{promptbox}

\subsubsection{Resource-Constrained Project Scheduling Problem (RCPSP)}
\begin{promptbox}[unbreakable]
\promptfield{Objective:}{\promptrole{Minimize project makespan while scheduling all activities subject to finish--start precedence and per-period renewable-resource capacities.  Activities 0 and $N-1$ are zero-duration, zero-demand dummy start and end nodes.}}
\promptfield{Encoding:}{\texttt{state.sequence} is an \texttt{int32} topological permutation of all activity IDs $[0,\ldots,N-1]$; for every precedence edge $(u,v)$, $u$ appears before $v$.}
\promptfield{Decoder:}{Apply the Serial Schedule Generation Scheme.  In sequence order, compute the latest predecessor finish, then find the earliest contiguous interval of the activity's duration during which every resource has sufficient residual capacity; allocate its requests there.  The latest finish time is the makespan.}
\promptfield{Environment~keys:}{\promptinput{num\_activities, num\_resources, durations, requests, capacities, adj\_matrix, edges, upper\_bound}.}
\promptfield{Constraints:}{\promptconstraint{The sequence must be a complete topological permutation.  Every predecessor must finish before its successor starts, and at every time unit the summed request for each renewable resource must not exceed its capacity.}}
\end{promptbox}

\subsubsection{Simple Assembly Line Balancing Problem Type 1 (SALBP-1)}
\begin{promptbox}[unbreakable]
\promptfield{Objective:}{\promptrole{For fixed cycle time $C$, assign every precedence-constrained task to the minimum number of sequential workstations.  The number of utilized workstations is the objective to be minimized,
while \texttt{upper\_bound} stores the theoretical lower bound
$\lceil \text{total time}/C\rceil$.}}
\promptfield{Encoding:}{\texttt{state.sequence} is an \texttt{int32} topological permutation of all task IDs $[0,\ldots,N-1]$; every predecessor appears before its successor.}
\promptfield{Decoder:}{Apply greedy First-Fit Station Assignment in sequence order.  For each task, set its earliest allowed station to the latest station of its predecessors, place it in the first station at or after that point whose remaining cycle time is sufficient, or open a new station.}
\promptfield{Environment~keys:}{\promptinput{num\_tasks, cycle\_time, times, adj\_matrix, edges, upper\_bound, total\_time}.}
\promptfield{Constraints:}{\promptconstraint{The sequence must be a complete topological permutation, each task must be assigned exactly once, every station load must not exceed \texttt{cycle\_time}, and each predecessor's station index must be no greater than its successor's.}}
\end{promptbox}
\section{Reinforcement-Learning Training and Credit Assignment}

\subsection{Autoregressive Factorization and Local Credit-Induced Optimization}
\label{sec:pipeline_policy_analysis}

This section derives the autoregressive factorization of the pipeline
policy and analyzes the optimization behavior of the credit-weighted
pipeline loss in the main text. We first show that the probability of
a complete pipeline factorizes into the conditional probabilities of
its operator-selection decisions. We then derive the corresponding
gradient update and establish its local effect on both the complete
pipeline and its individual actions.

\paragraph{Pipeline representation and masked action distribution.}
Recall that the $j$-th pipeline sampled at generation $t$ is
\begin{equation}
    P_{t,j}
    =
    \left(
        O_{\mathrm{in}},
        O_{t,j,1},
        \ldots,
        O_{t,j,L_{t,j}},
        O_{\mathrm{out}}
    \right).
    \label{eq:appendix_pipeline_definition}
\end{equation}
For notational convenience, define
\begin{equation}
    O_{t,j,0}:=O_{\mathrm{in}},
    \qquad
    O_{t,j,L_{t,j}+1}:=O_{\mathrm{out}},
    \label{eq:appendix_pipeline_sentinels}
\end{equation}
and let
\begin{equation}
    h_{t,j,i}
    :=
    \left(
        O_{t,j,0},
        O_{t,j,1},
        \ldots,
        O_{t,j,i-1}
    \right)
    \label{eq:appendix_pipeline_history}
\end{equation}
denote the complete decoding history before the $i$-th
operator-selection decision.

Let $\mathcal{M}_{t,j,i}$ denote the feasible action set at position
$i$. It contains the outgoing operators that preserve at least one
continuation to $O_{\mathrm{out}}$ within the remaining pipeline-length
budget. Given the pre-softmax policy logit
\(z_{\theta}(O \mid h_{t,j,i}, D_t)\), the masked conditional distribution is: 
\begin{equation}
    \pi_{\theta}
    \left(
        O\mid h_{t,j,i},D_t
    \right)
    =
    \begin{cases}
    \displaystyle
    \frac{
        \exp
        \left(
            z_{\theta}
            \left(
                O\mid h_{t,j,i},D_t
            \right)
        \right)
    }{
        \displaystyle
        \sum_{O'\in\mathcal{M}_{t,j,i}}
        \exp
        \left(
            z_{\theta}
            \left(
                O'\mid h_{t,j,i},D_t
            \right)
        \right)
    },
    &
    O\in\mathcal{M}_{t,j,i},
    \\[3ex]
    0,
    &
    O\notin\mathcal{M}_{t,j,i}.
    \end{cases}
    \label{eq:masked_pipeline_policy}
\end{equation}

\paragraph{Autoregressive factorization of the pipeline policy.}
Because the pipeline is generated sequentially, the probability of
the complete pipeline follows from the probability chain rule:
\begin{align}
    \pi_{\theta}^{\mathrm{pipe}}
    \left(
        P_{t,j}\mid D_t
    \right)
    &=
    \pi_{\theta}
    \left(
        O_{t,j,1}
        \mid
        h_{t,j,1},D_t
    \right)
    \nonumber
    \\
    &\quad\times
    \pi_{\theta}
    \left(
        O_{t,j,2}
        \mid
        h_{t,j,2},D_t
    \right)
    \times\cdots
    \nonumber
    \\
    &\quad\times
    \pi_{\theta}
    \left(
        O_{t,j,L_{t,j}+1}
        \mid
        h_{t,j,L_{t,j}+1},D_t
    \right).
    \label{eq:pipeline_policy_expanded}
\end{align}
Here, $\pi_\theta(O\mid h_{t,j,i},D_t)$ denotes the probability
induced by the constrained token-level decoder of generating the
complete alias associated with operator $O$. The quantity
$z_\theta(O\mid h_{t,j,i},D_t)$ is the corresponding operator-level
log-score. Equivalently,
\begin{equation}
    \pi_{\theta}^{\mathrm{pipe}}
    \left(
        P_{t,j}\mid D_t
    \right)
    =
    \prod_{i=1}^{L_{t,j}+1}
    \pi_{\theta}
    \left(
        O_{t,j,i}
        \mid
        h_{t,j,i},D_t
    \right).
    \label{eq:pipeline_policy_factorization}
\end{equation}
The product starts at $i=1$ because $O_{\mathrm{in}}$ is a fixed
sentinel rather than a sampled action. It ends at $L_{t,j}+1$ because
selecting $O_{\mathrm{out}}$ is also part of the generated pipeline.

Taking the logarithm of
Eq.~\eqref{eq:pipeline_policy_factorization} gives
\begin{equation}
    \log
    \pi_{\theta}^{\mathrm{pipe}}
    \left(
        P_{t,j}\mid D_t
    \right)
    =
    \sum_{i=1}^{L_{t,j}+1}
    \log
    \pi_{\theta}
    \left(
        O_{t,j,i}
        \mid
        h_{t,j,i},D_t
    \right).
    \label{eq:pipeline_log_probability_factorization}
\end{equation}
Thus, the log-probability of a complete pipeline is the sum of the
log-probabilities of all operator-selection decisions appearing in
that pipeline.

\paragraph{Credit-weighted surrogate optimization.}
During policy optimization, the evaluated pipeline credit is treated
as constant with respect to $\theta$. Define
\begin{equation}
    C_{t,j}
    :=
    \operatorname{sg}
    \left(
        \widetilde{R}_{t,j}
    \right),
    \label{eq:stop_gradient_pipeline_credit}
\end{equation}
where $\operatorname{sg}(\cdot)$ denotes the stop-gradient operator.
The credit-weighted pipeline loss is
\begin{equation}
    \mathcal{L}_{\mathrm{pipe}}(\theta)
    =
    -\frac{1}{N}
    \sum_{j=1}^{N}
    C_{t,j}
    \log
    \pi_{\theta}^{\mathrm{pipe}}
    \left(
        P_{t,j}\mid D_t
    \right),
    \label{eq:appendix_pipeline_loss}
\end{equation}
where \(N\) denotes the number of candidate pipelines sampled and
evaluated at generation \(t\), \(\theta\) denotes the trainable parameters of the pipeline policy. Substituting
Eq.~\eqref{eq:pipeline_log_probability_factorization}
into Eq.~\eqref{eq:appendix_pipeline_loss} yields
\begin{equation}
    \mathcal{L}_{\mathrm{pipe}}(\theta)
    =
    -\frac{1}{N}
    \sum_{j=1}^{N}
    \sum_{i=1}^{L_{t,j}+1}
    C_{t,j}
    \log
    \pi_{\theta}
    \left(
        O_{t,j,i}
        \mid
        h_{t,j,i},D_t
    \right).
    \label{eq:action_decomposed_pipeline_loss}
\end{equation}
Therefore, the terminal credit of a complete candidate algorithm is
propagated to every operator-selection decision in its sampled
pipeline.

Since $C_{t,j}$ is detached from the computation graph, differentiating
Eq.~\eqref{eq:appendix_pipeline_loss} gives
\begin{equation}
    \nabla_{\theta}
    \mathcal{L}_{\mathrm{pipe}}(\theta)
    =
    -\frac{1}{N}
    \sum_{j=1}^{N}
    C_{t,j}
    \nabla_{\theta}
    \log
    \pi_{\theta}^{\mathrm{pipe}}
    \left(
        P_{t,j}\mid D_t
    \right).
    \label{eq:pipeline_loss_gradient}
\end{equation}
A gradient-descent update with learning rate $\eta>0$ is consequently
\begin{align}
    \theta^{+}
    &=
    \theta
    -
    \eta
    \nabla_{\theta}
    \mathcal{L}_{\mathrm{pipe}}(\theta)
    \nonumber
    \\
    &=
    \theta
    +
    \frac{\eta}{N}
    \sum_{j=1}^{N}
    C_{t,j}
    \nabla_{\theta}
    \log
    \pi_{\theta}^{\mathrm{pipe}}
    \left(
        P_{t,j}\mid D_t
    \right).
    \label{eq:pipeline_parameter_update}
\end{align}

The following theorem formalizes the local contribution of an
individual sampled pipeline to this update.
\begin{theorem}[Local reinforcement effect of pipeline credit]
\label{thm:local_pipeline_credit}
Fix a sampled pipeline $P_{t,j}$ and define its log-probability as
\begin{equation}
    f_{t,j}(\theta)
    :=
    \log
    \pi_{\theta}^{\mathrm{pipe}}
    \left(
        P_{t,j}\mid D_t
    \right).
    \label{eq:pipeline_log_probability_function}
\end{equation}
 Ignoring the positive batch-scaling factor $1/N$, the
gradient-descent direction contributed by this pipeline is
\begin{equation}
    d_{t,j}
    :=
    C_{t,j}
    \nabla_{\theta}
    f_{t,j}(\theta).
    \label{eq:individual_pipeline_direction}
\end{equation}
Then the directional derivative of the pipeline log-probability
satisfies
\begin{equation}
    \left.
    \frac{\mathrm{d}}{\mathrm{d}\eta}
    f_{t,j}
    \left(
        \theta+\eta d_{t,j}
    \right)
    \right|_{\eta=0}
    =
    C_{t,j}
    \left\|
        \nabla_{\theta}
        f_{t,j}(\theta)
    \right\|_2^2.
    \label{eq:pipeline_directional_derivative}
\end{equation}
Consequently, whenever
$\nabla_{\theta}f_{t,j}(\theta)\neq 0$,
a positive credit $C_{t,j}>0$ gives a locally increasing direction
for the sampled pipeline's log-probability, whereas a negative credit
$C_{t,j}<0$ gives a locally decreasing direction.
\end{theorem}

\begin{proof}
By the chain rule and the definition of $d_{t,j}$,
\begin{align}
    \left.
    \frac{\mathrm{d}}{\mathrm{d}\eta}
    f_{t,j}
    \left(
        \theta+\eta d_{t,j}
    \right)
    \right|_{\eta=0}
    &=
    \nabla_{\theta}
    f_{t,j}(\theta)^{\top}
    d_{t,j}
    \\
    &=
    C_{t,j}
    \nabla_{\theta}
    f_{t,j}(\theta)^{\top}
    \nabla_{\theta}
    f_{t,j}(\theta)
    \\
    &=
    C_{t,j}
    \left\|
        \nabla_{\theta}
        f_{t,j}(\theta)
    \right\|_2^2.
\end{align}
Since the squared gradient norm is nonnegative, and is strictly
positive whenever the gradient is nonzero, the sign of the
directional derivative is determined by the sign of $C_{t,j}$.
\end{proof}
\paragraph{Action-level effect of pipeline credit.}
Theorem~\ref{thm:local_pipeline_credit} characterizes the
first-order effect of an individual pipeline credit on the
log-probability of the complete sampled pipeline.
The following corollary further describes how the same credit affects
the operator selected at each decoding position.

\begin{corollary}[Action-level effect of pipeline credit]
\label{cor:action_level_credit}
Fix a sampled pipeline $P_{t,j}$ and a decoding position
$i\in\{1,\ldots,L_{t,j}+1\}$. For notational convenience, define
\begin{equation}
    \pi_{t,j,i}(O)
    :=
    \pi_\theta
    \left(
        O
        \mid
        h_{t,j,i},D_t
    \right),
    \qquad
    z_{t,j,i}(O)
    :=
    z_\theta
    \left(
        O
        \mid
        h_{t,j,i},D_t
    \right),
    \label{eq:local_policy_notation}
\end{equation}
for every feasible operator
$O\in\mathcal{M}_{t,j,i}$.
Let
\begin{equation}
    O_{t,j,i}^{\star}
    :=
    O_{t,j,i}
    \label{eq:selected_operator}
\end{equation}
denote the operator selected at position $i$.
The corresponding action-level loss term is
\begin{equation}
    \ell_{t,j,i}
    =
    -C_{t,j}
    \log
    \pi_{t,j,i}
    \left(
        O_{t,j,i}^{\star}
    \right).
    \label{eq:action_level_loss}
\end{equation}
Then, for every feasible operator
$O\in\mathcal{M}_{t,j,i}$,
\begin{equation}
    \frac{
        \partial \ell_{t,j,i}
    }{
        \partial z_{t,j,i}(O)
    }
    =
    -C_{t,j}
    \left[
        \mathbf{1}
        \left\{
            O=O_{t,j,i}^{\star}
        \right\}
        -
        \pi_{t,j,i}(O)
    \right].
    \label{eq:action_level_logit_gradient}
\end{equation}

Consequently, under a local gradient-descent step of size
$\eta>0$ in logit space, the selected operator satisfies
\begin{equation}
    z_{t,j,i}^{+}
    \left(
        O_{t,j,i}^{\star}
    \right)
    =
    z_{t,j,i}
    \left(
        O_{t,j,i}^{\star}
    \right)
    +
    \eta C_{t,j}
    \left[
        1-
        \pi_{t,j,i}
        \left(
            O_{t,j,i}^{\star}
        \right)
    \right],
    \label{eq:selected_operator_logit_update}
\end{equation}
whereas every unselected feasible operator
$O\neq O_{t,j,i}^{\star}$ satisfies
\begin{equation}
    z_{t,j,i}^{+}(O)
    =
    z_{t,j,i}(O)
    -
    \eta C_{t,j}
    \pi_{t,j,i}(O).
    \label{eq:unselected_operator_logit_update}
\end{equation}

Therefore, when $C_{t,j}>0$, the update increases the logit of
the selected operator and decreases the logits of all competing
feasible operators. When $C_{t,j}<0$, both update directions are
reversed.
\end{corollary}

\begin{proof}
Because the policy is a masked softmax over the feasible action set
$\mathcal{M}_{t,j,i}$, we have
\begin{equation}
    \pi_{t,j,i}(O)
    =
    \frac{
        \exp\left(
            z_{t,j,i}(O)
        \right)
    }{
        \displaystyle
        \sum_{O'\in\mathcal{M}_{t,j,i}}
        \exp\left(
            z_{t,j,i}(O')
        \right)
    },
    \qquad
    O\in\mathcal{M}_{t,j,i}.
    \label{eq:local_masked_softmax}
\end{equation}
Hence, the log-probability of the selected operator can be written as
\begin{equation}
    \log
    \pi_{t,j,i}
    \left(
        O_{t,j,i}^{\star}
    \right)
    =
    z_{t,j,i}
    \left(
        O_{t,j,i}^{\star}
    \right)
    -
    \log
    \sum_{O'\in\mathcal{M}_{t,j,i}}
    \exp\left(
        z_{t,j,i}(O')
    \right).
    \label{eq:selected_operator_log_probability}
\end{equation}
Differentiating Eq.~\eqref{eq:selected_operator_log_probability}
with respect to the logit of an arbitrary feasible operator $O$ gives
\begin{equation}
    \frac{
        \partial
        \log
        \pi_{t,j,i}
        \left(
            O_{t,j,i}^{\star}
        \right)
    }{
        \partial z_{t,j,i}(O)
    }
    =
    \mathbf{1}
    \left\{
        O=O_{t,j,i}^{\star}
    \right\}
    -
    \pi_{t,j,i}(O).
    \label{eq:softmax_log_gradient}
\end{equation}
Multiplying Eq.~\eqref{eq:softmax_log_gradient} by
$-C_{t,j}$ yields
Eq.~\eqref{eq:action_level_logit_gradient}.

For the selected operator
$O=O_{t,j,i}^{\star}$,
Eq.~\eqref{eq:action_level_logit_gradient} becomes
\begin{equation}
    \frac{
        \partial \ell_{t,j,i}
    }{
        \partial
        z_{t,j,i}
        \left(
            O_{t,j,i}^{\star}
        \right)
    }
    =
    -C_{t,j}
    \left[
        1-
        \pi_{t,j,i}
        \left(
            O_{t,j,i}^{\star}
        \right)
    \right].
    \label{eq:selected_operator_gradient}
\end{equation}
Applying the gradient-descent update
\begin{equation}
    z_{t,j,i}^{+}(O)
    =
    z_{t,j,i}(O)
    -
    \eta
    \frac{
        \partial \ell_{t,j,i}
    }{
        \partial z_{t,j,i}(O)
    }
    \label{eq:logit_gradient_descent}
\end{equation}
to Eq.~\eqref{eq:selected_operator_gradient} gives
Eq.~\eqref{eq:selected_operator_logit_update}.

For an unselected feasible operator
$O\neq O_{t,j,i}^{\star}$,
Eq.~\eqref{eq:action_level_logit_gradient} reduces to
\begin{equation}
    \frac{
        \partial \ell_{t,j,i}
    }{
        \partial z_{t,j,i}(O)
    }
    =
    C_{t,j}
    \pi_{t,j,i}(O).
    \label{eq:unselected_operator_gradient}
\end{equation}
Substituting Eq.~\eqref{eq:unselected_operator_gradient}
into Eq.~\eqref{eq:logit_gradient_descent} gives
Eq.~\eqref{eq:unselected_operator_logit_update}.

Finally, since
\begin{equation}
    1-
    \pi_{t,j,i}
    \left(
        O_{t,j,i}^{\star}
    \right)
    \geq 0,
    \qquad
    \pi_{t,j,i}(O)\geq 0,
\end{equation}
the directions of both updates are determined by the sign of
$C_{t,j}$. A positive credit increases the selected logit and
decreases every competing feasible logit, while a negative credit
reverses these directions.
\end{proof}

\begin{remark}[Scope of the action-level result]
\label{rem:action_level_scope}
Corollary~\ref{cor:action_level_credit} describes a local
gradient-descent update in the logits associated with one fixed
decoding history. In the implemented policy, the logits are produced
by shared parameters $\theta$, and gradients from multiple decoding
positions and sampled pipelines are aggregated in a single update.
Consequently, the corollary does not imply that the probability of
every selected operator changes monotonically after an arbitrary
finite parameter update. It instead characterizes the intended
first-order contribution of an individual action-level loss term.
\end{remark}

\subsection{Statistical Analysis of Contextual Credit Estimation}
\label{sec:first_order_finite_sample}

This section analyzes how the context order used for implementation-level
credit assignment affects finite-sample conditional-mean estimation. We
compare zero-order, first-order, higher-order, and full-path contextual
estimators under a fixed finite sample of activated implementation
positions.

\paragraph{Context hierarchy.}
Consider \(n\) activated implementation positions indexed by
\(s\in\{1,\ldots,n\}\). Let \(\mathcal H\) denote the space of possible
implementation histories, and let \(H_s\in\mathcal H\) denote the complete
implementation history up to and including position \(s\). Let \(Y_s\)
denote the standardized terminal credit assigned to that position.

For each context order \(r\in\{0,1,\ldots,R\}\), let
\(\mathcal X_r\) denote the ambient space of possible order-\(r\) contexts,
and let
\begin{equation}
    \phi_r:\mathcal H\to\mathcal X_r
    \label{eq:d2-context-map}
\end{equation}
be a deterministic context-extraction map. Define
\begin{equation}
    X_s^{(r)}
    :=
    \phi_r(H_s)
    \in
    \mathcal X_r.
    \label{eq:d2-context}
\end{equation}
In particular,
\begin{equation}
    X_s^{(0)}=C_s,
    \qquad
    X_s^{(1)}=(C_s^{-},C_s),
    \label{eq:d2-zero-first-contexts}
\end{equation}
where \(C_s^{-}\) and \(C_s\) denote the predecessor and current
implementations, respectively. For \(r\ge2\), \(X_s^{(r)}\) contains a
longer implementation suffix, while \(X_s^{(R)}\) may represent the complete
implementation history.

We assume that the context representations form nested refinements. For every
\(r\in\{0,\ldots,R-1\}\), there exists a deterministic projection
\begin{equation}
    \rho_{r+1\to r}:
    \mathcal X_{r+1}\to\mathcal X_r
    \label{eq:d2-adjacent-projection}
\end{equation}
such that
\begin{equation}
    X_s^{(r)}
    =
    \rho_{r+1\to r}\!\left(X_s^{(r+1)}\right)
    \qquad
    \text{for every }s.
    \label{eq:d2-nested-refinement}
\end{equation}
For \(r\ge1\), let
\(\rho_{r\to1}:\mathcal X_r\to\mathcal X_1\) denote the corresponding
composition of adjacent projections, with
\(\rho_{1\to1}\) equal to the identity. Hence,
\begin{equation}
    X_s^{(1)}
    =
    \rho_{r\to1}\!\left(X_s^{(r)}\right),
    \qquad
    r\ge1.
    \label{eq:d2-first-order-projection}
\end{equation}

The set of order-\(r\) contexts observed in the realized sample is
\begin{equation}
    \mathcal X_r^{(n)}
    :=
    \left\{
        X_s^{(r)}:
        s=1,\ldots,n
    \right\}
    \subseteq
    \mathcal X_r.
    \label{eq:d2-active-context-set}
\end{equation}
For each \(x\in\mathcal X_r^{(n)}\), define
\begin{equation}
    \mathcal I_r(x)
    :=
    \left\{
        s:
        X_s^{(r)}=x
    \right\},
    \qquad
    n_r(x)
    :=
    |\mathcal I_r(x)|,
    \label{eq:d2-cell-index-count}
\end{equation}
and let
\begin{equation}
    d_r
    :=
    |\mathcal X_r^{(n)}|
    \label{eq:d2-active-cell-count}
\end{equation}
denote the number of active order-\(r\) context cells. By construction,
\(n_r(x)\ge1\) for every \(x\in\mathcal X_r^{(n)}\).

\paragraph{First-order sufficiency and fixed-design observation model.}
We assume that the first-order context is sufficient for the conditional
expected terminal credit:
\begin{equation}
    \mathbb E[Y_s\mid H_s]
    =
    \mu\!\left(X_s^{(1)}\right),
    \label{eq:d2-first-order-sufficiency}
\end{equation}
where
\begin{equation}
    \mu(x)
    :=
    \mathbb E[Y_s\mid X_s^{(1)}=x],
    \qquad
    x\in\mathcal X_1,
    \label{eq:d2-mu-definition}
\end{equation}
is assumed not to depend on the position index \(s\).

For the exact finite-sample calculation below, we additionally condition on
the realized context design
\[
    \mathbf X
    :=
    \left\{
        X_s^{(r)}:
        s=1,\ldots,n,\;
        r=0,\ldots,R
    \right\}
\]
and assume the homoskedastic, conditionally uncorrelated noise model
\begin{equation}
    Y_s
    =
    \mu\!\left(X_s^{(1)}\right)
    +
    \epsilon_s,
    \label{eq:d2-observation-model}
\end{equation}
with
\begin{equation}
    \mathbb E[\epsilon_s\mid\mathbf X]
    =
    0,
    \qquad
    \operatorname{Cov}
    \!\left(
        \epsilon_s,\epsilon_t
        \mid
        \mathbf X
    \right)
    =
    \sigma^2\mathbf 1\{s=t\}.
    \label{eq:d2-noise-assumption}
\end{equation}
All expectations and variances in the remainder of this subsection are
conditional on \(\mathbf X\); this conditioning is suppressed for
notational simplicity.

For each context order \(r\), define the empirical cell-mean estimator
\begin{equation}
    \widehat q_r(x)
    :=
    \frac{1}{n_r(x)}
    \sum_{s\in\mathcal I_r(x)}
    Y_s,
    \qquad
    x\in\mathcal X_r^{(n)}.
    \label{eq:d2-cell-mean-estimator}
\end{equation}

\paragraph{Design-weighted conditional-mean estimation risk.}
We evaluate each estimator at the observed design points:
\begin{equation}
    \mathcal R_r
    :=
    \frac{1}{n}
    \sum_{s=1}^{n}
    \mathbb E
    \left[
        \left(
            \widehat q_r\!\left(X_s^{(r)}\right)
            -
            \mu\!\left(X_s^{(1)}\right)
        \right)^2
    \right].
    \label{eq:d2-risk}
\end{equation}

For a zero-order context \(u\in\mathcal X_0^{(n)}\), define its
design-weighted first-order mean as
\begin{equation}
    \nu(u)
    :=
    \frac{1}{n_0(u)}
    \sum_{s\in\mathcal I_0(u)}
    \mu\!\left(X_s^{(1)}\right).
    \label{eq:d2-zero-order-mean}
\end{equation}
The representation-induced squared approximation bias of the zero-order estimator is
\begin{equation}
    B_0
    :=
    \frac{1}{n}
    \sum_{s=1}^{n}
    \left[
        \mu\!\left(X_s^{(1)}\right)
        -
        \nu\!\left(X_s^{(0)}\right)
    \right]^2.
    \label{eq:d2-zero-order-bias}
\end{equation}

\begin{theorem}[Finite-sample bias--variance comparison under first-order sufficiency]
\label{thm:d2-context-risk}
Under
Eqs.~\eqref{eq:d2-nested-refinement},
\eqref{eq:d2-first-order-sufficiency},
\eqref{eq:d2-observation-model}, and
\eqref{eq:d2-noise-assumption}, the risks of the empirical cell-mean
estimators satisfy
\begin{equation}
    \mathcal R_0
    =
    B_0
    +
    \frac{\sigma^2d_0}{n},
    \label{eq:d2-zero-order-risk}
\end{equation}
and, for every \(r\ge1\),
\begin{equation}
    \mathcal R_r
    =
    \frac{\sigma^2d_r}{n}.
    \label{eq:d2-positive-order-risk}
\end{equation}

The nested refinement property implies
\begin{equation}
    d_0
    \le
    d_1
    \le
    \cdots
    \le
    d_R.
    \label{eq:d2-cell-count-monotonicity}
\end{equation}
Consequently,
\begin{equation}
    \mathcal R_1
    \le
    \mathcal R_r,
    \qquad
    r\ge1.
    \label{eq:d2-first-versus-higher}
\end{equation}
Moreover,
\begin{equation}
    \mathcal R_1
    \le
    \mathcal R_0
    \quad\Longleftrightarrow\quad
    B_0
    \ge
    \frac{\sigma^2(d_1-d_0)}{n}.
    \label{eq:d2-first-versus-zero}
\end{equation}
If the inequality in Eq.~\eqref{eq:d2-first-versus-zero} is strict and
\begin{equation}
    d_r>d_1,
    \qquad
    r\ge2,
    \label{eq:d2-strict-refinement}
\end{equation}
then the first-order estimator is the unique risk minimizer among the
empirical cell-mean estimators considered here:
\begin{equation}
    \mathcal R_1
    <
    \mathcal R_r,
    \qquad
    r\neq1.
    \label{eq:d2-unique-minimizer}
\end{equation}
\end{theorem}

\begin{proof}
Fix \(r\ge1\) and an active context
\(x\in\mathcal X_r^{(n)}\). By
Eq.~\eqref{eq:d2-first-order-projection}, every
\(s\in\mathcal I_r(x)\) has the same first-order context
\(\rho_{r\to1}(x)\). Therefore,
\begin{align}
    \widehat q_r(x)
    &=
    \frac{1}{n_r(x)}
    \sum_{s\in\mathcal I_r(x)}
    \left[
        \mu\!\left(X_s^{(1)}\right)
        +
        \epsilon_s
    \right]
    \notag\\
    &=
    \mu\!\left(\rho_{r\to1}(x)\right)
    +
    \frac{1}{n_r(x)}
    \sum_{s\in\mathcal I_r(x)}
    \epsilon_s.
    \label{eq:d2-positive-order-expansion}
\end{align}
Hence,
\begin{equation}
    \mathbb E[\widehat q_r(x)]
    =
    \mu\!\left(\rho_{r\to1}(x)\right),
    \label{eq:d2-positive-order-unbiased}
\end{equation}
and Eq.~\eqref{eq:d2-noise-assumption} gives
\begin{equation}
    \operatorname{Var}[\widehat q_r(x)]
    =
    \frac{\sigma^2}{n_r(x)}.
    \label{eq:d2-positive-order-variance}
\end{equation}
Grouping the risk in Eq.~\eqref{eq:d2-risk} by active context cells yields
\begin{align}
    \mathcal R_r
    &=
    \frac{1}{n}
    \sum_{x\in\mathcal X_r^{(n)}}
    \sum_{s\in\mathcal I_r(x)}
    \mathbb E
    \left[
        \left(
            \widehat q_r(x)
            -
            \mu\!\left(X_s^{(1)}\right)
        \right)^2
    \right]
    \notag\\
    &=
    \frac{1}{n}
    \sum_{x\in\mathcal X_r^{(n)}}
    n_r(x)
    \operatorname{Var}[\widehat q_r(x)]
    \notag\\
    &=
    \frac{1}{n}
    \sum_{x\in\mathcal X_r^{(n)}}
    n_r(x)
    \frac{\sigma^2}{n_r(x)}
    =
    \frac{\sigma^2d_r}{n}.
    \label{eq:d2-positive-order-risk-proof}
\end{align}
This proves Eq.~\eqref{eq:d2-positive-order-risk}.

We next consider the zero-order estimator. For any
\(u\in\mathcal X_0^{(n)}\),
\begin{align}
    \widehat q_0(u)
    &=
    \frac{1}{n_0(u)}
    \sum_{s\in\mathcal I_0(u)}
    \left[
        \mu\!\left(X_s^{(1)}\right)
        +
        \epsilon_s
    \right]
    \notag\\
    &=
    \nu(u)
    +
    \overline\epsilon_u,
    \label{eq:d2-zero-order-expansion}
\end{align}
where
\begin{equation}
    \overline\epsilon_u
    :=
    \frac{1}{n_0(u)}
    \sum_{s\in\mathcal I_0(u)}
    \epsilon_s.
    \label{eq:d2-zero-order-average-noise}
\end{equation}
By Eq.~\eqref{eq:d2-noise-assumption},
\begin{equation}
    \mathbb E[\overline\epsilon_u]
    =
    0,
    \qquad
    \operatorname{Var}[\overline\epsilon_u]
    =
    \frac{\sigma^2}{n_0(u)}.
    \label{eq:d2-zero-order-noise-moments}
\end{equation}
For every \(s\in\mathcal I_0(u)\),
\begin{equation}
    \widehat q_0(u)
    -
    \mu\!\left(X_s^{(1)}\right)
    =
    \nu(u)
    -
    \mu\!\left(X_s^{(1)}\right)
    +
    \overline\epsilon_u.
    \label{eq:d2-zero-order-error}
\end{equation}
Squaring, taking expectations, and using
\(\mathbb E[\overline\epsilon_u]=0\) give
\begin{align}
    \mathcal R_0
    &=
    \frac{1}{n}
    \sum_{u\in\mathcal X_0^{(n)}}
    \sum_{s\in\mathcal I_0(u)}
    \left[
        \nu(u)
        -
        \mu\!\left(X_s^{(1)}\right)
    \right]^2
    \notag\\
    &\quad+
    \frac{1}{n}
    \sum_{u\in\mathcal X_0^{(n)}}
    n_0(u)
    \frac{\sigma^2}{n_0(u)}
    \notag\\
    &=
    B_0
    +
    \frac{\sigma^2d_0}{n},
    \label{eq:d2-zero-order-risk-proof}
\end{align}
which proves Eq.~\eqref{eq:d2-zero-order-risk}.

It remains to compare the active cell counts. For every
\(r\in\{0,\ldots,R-1\}\), the restriction of
\(\rho_{r+1\to r}\) to \(\mathcal X_{r+1}^{(n)}\) is surjective onto
\(\mathcal X_r^{(n)}\): every active order-\(r\) context is realized by at
least one sample position, whose order-\((r+1)\) context projects to it.
Therefore,
\begin{equation}
    d_r
    \le
    d_{r+1},
    \label{eq:d2-cell-count-step}
\end{equation}
which proves Eq.~\eqref{eq:d2-cell-count-monotonicity}. Combining this
monotonicity with Eq.~\eqref{eq:d2-positive-order-risk} proves
Eq.~\eqref{eq:d2-first-versus-higher}.

Finally, using
Eqs.~\eqref{eq:d2-zero-order-risk} and
\eqref{eq:d2-positive-order-risk},
\begin{align}
    \mathcal R_1
    \le
    \mathcal R_0
    &\quad\Longleftrightarrow\quad
    \frac{\sigma^2d_1}{n}
    \le
    B_0+\frac{\sigma^2d_0}{n}
    \notag\\
    &\quad\Longleftrightarrow\quad
    B_0
    \ge
    \frac{\sigma^2(d_1-d_0)}{n}.
    \label{eq:d2-zero-first-algebra}
\end{align}
This proves Eq.~\eqref{eq:d2-first-versus-zero}. If the inequality is
strict, then \(\mathcal R_1<\mathcal R_0\). If additionally
Eq.~\eqref{eq:d2-strict-refinement} holds, then
Eq.~\eqref{eq:d2-positive-order-risk} gives
\(\mathcal R_1<\mathcal R_r\) for every \(r\ge2\), proving
Eq.~\eqref{eq:d2-unique-minimizer}.
\end{proof}

\paragraph{Interpretation.}
Theorem~\ref{thm:d2-context-risk} separates the effect of context order into
representation-induced approximation bias squared and estimation variance.
Zero-order credit pools positions that share the current implementation but
have different predecessors, which produces the bias term \(B_0\).
Under first-order sufficiency, every context order \(r\ge1\) has zero
approximation bias. Higher-order and full-path contexts induce partitions
that are no coarser than the first-order partition and therefore cannot
reduce the number of active credit cells. Under the homoskedastic,
conditionally uncorrelated model, their finite-sample variance is consequently
no smaller than that of the first-order estimator, and is strictly larger
whenever the refinement increases the number of active cells.

\paragraph{Relation to the online credit update.}
Suppose that the \(k\)-th activation of a context \(x\) yields credit
\(Y_k(x)\). The recursion
\begin{equation}
    Q_k(x)
    =
    \left(1-\frac{1}{k}\right)Q_{k-1}(x)
    +
    \frac{1}{k}Y_k(x)
    \label{eq:d2-online-sample-mean}
\end{equation}
satisfies
\begin{equation}
    Q_k(x)
    =
    \frac{1}{k}
    \sum_{m=1}^{k}Y_m(x),
    \qquad
    k\ge1,
    \label{eq:d2-online-sample-mean-identity}
\end{equation}
regardless of the initialization \(Q_0(x)\). Hence, the estimator in
Eq.~\eqref{eq:d2-cell-mean-estimator} is obtained exactly when the
context-specific learning rate is the reciprocal activation count. With a
constant learning rate, the update becomes an exponentially weighted
estimator; in that case, Theorem~\ref{thm:d2-context-risk} characterizes the
context-partition bias--variance tradeoff but not the exact finite-sample risk
of the implemented constant-step estimator.

\paragraph{Scope of the noise model.}
The conditionally uncorrelated position-level noise model in
Eq.~\eqref{eq:d2-noise-assumption} is an idealized assumption used to isolate
the statistical effect of context-partition granularity. In the implemented
algorithm, multiple transitions, including repeated occurrences of the same
transition, may receive the same standardized terminal credit from one
candidate. Generation-level standardization may also induce dependence
across observations. Under a general correlated-noise model,
\begin{equation}
    \operatorname{Var}[\widehat q_r(x)]
    =
    \frac{1}{n_r(x)^2}
    \sum_{s,t\in\mathcal I_r(x)}
    \operatorname{Cov}(\epsilon_s,\epsilon_t\mid\mathbf X),
    \label{eq:d2-correlated-variance}
\end{equation}
so additional covariance terms enter the finite-sample risk and the simplified
forms in
Eqs.~\eqref{eq:d2-zero-order-risk}--\eqref{eq:d2-positive-order-risk}
need not hold.

\paragraph{Scope of the first-order assumption.}
The result does not claim that first-order contextual credit is universally
optimal. Its comparison relies on
Eq.~\eqref{eq:d2-first-order-sufficiency}. If longer histories contain
additional information about the conditional expected terminal credit, then
higher-order estimators may reduce approximation bias, and that benefit must
be compared against their larger finite-sample estimation variance.
\subsection{Spectral Expressivity of Directed-Graph Algorithm Spaces}
\label{app:spectral-expressivity}

This section characterizes the structural expressivity of the directed-graph
representation through the number and asymptotic growth of graph-valid
operator pipelines. Throughout this section, a pipeline refers only to an
operator-level directed walk in the current graph. Different code
implementations of the same operator and execution-level validity are not
distinguished. Thus, the analysis concerns the structural pipeline space
rather than the number of fully instantiated executable algorithms.

We first express the number of fixed-length pipelines exactly using powers of
the graph adjacency matrix. We then contrast the finite-depth behavior of
directed acyclic graphs with the unbounded iterative behavior enabled by
directed cycles. Finally, we show that the exponential limsup growth rate of
the pipeline space is determined by the spectral radius of the dominant
input--output-relevant strongly connected components.

\paragraph{Input--output-relevant graph.}
Fix a generation \(t\) and write the current directed operator graph as
\begin{equation}
    D=(V,E).
    \label{eq:d3-graph}
\end{equation}
Let
\begin{equation}
    s:=O_{\mathrm{in}},
    \qquad
    z:=O_{\mathrm{out}}
    \label{eq:d3-sentinels}
\end{equation}
denote the input and output sentinel nodes. We assume that \(s\neq z\) and
that \(D\) contains at least one directed path from \(s\) to \(z\), as
guaranteed by the graph-validity checks.

A vertex is called input--output relevant if it is reachable from \(s\) and
can reach \(z\). Define
\begin{equation}
    V_{\mathrm{io}}
    :=
    \left\{
        v\in V:
        s\rightsquigarrow v
        \ \text{and}\
        v\rightsquigarrow z
    \right\},
    \label{eq:d3-vio}
\end{equation}
and let
\begin{equation}
    D_{\mathrm{io}}
    :=
    D[V_{\mathrm{io}}]
    \label{eq:d3-dio}
\end{equation}
be the subgraph induced by these vertices. Vertices outside
\(V_{\mathrm{io}}\) cannot occur in any directed walk from \(s\) to \(z\)
and are therefore irrelevant to the following analysis.

Let
\begin{equation}
    A\in\{0,1\}^{|V_{\mathrm{io}}|\times |V_{\mathrm{io}}|}
    \label{eq:d3-adjacency-space}
\end{equation}
be the adjacency matrix of \(D_{\mathrm{io}}\), with entries
\begin{equation}
    A_{uv}
    :=
    \begin{cases}
        1, & (u,v)\in E,\\
        0, & (u,v)\notin E.
    \end{cases}
    \label{eq:d3-adjacency}
\end{equation}
For \(v\in V_{\mathrm{io}}\), let \(\mathbf e_v\) denote the corresponding
standard basis vector.

For every integer \(\ell\ge 0\), let \(W_\ell\) denote the number of directed
walks from \(s\) to \(z\) containing exactly \(\ell\) edges. A pipeline with
\(L\) functional operators contains \(L+1\) graph edges and is therefore
counted by \(W_{L+1}\).

\begin{lemma}[Exact counting of fixed-length pipelines]
\label{lem:d3-walk-count}
For every integer \(\ell\ge 0\),
\begin{equation}
    W_\ell
    =
    \mathbf e_s^\top A^\ell \mathbf e_z.
    \label{eq:d3-fixed-count}
\end{equation}
Consequently, when candidate pipelines are required to contain at least one
functional operator, the number of pipelines containing at most
\(L_{\max}\) functional operators is
\begin{equation}
    W_{\le L_{\max}}
    =
    \sum_{\ell=2}^{L_{\max}+1}
    \mathbf e_s^\top A^\ell \mathbf e_z.
    \label{eq:d3-bounded-count}
\end{equation}
\end{lemma}

\begin{proof}
We use the standard identity
\begin{equation}
    (A^\ell)_{uv}
    =
    \text{the number of length-\(\ell\) directed walks from \(u\) to \(v\)}.
    \label{eq:d3-standard-walk-identity}
\end{equation}
For \(\ell=0\), \(A^0=I\), which counts the unique zero-edge walk from each
vertex to itself. For \(\ell=1\), the identity follows directly from the
definition of the adjacency matrix.

Assume that the identity holds for some \(\ell\ge 1\). Then
\begin{equation}
    (A^{\ell+1})_{uv}
    =
    \sum_{w\in V_{\mathrm{io}}}
    (A^\ell)_{uw}A_{wv}.
    \label{eq:d3-induction}
\end{equation}
By the induction hypothesis, \((A^\ell)_{uw}\) counts all length-\(\ell\)
walks from \(u\) to \(w\). Multiplication by \(A_{wv}\) retains exactly those
walks that can be extended by the edge \((w,v)\). Summing over \(w\) counts
all length-\((\ell+1)\) walks from \(u\) to \(v\).

Setting \(u=s\) and \(v=z\) proves
Eq.~\eqref{eq:d3-fixed-count}. Summing the exact counts over all permitted
lengths proves Eq.~\eqref{eq:d3-bounded-count}.
\end{proof}

\begin{proposition}[Finite depth of DAG representations]
\label{prop:d3-dag-depth}
Suppose that \(D_{\mathrm{io}}\) is a directed acyclic graph and let
\begin{equation}
    n_{\mathrm{io}}
    :=
    |V_{\mathrm{io}}|.
    \label{eq:d3-nio}
\end{equation}
Then
\begin{equation}
    A^{n_{\mathrm{io}}}=0,
    \label{eq:d3-nilpotent}
\end{equation}
and hence
\begin{equation}
    W_\ell=0,
    \qquad
    \ell\ge n_{\mathrm{io}}.
    \label{eq:d3-dag-zero}
\end{equation}
Equivalently, a graph-valid pipeline in a fixed-size DAG contains at most
\(n_{\mathrm{io}}-2\) functional operators.

By contrast, if \(D_{\mathrm{io}}\) contains a directed cycle, then it admits
input-to-output walks of unbounded length.
\end{proposition}

\begin{proof}
A directed walk in a DAG cannot visit the same vertex twice, because a
repeated vertex would induce a directed cycle. Therefore, every directed walk
in an \(n_{\mathrm{io}}\)-vertex DAG contains at most
\(n_{\mathrm{io}}-1\) edges. It follows that every entry of
\(A^{n_{\mathrm{io}}}\) is zero, proving
Eq.~\eqref{eq:d3-nilpotent}. Lemma~\ref{lem:d3-walk-count} then gives
Eq.~\eqref{eq:d3-dag-zero}. Since a pipeline with \(L\) functional operators
contains \(L+1\) edges, its number of functional operators is at most
\(n_{\mathrm{io}}-2\).

Now suppose that \(D_{\mathrm{io}}\) contains a directed cycle. Every vertex
on that cycle is reachable from \(s\) and can reach \(z\). Concatenating a
fixed entrance path, an arbitrary number of traversals of the cycle, and a
fixed exit path produces input-to-output walks of unbounded length.
\end{proof}

\paragraph{Spectral growth rate.}
Let
\begin{equation}
    \mathfrak S_{\mathrm{io}}
    \label{eq:d3-scc-collection}
\end{equation}
denote the collection of strongly connected components of
\(D_{\mathrm{io}}\). For each \(S\in\mathfrak S_{\mathrm{io}}\), let
\begin{equation}
    A_S
    \label{eq:d3-scc-adjacency}
\end{equation}
be the adjacency matrix of the subgraph induced by \(S\). Define the
input--output-relevant spectral radius as
\begin{equation}
    \rho^\star
    :=
    \max_{S\in\mathfrak S_{\mathrm{io}}}
    \rho(A_S),
    \label{eq:d3-rho-star}
\end{equation}
where \(\rho(\cdot)\) denotes the matrix spectral radius.

\begin{theorem}[Spectral growth of the structural pipeline space]
\label{thm:d3-spectral-growth}
Suppose that \(D_{\mathrm{io}}\) admits input-to-output walks of unbounded
length. Then
\begin{equation}
    \limsup_{\ell\to\infty}
    W_\ell^{1/\ell}
    =
    \rho^\star.
    \label{eq:d3-spectral-growth}
\end{equation}
Thus, \(\rho^\star\) is the exponential limsup growth rate of the number of
fixed-length structural pipelines.
\end{theorem}

\begin{proof}
We first establish the upper bound. Arrange the strongly connected components
of \(D_{\mathrm{io}}\) in a topological order of the condensation graph. The
adjacency matrix \(A\) then has Frobenius block upper-triangular form, with
diagonal blocks \(A_S\). Therefore,
\begin{equation}
    \rho(A)
    =
    \max_{S\in\mathfrak S_{\mathrm{io}}}
    \rho(A_S)
    =
    \rho^\star.
    \label{eq:d3-rho-A}
\end{equation}
Using the induced matrix \(1\)-norm and the nonnegativity of \(A^\ell\),
\begin{equation}
    W_\ell
    =
    (A^\ell)_{sz}
    \le
    \|A^\ell\|_1.
    \label{eq:d3-upper-norm}
\end{equation}
By Gelfand's formula,
\begin{equation}
    \lim_{\ell\to\infty}
    \|A^\ell\|_1^{1/\ell}
    =
    \rho(A)
    =
    \rho^\star.
    \label{eq:d3-gelfand}
\end{equation}
Combining Eqs.~\eqref{eq:d3-upper-norm} and \eqref{eq:d3-gelfand} yields
\begin{equation}
    \limsup_{\ell\to\infty}
    W_\ell^{1/\ell}
    \le
    \rho^\star.
    \label{eq:d3-upper-bound}
\end{equation}

For the lower bound, choose a strongly connected component
\(S^\star\in\mathfrak S_{\mathrm{io}}\) satisfying
\begin{equation}
    \rho(A_{S^\star})=\rho^\star,
    \label{eq:d3-dominant-scc}
\end{equation}
and write
\begin{equation}
    B:=A_{S^\star}.
    \label{eq:d3-B}
\end{equation}
Choose arbitrary vertices \(u,v\in S^\star\). Since \(S^\star\) is
input--output relevant, there exists a path from \(s\) to \(u\) of some fixed
length \(a\), and a path from \(v\) to \(z\) of some fixed length \(b\).

The matrix \(B\) is irreducible because \(S^\star\) is strongly connected. A
standard consequence of the Perron--Frobenius theorem for irreducible
nonnegative matrices is
\begin{equation}
    \limsup_{m\to\infty}
    \left[(B^m)_{uv}\right]^{1/m}
    =
    \rho(B)
    =
    \rho^\star.
    \label{eq:d3-pf-limsup}
\end{equation}
Every length-\(m\) walk from \(u\) to \(v\) inside \(S^\star\) can be
concatenated with the fixed entrance and exit paths. Consequently,
\begin{equation}
    W_{a+m+b}
    \ge
    (B^m)_{uv}.
    \label{eq:d3-lower-concatenation}
\end{equation}
Taking \((a+m+b)\)-th roots and then the limit superior gives
\begin{align}
    \limsup_{\ell\to\infty} W_\ell^{1/\ell}
    &\ge
    \limsup_{m\to\infty}
    \left[(B^m)_{uv}\right]^{1/(a+m+b)}
    \notag\\
    &=
    \rho^\star.
    \label{eq:d3-lower-bound}
\end{align}
Together with Eq.~\eqref{eq:d3-upper-bound}, this proves
Eq.~\eqref{eq:d3-spectral-growth}.
\end{proof}

\paragraph{Primitive dominant-core case.}
The limit superior in Theorem~\ref{thm:d3-spectral-growth} is necessary in
general because a strongly connected component may be periodic. Under a
primitive dominant-core assumption, the pipeline count admits a sharper
asymptotic form.

\begin{corollary}[Primitive-core asymptotic pipeline count]
\label{cor:d3-primitive}
Suppose that \(D_{\mathrm{io}}\) contains a unique cyclic strongly connected
component \(S^\star\), and that
\begin{equation}
    B:=A_{S^\star}
    \label{eq:d3-primitive-B}
\end{equation}
is primitive. Let
\begin{equation}
    \rho:=\rho(B).
    \label{eq:d3-primitive-rho}
\end{equation}
Then there exists a constant \(\kappa>0\) such that
\begin{equation}
    W_\ell
    =
    \kappa\rho^\ell\bigl(1+o(1)\bigr)
    \qquad
    \text{as }\ell\to\infty.
    \label{eq:d3-primitive-asymptotic}
\end{equation}
Equivalently,
\begin{equation}
    \lim_{\ell\to\infty}
    \frac{W_\ell}{\rho^\ell}
    =
    \kappa.
    \label{eq:d3-primitive-limit}
\end{equation}
\end{corollary}

\begin{proof}
Because \(S^\star\) is the unique cyclic strongly connected component, the
subgraph induced by \(V_{\mathrm{io}}\setminus S^\star\) is acyclic.
Consequently, there are only finitely many entrance paths from \(s\) to
\(S^\star\) and finitely many exit paths from \(S^\star\) to \(z\).

Let \(\mathcal P_{\mathrm{in}}\) be the set of directed paths
\[
    p=(p_0,\ldots,p_a)
\]
such that \(p_0=s\), \(p_a\in S^\star\), and
\(p_i\notin S^\star\) for every \(i<a\). Thus, \(p\) first enters
\(S^\star\) at its final vertex. Let
\begin{equation}
    u(p):=p_a\in S^\star,
    \label{eq:d3-entry-vertex}
\end{equation}
and let \(|p|:=a\) denote its number of edges.

Similarly, let \(\mathcal P_{\mathrm{out}}\) be the set of directed paths
\[
    q=(q_0,\ldots,q_b)
\]
such that \(q_0\in S^\star\), \(q_b=z\), and
\(q_i\notin S^\star\) for every \(i>0\). Thus, after its initial vertex,
\(q\) lies entirely outside \(S^\star\). Let
\begin{equation}
    v(q):=q_0\in S^\star,
    \label{eq:d3-exit-vertex}
\end{equation}
and let \(|q|:=b\) denote its number of edges.

Since all walks that avoid \(S^\star\) lie in an acyclic subgraph, their
lengths are uniformly bounded. Hence, for all sufficiently large \(\ell\),
every length-\(\ell\) input-to-output walk passes through \(S^\star\).
Moreover, because the condensation graph is acyclic, once such a walk leaves
\(S^\star\), it cannot return. Therefore, every sufficiently long walk has a
unique decomposition into an entrance path, an internal walk in
\(S^\star\), and an exit path. It follows that
\begin{equation}
    W_\ell
    =
    \sum_{p\in\mathcal P_{\mathrm{in}}}
    \sum_{q\in\mathcal P_{\mathrm{out}}}
    \left(
        B^{\,\ell-|p|-|q|}
    \right)_{u(p),v(q)}
    \label{eq:d3-decomposition}
\end{equation}
for all sufficiently large \(\ell\), with terms having negative exponents
omitted.

Since \(B\) is primitive, the Perron--Frobenius theorem gives strictly
positive right and left eigenvectors \(\mathbf r\) and \(\mathbf h\)
satisfying
\begin{equation}
    B\mathbf r=\rho\mathbf r,
    \qquad
    \mathbf h^\top B=\rho\mathbf h^\top,
    \qquad
    \mathbf h^\top\mathbf r=1,
    \label{eq:d3-perron-vectors}
\end{equation}
and
\begin{equation}
    B^m
    =
    \rho^m\mathbf r\mathbf h^\top
    +
    o(\rho^m).
    \label{eq:d3-pf-asymptotic}
\end{equation}
Hence, for every fixed pair \(u,v\in S^\star\),
\begin{equation}
    (B^m)_{uv}
    =
    \rho^m r_u h_v
    +
    o(\rho^m).
    \label{eq:d3-entry-asymptotic}
\end{equation}
Substituting Eq.~\eqref{eq:d3-entry-asymptotic} into
Eq.~\eqref{eq:d3-decomposition}, and using the finiteness of
\(\mathcal P_{\mathrm{in}}\) and \(\mathcal P_{\mathrm{out}}\), gives
\begin{equation}
    W_\ell
    =
    \kappa\rho^\ell
    +
    o(\rho^\ell),
    \label{eq:d3-kappa-asymptotic}
\end{equation}
where
\begin{equation}
    \kappa
    :=
    \sum_{p\in\mathcal P_{\mathrm{in}}}
    \sum_{q\in\mathcal P_{\mathrm{out}}}
    \rho^{-(|p|+|q|)}
    r_{u(p)}h_{v(q)}.
    \label{eq:d3-kappa}
\end{equation}
The entrance and exit path sets are nonempty, \(\rho>0\), and the Perron
vectors are strictly positive. Therefore, \(\kappa>0\), proving
Eq.~\eqref{eq:d3-primitive-asymptotic}.
\end{proof}

\paragraph{Interpretation.}
Lemma~\ref{lem:d3-walk-count} gives an exact matrix characterization of the
bounded structural pipeline space. In a fixed-size DAG, the adjacency matrix
is nilpotent, so the maximum representable pipeline depth is intrinsically
finite. As the permitted length bound \(L_{\max}\) increases, a directed
cycle removes this intrinsic finite-depth restriction and permits repeated
execution of an operator motif without duplicating its nodes. For every fixed
\(L_{\max}\), however, the sampled pipeline space remains finite.

Theorem~\ref{thm:d3-spectral-growth} further distinguishes two forms of
cyclic expressivity. When
\begin{equation}
    \rho^\star=1,
    \label{eq:d3-rho-one}
\end{equation}
the graph can represent unbounded iterative behavior, but the number of
length-\(\ell\) pipelines need not grow exponentially. When
\begin{equation}
    \rho^\star>1,
    \label{eq:d3-rho-greater-one}
\end{equation}
the fixed-length structural pipeline count has exponential limsup growth rate
\(\rho^\star\). Under the primitive-core assumptions of
Corollary~\ref{cor:d3-primitive}, this strengthens to the ordinary asymptotic
relation
\[
    W_\ell\sim \kappa\rho^\ell.
\]
Thus, the spectral radius quantifies the structural branching capacity of the
directed-graph algorithm space.

These results concern representational expressivity rather than optimization
performance. A larger pipeline space does not by itself guarantee that a
finite-budget search procedure will discover a better algorithm; the
empirical benefit of the directed-graph representation is evaluated
separately through the structural ablation study.
\subsection{Transition-Credit Stability and Tracking}
\label{app:transition-credit-analysis}

This section analyzes the boundedness, expected-value convergence, and
tracking properties of the transition-credit update. Since a transition is
updated only when it is activated, the analysis is conducted on the
transition-specific activation clock rather than on the global generation
index.

\paragraph{Transition-specific activation clock.}
Fix an implementation transition $e$. Let
\[
\bigl(t_n(e),j_n(e),i_n(e)\bigr)
\]
denote, respectively, the generation index, candidate index, and
pipeline-position index of the $n$-th activation of $e$. Activations
are ordered according to the nested update order over $(t,j,i)$ in
Algorithm~1. The terminal credit observed at the $n$-th activation is
\begin{equation}
    Y_n(e)
    :=
    \widetilde{R}_{t_n(e),j_n(e)},
    \label{eq:d4-observed-credit}
\end{equation}
and $Q_n(e)$ denotes the credit of transition $e$ immediately after
this activation has been processed. 

If the same transition is activated multiple times within one sampled
candidate, these occurrences have different position indices
$i_n(e)$ and are treated as separate activation events, although they
receive the same candidate-level terminal credit
$\widetilde{R}_{t_n(e),j_n(e)}$.

Let $\{\mathcal F_n^e\}_{n\ge 0}$ denote the corresponding
activation-clock filtration, where $\mathcal F_n^e$ contains all
random variables revealed up to and including the $n$-th activation
of $e$.

The asymptotic statements below are understood for transitions that
are activated infinitely often. If a transition is activated only
finitely many times, the finite-$n$ identities and bounds apply up to
its final activation, after which its stored credit remains unchanged
unless the transition is removed.

We consider the constant-rate update used in the implementation,
\begin{equation}
    \alpha_n(e)\equiv\alpha,
    \qquad
    0<\alpha<1,
    \label{eq:d4-constant-rate}
\end{equation}
for which
\begin{equation}
    Q_n(e)
    =
    (1-\alpha)Q_{n-1}(e)
    +
    \alpha Y_n(e).
    \label{eq:d4-credit-update}
\end{equation}
For notational convenience, define
\begin{equation}
    \beta:=1-\alpha.
    \label{eq:d4-beta}
\end{equation}

\begin{proposition}[Exponential representation and boundedness]
\label{prop:d4-exponential-boundedness}
For every integer \(n\ge 1\),
\begin{equation}
    Q_n(e)
    =
    \beta^n Q_0(e)
    +
    \alpha\sum_{k=1}^{n}\beta^{n-k}Y_k(e).
    \label{eq:d4-exponential-representation}
\end{equation}
The coefficients in Eq.~\eqref{eq:d4-exponential-representation} are
nonnegative and sum to one:
\begin{equation}
    \beta^n
    +
    \alpha\sum_{k=1}^{n}\beta^{n-k}
    =
    1.
    \label{eq:d4-convex-weights}
\end{equation}
Consequently, if there exists \(R_{\max}>0\) such that
\begin{equation}
    |Q_0(e)|\le R_{\max},
    \qquad
    |Y_n(e)|\le R_{\max}
    \quad
    \text{for all }n,
    \label{eq:d4-bounded-observations}
\end{equation}
then
\begin{equation}
    |Q_n(e)|
    \le
    R_{\max}
    \quad
    \text{for all }n.
    \label{eq:d4-credit-boundedness}
\end{equation}
\end{proposition}

\begin{proof}
We prove Eq.~\eqref{eq:d4-exponential-representation} by induction. For
\(n=1\),
\begin{equation}
    Q_1(e)
    =
    \beta Q_0(e)
    +
    \alpha Y_1(e),
    \label{eq:d4-base-case}
\end{equation}
which agrees with Eq.~\eqref{eq:d4-exponential-representation}. Suppose that
the expansion holds for some \(n\ge 1\). Then
\begin{align}
    Q_{n+1}(e)
    &=
    \beta Q_n(e)+\alpha Y_{n+1}(e)
    \notag\\
    &=
    \beta\left[
        \beta^nQ_0(e)
        +
        \alpha\sum_{k=1}^{n}\beta^{n-k}Y_k(e)
    \right]
    +
    \alpha Y_{n+1}(e)
    \notag\\
    &=
    \beta^{n+1}Q_0(e)
    +
    \alpha\sum_{k=1}^{n+1}\beta^{n+1-k}Y_k(e).
    \label{eq:d4-induction-step}
\end{align}
Thus, the expansion holds for every \(n\ge 1\).

Moreover,
\begin{align}
    \beta^n
    +
    \alpha\sum_{k=1}^{n}\beta^{n-k}
    &=
    \beta^n
    +
    \alpha\sum_{j=0}^{n-1}\beta^j
    \notag\\
    &=
    \beta^n
    +
    \alpha\frac{1-\beta^n}{1-\beta}
    \notag\\
    &=
    \beta^n+1-\beta^n
    =
    1,
    \label{eq:d4-weight-sum-proof}
\end{align}
where \(1-\beta=\alpha\). Therefore, \(Q_n(e)\) is a convex combination
of \(Q_0(e)\) and \(Y_1(e),\ldots,Y_n(e)\). Under
Eq.~\eqref{eq:d4-bounded-observations}, this convex combination lies in
\([-R_{\max},R_{\max}]\), proving
Eq.~\eqref{eq:d4-credit-boundedness}.
\end{proof}

\paragraph{Stationary transition value.}
We first analyze a locally stationary phase in which the conditional expected
terminal credit associated with transition \(e\) is constant. Assume
\begin{equation}
    Y_n(e)
    =
    \mu_e+\xi_n(e),
    \label{eq:d4-stationary-observation}
\end{equation}
where \(\mu_e\) is a deterministic constant and
\(\{\xi_n(e)\}_{n\ge 1}\) is adapted to
\(\{\mathcal F_n^e\}_{n\ge 0}\) and satisfies the martingale-difference
condition
\begin{equation}
    \mathbb E\!\left[
        \xi_n(e)
        \mid
        \mathcal F_{n-1}^e
    \right]
    =
    0.
    \label{eq:d4-mds}
\end{equation}
For the exact mean-squared-error formula below, additionally assume
\begin{equation}
    \mathbb E\!\left[
        \xi_n(e)^2
        \mid
        \mathcal F_{n-1}^e
    \right]
    =
    \sigma_e^2.
    \label{eq:d4-conditional-variance}
\end{equation}

\begin{theorem}[Expected-value convergence and stationary estimation error]
\label{thm:d4-stationary}
Suppose that \(Q_0(e)\) is deterministic and
Eqs.~\eqref{eq:d4-stationary-observation}--\eqref{eq:d4-conditional-variance}
hold. Then
\begin{equation}
    \mathbb E[Q_n(e)]-\mu_e
    =
    \beta^n\bigl(Q_0(e)-\mu_e\bigr).
    \label{eq:d4-mean-convergence}
\end{equation}
Consequently,
\begin{equation}
    \lim_{n\to\infty}
    \mathbb E[Q_n(e)]
    =
    \mu_e,
    \label{eq:d4-mean-limit}
\end{equation}
and the expected bias decreases geometrically at rate
\(\beta=1-\alpha\).

Moreover,
\begin{equation}
    \mathbb E\!\left[
        \bigl(Q_n(e)-\mu_e\bigr)^2
    \right]
    =
    \beta^{2n}\bigl(Q_0(e)-\mu_e\bigr)^2
    +
    \frac{\alpha\sigma_e^2}{2-\alpha}
    \bigl(1-\beta^{2n}\bigr).
    \label{eq:d4-stationary-mse}
\end{equation}
Therefore,
\begin{equation}
    \lim_{n\to\infty}
    \mathbb E\!\left[
        \bigl(Q_n(e)-\mu_e\bigr)^2
    \right]
    =
    \frac{\alpha\sigma_e^2}{2-\alpha}.
    \label{eq:d4-stationary-mse-limit}
\end{equation}
\end{theorem}

\begin{proof}
Substituting Eq.~\eqref{eq:d4-stationary-observation} into
Eq.~\eqref{eq:d4-exponential-representation} yields
\begin{align}
    Q_n(e)
    &=
    \beta^nQ_0(e)
    +
    \alpha\sum_{k=1}^{n}
    \beta^{n-k}\bigl[\mu_e+\xi_k(e)\bigr]
    \notag\\
    &=
    \mu_e
    +
    \beta^n\bigl(Q_0(e)-\mu_e\bigr)
    +
    \alpha\sum_{k=1}^{n}\beta^{n-k}\xi_k(e),
    \label{eq:d4-stationary-expansion}
\end{align}
where Eq.~\eqref{eq:d4-convex-weights} was used. Taking expectations and
using Eq.~\eqref{eq:d4-mds} proves
Eq.~\eqref{eq:d4-mean-convergence}.

The martingale-difference property implies that distinct noise terms are
orthogonal. Indeed, for \(k<\ell\),
\begin{align}
    \mathbb E[\xi_k(e)\xi_\ell(e)]
    &=
    \mathbb E\!\left[
        \xi_k(e)
        \mathbb E\!\left[
            \xi_\ell(e)
            \mid
            \mathcal F_{\ell-1}^e
        \right]
    \right]
    =
    0,
    \label{eq:d4-noise-orthogonality}
\end{align}
because \(\xi_k(e)\) is \(\mathcal F_{\ell-1}^e\)-measurable. Squaring
Eq.~\eqref{eq:d4-stationary-expansion}, taking expectations, and applying
Eqs.~\eqref{eq:d4-conditional-variance} and
\eqref{eq:d4-noise-orthogonality} give
\begin{equation}
    \mathbb E\!\left[
        \bigl(Q_n(e)-\mu_e\bigr)^2
    \right]
    =
    \beta^{2n}\bigl(Q_0(e)-\mu_e\bigr)^2
    +
    \alpha^2\sigma_e^2
    \sum_{k=1}^{n}\beta^{2(n-k)}.
    \label{eq:d4-mse-geometric-sum}
\end{equation}
The geometric sum satisfies
\begin{align}
    \alpha^2
    \sum_{k=1}^{n}\beta^{2(n-k)}
    &=
    \alpha^2\frac{1-\beta^{2n}}{1-\beta^2}
    \notag\\
    &=
    \frac{\alpha}{2-\alpha}
    \bigl(1-\beta^{2n}\bigr),
    \label{eq:d4-mse-geometric-evaluation}
\end{align}
because
\begin{equation}
    1-\beta^2
    =
    1-(1-\alpha)^2
    =
    \alpha(2-\alpha).
    \label{eq:d4-beta-square}
\end{equation}
This proves Eqs.~\eqref{eq:d4-stationary-mse} and
\eqref{eq:d4-stationary-mse-limit}.
\end{proof}

\paragraph{Tracking a time-varying transition value.}
The graph, pipeline policy, and implementation pools evolve during training,
so the conditional expected credit associated with a transition may vary
along its activation clock. We therefore consider
\begin{equation}
    Y_n(e)
    =
    \mu_n(e)+\xi_n(e),
    \label{eq:d4-drifting-observation}
\end{equation}
where the deterministic target sequence
\(\{\mu_n(e)\}_{n\ge 0}\) satisfies
\begin{equation}
    |\mu_n(e)-\mu_{n-1}(e)|
    \le
    \Delta_e
    \qquad
    \text{for every }n\ge 1.
    \label{eq:d4-drift-bound}
\end{equation}

\begin{theorem}[Tracking bound for a drifting transition value]
\label{thm:d4-tracking}
Suppose that \(Q_0(e)\) is deterministic and the deterministic target
sequence \(\{\mu_n(e)\}_{n\ge 0}\) satisfies
Eq.~\eqref{eq:d4-drift-bound}. Assume that
\(\{\xi_n(e)\}_{n\ge 1}\) is adapted to
\(\{\mathcal F_n^e\}_{n\ge 0}\) and satisfies
\begin{equation}
    \mathbb E\!\left[
        \xi_n(e)
        \mid
        \mathcal F_{n-1}^e
    \right]
    =
    0,
    \qquad
    \mathbb E\!\left[
        \xi_n(e)^2
        \mid
        \mathcal F_{n-1}^e
    \right]
    \le
    \sigma_e^2.
    \label{eq:d4-drifting-noise}
\end{equation}
Then
\begin{equation}
    \left|
        \mathbb E[Q_n(e)]-\mu_n(e)
    \right|
    \le
    \beta^n|Q_0(e)-\mu_0(e)|
    +
    \frac{\beta\Delta_e}{\alpha}
    \bigl(1-\beta^n\bigr).
    \label{eq:d4-tracking-bias}
\end{equation}
Consequently,
\begin{equation}
    \limsup_{n\to\infty}
    \left|
        \mathbb E[Q_n(e)]-\mu_n(e)
    \right|
    \le
    \frac{(1-\alpha)\Delta_e}{\alpha}.
    \label{eq:d4-tracking-bias-limit}
\end{equation}

The mean-squared tracking error further satisfies
\begin{align}
    \mathbb E\!\left[
        \bigl(Q_n(e)-\mu_n(e)\bigr)^2
    \right]
    &\le
    \left[
        \beta^n|Q_0(e)-\mu_0(e)|
        +
        \frac{\beta\Delta_e}{\alpha}
        \bigl(1-\beta^n\bigr)
    \right]^2
    \notag\\
    &\quad+
    \frac{\alpha\sigma_e^2}{2-\alpha}
    \bigl(1-\beta^{2n}\bigr).
    \label{eq:d4-tracking-mse}
\end{align}
Therefore,
\begin{equation}
    \limsup_{n\to\infty}
    \mathbb E\!\left[
        \bigl(Q_n(e)-\mu_n(e)\bigr)^2
    \right]
    \le
    \left(
        \frac{(1-\alpha)\Delta_e}{\alpha}
    \right)^2
    +
    \frac{\alpha\sigma_e^2}{2-\alpha}.
    \label{eq:d4-tracking-mse-limit}
\end{equation}
\end{theorem}

\begin{proof}
Define
\begin{equation}
    E_n(e)
    :=
    Q_n(e)-\mu_n(e).
    \label{eq:d4-tracking-error}
\end{equation}
Using Eqs.~\eqref{eq:d4-credit-update} and
\eqref{eq:d4-drifting-observation},
\begin{align}
    E_n(e)
    &=
    \beta Q_{n-1}(e)
    +
    \alpha\bigl[\mu_n(e)+\xi_n(e)\bigr]
    -
    \mu_n(e)
    \notag\\
    &=
    \beta E_{n-1}(e)
    +
    \beta\bigl[\mu_{n-1}(e)-\mu_n(e)\bigr]
    +
    \alpha\xi_n(e).
    \label{eq:d4-error-recursion}
\end{align}
Iterating the recursion gives
\begin{align}
    E_n(e)
    &=
    \beta^nE_0(e)
    +
    \beta\sum_{k=1}^{n}
    \beta^{n-k}
    \bigl[\mu_{k-1}(e)-\mu_k(e)\bigr]
    \notag\\
    &\quad+
    \alpha\sum_{k=1}^{n}
    \beta^{n-k}\xi_k(e).
    \label{eq:d4-error-expansion}
\end{align}
Taking expectations, applying the triangle inequality, and using
Eq.~\eqref{eq:d4-drift-bound} yield
\begin{align}
    |\mathbb E[E_n(e)]|
    &\le
    \beta^n|E_0(e)|
    +
    \beta\Delta_e\sum_{k=1}^{n}\beta^{n-k}
    \notag\\
    &=
    \beta^n|E_0(e)|
    +
    \frac{\beta\Delta_e}{\alpha}
    \bigl(1-\beta^n\bigr),
    \label{eq:d4-tracking-bias-proof}
\end{align}
which proves Eqs.~\eqref{eq:d4-tracking-bias} and
\eqref{eq:d4-tracking-bias-limit}.

In Eq.~\eqref{eq:d4-error-expansion}, the first two terms are deterministic
under the assumed target sequence, while the final term has zero mean.
The martingale-difference property implies orthogonality of the noise terms.
Therefore,
\begin{align}
    \mathbb E[E_n(e)^2]
    &\le
    \left[
        \beta^n|E_0(e)|
        +
        \beta\Delta_e\sum_{k=1}^{n}\beta^{n-k}
    \right]^2
    \notag\\
    &\quad+
    \alpha^2\sigma_e^2
    \sum_{k=1}^{n}\beta^{2(n-k)}.
    \label{eq:d4-tracking-mse-proof}
\end{align}
Evaluating the two geometric sums gives
Eq.~\eqref{eq:d4-tracking-mse}. Taking the limit superior proves
Eq.~\eqref{eq:d4-tracking-mse-limit}.
\end{proof}

\begin{remark}[Scope of the stochastic-noise model]
\label{rem:d4-noise-scope}
The martingale-difference assumptions in
Theorems~\ref{thm:d4-stationary} and~\ref{thm:d4-tracking} are idealized
activation-clock conditions. In the implemented algorithm, multiple
transitions, and potentially repeated activations of the same transition,
may receive the same standardized terminal credit from one candidate.
Generation-level standardization may also induce dependence across observed
credits. Proposition~\ref{prop:d4-exponential-boundedness} and
Corollary~\ref{cor:d4-aggregated-boundedness} are deterministic and do not
require the martingale-difference model. The expected-value and
mean-squared-error results in Theorems~\ref{thm:d4-stationary}
and~\ref{thm:d4-tracking} apply under their stated stochastic assumptions;
under correlated noise, additional covariance terms enter the
mean-squared-error expressions.
\end{remark}

\paragraph{Boundedness of aggregated credits.}
Transition credits are subsequently aggregated into implementation-,
operator-, and operator-edge-level statistics. The following result shows
that these aggregations preserve uniform boundedness.

\begin{corollary}[Boundedness of aggregated contextual credits]
\label{cor:d4-aggregated-boundedness}
Suppose that
\begin{equation}
    |Q_t(e)|
    \le
    R_{\max}
    \label{eq:d4-transition-bound-global}
\end{equation}
for every activated transition \(e\). Recall the implementation-level
aggregation used in Algorithm~2:
\begin{equation}
    \overline Q_t(c)
    =
    \frac{
        \sum_{e\in\mathcal T_t(c)}
        n_t(e)Q_t(e)
    }{
        \sum_{e\in\mathcal T_t(c)}
        n_t(e)+\epsilon
    },
    \qquad
    n_t(e)\ge 0,
    \quad
    \epsilon>0.
    \label{eq:d4-recalled-implementation-aggregation}
\end{equation}
Then
\begin{equation}
    |\overline Q_t(c)|
    \le
    R_{\max}.
    \label{eq:d4-implementation-bound}
\end{equation}
The same bound holds for the operator-level average and for the
operator-edge aggregation used in Algorithm~2.
\end{corollary}

\begin{proof}
Let
\begin{equation}
    S_t(c)
    :=
    \sum_{e\in\mathcal T_t(c)}n_t(e),
    \label{eq:d4-total-count}
\end{equation}
and define
\begin{equation}
    w_e
    :=
    \frac{n_t(e)}{S_t(c)+\epsilon},
    \qquad
    w_0
    :=
    \frac{\epsilon}{S_t(c)+\epsilon}.
    \label{eq:d4-aggregation-weights}
\end{equation}
These weights are nonnegative and satisfy
\begin{equation}
    w_0
    +
    \sum_{e\in\mathcal T_t(c)}w_e
    =
    1.
    \label{eq:d4-aggregation-weight-sum}
\end{equation}
Equation~\eqref{eq:d4-recalled-implementation-aggregation} can therefore be
written as
\begin{equation}
    \overline Q_t(c)
    =
    w_0\cdot 0
    +
    \sum_{e\in\mathcal T_t(c)}
    w_eQ_t(e).
    \label{eq:d4-aggregation-convex-combination}
\end{equation}
Thus, \(\overline Q_t(c)\) is a convex combination of zero and uniformly
bounded transition credits, proving
Eq.~\eqref{eq:d4-implementation-bound}. The operator-level credit is an
arithmetic average of bounded implementation credits, while the
operator-edge credit has the same nonnegative weighted-average form.
Therefore, both inherit the same uniform bound.
\end{proof}

\paragraph{Interpretation.}
Proposition~\ref{prop:d4-exponential-boundedness} shows that the
transition-credit update is an exponential moving average of observed
terminal credits and cannot diverge under bounded observations.
Theorem~\ref{thm:d4-stationary} shows that, under a stationary transition
value and the stated martingale-difference model, the expected estimate
converges geometrically to the conditional expected credit. With a constant
learning rate, however, the estimator retains the nonzero limiting
mean-squared error in Eq.~\eqref{eq:d4-stationary-mse-limit}; the theorem
therefore does not imply mean-square or almost-sure convergence of the sample
path. A smaller \(\alpha\) reduces the stationary noise term but slows the
decay of the initial bias.

For a time-varying transition value,
Theorem~\ref{thm:d4-tracking} makes the adaptation--variance tradeoff
explicit. Increasing \(\alpha\) reduces the drift-induced asymptotic bias
term \((1-\alpha)\Delta_e/\alpha\), but increases the stationary noise term
\(\alpha\sigma_e^2/(2-\alpha)\). Finally,
Corollary~\ref{cor:d4-aggregated-boundedness} shows that the implementation-,
operator-, and operator-edge-level credits used for dual-level evolution
inherit the same uniform boundedness as their constituent transition-credit
estimates.

\section{Experimental Setup}
\subsection{Implementation and Reproducibility Details}

The pipeline policy is implemented using Qwen2.5-0.5B-Instruct. An external
LLM backbone is used for directed-graph initialization, operator-implementation
code generation, and graph-topology evolution proposals. Unless otherwise
stated, the external backbone is either DeepSeek-V4-Flash or GPT-5.6 Sol, as
specified in the corresponding experimental group.

For each problem domain, the algorithm-generation process is limited to at most
500 external LLM calls. This budget includes calls for directed-graph
initialization, scheduled graph-edit proposals, and code generation for
\textsc{Add} and \textsc{Replace} actions. Local pipeline-policy inference,
rule-based \textsc{Delete} actions, candidate execution, and deterministic
validation are not counted toward this budget. Training terminates when either
the external-call budget is exhausted or the configured maximum number of
training generations is reached.

Every $I_{\mathrm{val}}=5$ generations, deterministic constrained beam search
with beam width 4 is used to decode candidate pipelines from the current
policy. Each decoded pipeline is instantiated by selecting the highest-credit
implementation for every operator and is then evaluated on the complete
validation split. The checkpoint achieving the highest mean raw validation
reward is retained. Validation does not update the pipeline policy,
transition-credit table, implementation pools, or graph topology. After
training, the best validation checkpoint is restored and evaluated once on the
held-out test split for each independent training run.

Unless otherwise stated, each candidate algorithm is executed under a
wall-clock budget of 90 seconds per instance. Compilation failures, runtime
errors, infeasible solutions, and timeouts receive the failure reward. The main
in-distribution combinatorial-optimization results are aggregated over 10
independent training runs with different random seeds. Experiments using a
different number of runs or evaluation budgets are explicitly identified in
their corresponding sections.

The anonymized repository submitted with the supplementary material includes
the method implementation, operator repository, prompt templates, experiment
configurations, offline smoke tests, and single-run entry points. 

\begin{table}[htbp]

    \centering
    \small
    \setlength{\tabcolsep}{6pt}
    \begin{tabular}{lc}
        \toprule
        \textbf{Hyperparameter} & \textbf{Value} \\
        \midrule
        Base policy model & Qwen2.5-0.5B-Instruct \\
        Pipelines per generation $N$ & 4 \\
        Shared instances per generation $k$ & 3 \\
        Maximum pipeline length & 10 \\
        Credit update rate $\alpha$ & 0.1 \\
        Maximum implementation-pool size $M$ & 10\\
        Sampling temperature & 0.7 \\
        Optimizer & AdamW \\
        Learning rate & \(1\times10^{-4}\) \\
        Maximum gradient norm & 1.0 \\
        Reward normalization constant $\epsilon$ & \(1\times10^{-8}\) \\
        Failure reward $R_{\mathrm{fail}}$ & -5.0 \\
        Number of training generations & 50 \\
        Dataset split seed & 0 \\
        External LLM call budget & 500 \\
Validation interval $I_{\mathrm{val}}$ & 5 generations \\
Validation beam width & 4 \\
Per-instance wall-clock budget & 90 s \\
        \bottomrule
    \end{tabular}
    \caption{Training and evolution hyperparameters used in DGA2D.}
    \label{tab:rl-hyperparameters}
\end{table}

\subsection{Experimental Datasets}

We draw instances from the available benchmark groups in 12 domains, subject
to a fixed size filter determined before training. RCPSP and SALBP-1 retain
their original single-group evaluation scope. The retained collection contains
2,873 unique instances. Table~\ref{tab:experimental-datasets} reports the
included benchmark groups, split sizes aggregated by domain, and problem-size
ranges. Auxiliary index, metadata, solution, and best-known-solution (BKS)
files are not counted as instances. Alternative JSON representations of the
same source instance are likewise excluded. The 34 instances in
\texttt{DIMACS\_subset} also occur in \texttt{DIMACS\_all} and are therefore
counted only once.

To keep the common per-instance evaluation budget feasible, we exclude the
CVRP group \texttt{AGS}, the FSSP group \texttt{tai500\_20}, and the SALBP-1
dataset \texttt{very large data set\_n=1000}. We also exclude the Max-Cut
instances \texttt{G50}, \texttt{G55}, \texttt{G60}, \texttt{G65},
\texttt{G70}, \texttt{G72}, and \texttt{G81}, together with the MIS instances
\texttt{C4000.5}, \texttt{MANN\_a81}, and \texttt{keller6}. For RCPSP, we
retain only PSPLIB \texttt{j30}; for SALBP-1, we retain only the Scholl small
set with 20 tasks. Every instance belonging to the listed retained groups is
included.

Dataset splitting is performed independently within each benchmark group.
Within a group, instances are ordered by a stable identifier, shuffled using
split seed 0, and deterministically divided into disjoint training, validation,
and test subsets according to a 60/20/20 ratio. The resulting subsets are then
aggregated by domain for the reported domain-level experiments. Consequently,
the split sizes in Table~\ref{tab:experimental-datasets} are obtained by
applying the integer-rounding rule within each benchmark group before
aggregation; they are not computed by applying the split ratio directly to the
total number of instances in a domain.

Before each training run, $k$ instances are sampled without replacement from
the active training split and reused across all generations in that run. All
candidate pipelines are therefore evaluated on the same shared instance batch,
which makes their scores directly comparable across generations. Validation is
used only for checkpoint selection and does not update the pipeline policy,
transition-credit table, graph topology, or implementation pools. The test
split remains held out until the best validation checkpoint has been restored.

For reward normalization and gap reporting, each evaluator provides an
associated known optimum, BKS, or designated reference value. Because no
universally accepted BKS is available for the 3D-CLP suite, we use the
reference values produced by the Extreme-Point--First-Fit packing baseline.
These values serve as comparison references rather than proven bounds;
therefore, a negative reported gap is possible when a method outperforms the
reference baseline.

\begin{table}[htbp]
    \centering
    \footnotesize
    \setlength{\tabcolsep}{2.5pt}
    \begin{tabular*}{\linewidth}{@{\extracolsep{\fill}}p{0.10\linewidth}p{0.38\linewidth}rrrp{0.19\linewidth}@{}}
        \toprule
        \textbf{Problem} & \textbf{Included benchmark groups} & \textbf{Train} &
        \textbf{Validation} & \textbf{Test} & \textbf{Problem size} \\
        \midrule
        3D-BPP  & MPV \texttt{n10}, \texttt{n20}, \texttt{n30}, \texttt{n40}, \texttt{n50},
                   \texttt{n60}, \texttt{n100}, \texttt{n150}, \texttt{n200}
                 & 432 & 144 & 144 & 10--200 items \\
        3D-CLP  & Bischoff--Ratcliff \texttt{BR1}--\texttt{BR15}
                 & 90 & 30 & 30 & 95--476 items \\
        CVRP    & A, B, CMT, E, F, Golden, M, P, tai
                 & 81 & 25 & 36 & 12--483 customers \\
        FJSP    & Barnes, Behnke, Brandimarte, Dauzere, Fattahi, Hurink, Kacem
                 & 199 & 66 & 71 & 2--100 jobs, 2--60 machines \\
        FSSP    & Taillard 20-, 50-, and 100-job groups; \texttt{200x10}, \texttt{200x20}
                 & 66 & 22 & 22 & \texttt{20x5}--\texttt{200x20} \\
        JSP     & abz, ft, la, orb, swv, ta, yn
                 & 96 & 33 & 33 & \texttt{6x5}--\texttt{100x20} \\
        Max-Cut & G-set instances with at most 2,000 vertices
                 & 24 & 8 & 9 & 800--2,000 vertices \\
        MIS     & BHOSLIB and DIMACS, deduplicated, with at most 2,000 vertices
                 & 67 & 22 & 23 & 28--2,000 vertices \\
        OSSP    & Taillard \texttt{4x4}, \texttt{5x5}, \texttt{7x7}, \texttt{10x10},
                   \texttt{15x15}, \texttt{20x20}
                 & 36 & 12 & 12 & \texttt{4x4}--\texttt{20x20} \\
        RCPSP   & PSPLIB \texttt{j30}
                 & 288 & 96 & 96 & 30 real activities \\
        SALBP-1 & Scholl small set
                 & 315 & 105 & 105 & 20 tasks \\
        TSP     & TSPLIB collection
                 & 21 & 7 & 7 & 51--299 cities \\
        \bottomrule
    \end{tabular*}
    \caption{Benchmark groups retained for the experiments.  The training, validation, and test
    splits contain 1,715, 570, and 588 instances, respectively, for a total of 2,873 unique
    instances across all 12 domains.}
    \label{tab:experimental-datasets}
\end{table}

\section{Extended Experimental Results}

\subsection{Full Results Across All Benchmark Domains}
\label{app:full_results}

To complement the four representative benchmarks reported in the main text, we further evaluate $\mathrm{DGA}_{2}\mathrm{D}$ on eight additional combinatorial optimization problems: CVRP, JSP, OSSP, SALBP-1, 3D-BPP, FSSP, RCPSP, and Max-Cut. Together with FJSP, 3D-CLP, MIS, and TSP, the complete evaluation covers twelve domains spanning routing, scheduling, packing, project planning, graph optimization, and assembly-line balancing. All LLM-based AHD methods are evaluated using the same prompt template and evaluation budget under both DeepSeek-V4-Flash and GPT-5.6 Sol.

Across the eight additional benchmarks, $\mathrm{DGA}_{2}\mathrm{D}$ consistently achieves the best performance under both LLM backbones. With DeepSeek-V4-Flash, it obtains an average normalized gap of $2.98\%$, compared with $10.66\%$ for the strongest baseline on average. With GPT-5.6 Sol, the average normalized gap further decreases to $2.14\%$, compared with $9.77\%$ for the strongest baseline, corresponding to an absolute improvement of $7.63$ percentage points.

\begin{table*}[t]
\centering
\caption{
Performance of $\text{DGA}_2\text{D}$ compared with baselines across combinatorial optimization benchmarks, grouped by LLM backbones. Obj.: Objective value ($\downarrow$ lower is better, $\uparrow$ higher is better). Gap (\%): normalized gap to the corresponding reference solution. Results are reported as mean $\pm$ standard deviation (10 trials). \textbf{Bold} signifies the best overall result and the best LLM-based result under each backbone.\colorbox{gray!20}{Gray} indicates the best performance under the same LLM API.
}
\label{tab:additional_results}

\small
\setlength{\tabcolsep}{3.5pt}
\renewcommand{\arraystretch}{1.08}

\resizebox{\textwidth}{!}{%
\begin{tabular}{lcccccccc}
\toprule

\multirow{2}{*}{\textbf{Method}}
& \multicolumn{2}{c}{\textbf{CVRP}}
& \multicolumn{2}{c}{\textbf{JSP}}
& \multicolumn{2}{c}{\textbf{OSSP}}
& \multicolumn{2}{c}{\textbf{SALBP-1}} \\

\cmidrule(lr){2-3}
\cmidrule(lr){4-5}
\cmidrule(lr){6-7}
\cmidrule(lr){8-9}

& Obj. ($10^3$) $\downarrow$
& Gap (\%) $\pm$ std $\downarrow$
& Obj. ($10^4$) $\downarrow$
& Gap (\%) $\pm$ std $\downarrow$
& Obj. ($10^3$) $\downarrow$
& Gap (\%) $\pm$ std $\downarrow$
& Obj. ($10^2$) $\downarrow$
& Gap (\%) $\pm$ std $\downarrow$ \\

\midrule

Reference
& 3.388 & --
& 3.451 & --
& 1.210& --
& 6.180 & -- \\

\midrule
\multicolumn{9}{l}{\textit{\textbf{LLM: DeepSeek-V4-Flash}}} \\

EoH
& 4.262 & 27.21$\pm$1.36
& 4.350 & 27.02$\pm$4.12
& 1.216 & 0.49$\pm$0.08
& 6.980 & 9.10$\pm$0.31 \\

MCTS-AHD
& 4.252 & 25.18$\pm$5.84
& 4.511 & 30.71$\pm$5.02
& 1.217 & 0.54$\pm$0.11
& 7.060 & 10.23$\pm$0.64 \\

MEoH
& 4.340 & 28.43$\pm$1.27
& 4.318 & 25.13$\pm$2.21
& 1.213 & 0.24$\pm$0.04
& 6.695 & 8.34$\pm$0.18 \\

ReEvo
& 4.203 & 24.54$\pm$4.30
& 3.797 & 10.01$\pm$2.76
& 1.230 & 1.69$\pm$0.43
& 7.040 & 9.67$\pm$0.52 \\

\rowcolor{gray!20}
$\mathrm{DGA}_2\mathrm{D}$ (Ours)
& 3.640 & 7.42$\pm$1.93
& 3.555 & 3.02$\pm$0.88
& \textbf{1.213} & 0.21$\pm$0.09
& 6.343 & 2.64$\pm$0.46 \\

\midrule
\multicolumn{9}{l}{\textit{\textbf{LLM: GPT-5.6 Sol}}} \\

EoH
& 3.600 & 6.75$\pm$1.12
& 4.292 & 24.36$\pm$1.48
& 1.213& 0.25$\pm$0.04
& 7.010 & 9.07$\pm$0.23 \\

MCTS-AHD
& 4.410 & 30.84$\pm$5.21
& 4.596 & 33.18$\pm$3.74
& 1.213 & 0.27$\pm$0.07
& 7.030 & 9.67$\pm$0.58 \\

MEoH
& 4.233 & 24.92$\pm$1.06
& 4.225 & 22.41$\pm$1.15
& 1.212 & 0.19$\pm$0.03
& 6.682 & 8.12$\pm$0.16 \\

ReEvo
& 4.432 & 31.47$\pm$5.73
& 3.749 & 8.64$\pm$3.12
& 1.216 & 0.50$\pm$0.14
& 6.757 & 9.33$\pm$0.49 \\

\rowcolor{gray!20}
$\mathrm{DGA}_2\mathrm{D}$ (Ours)
& \textbf{3.587} & \textbf{5.86$\pm$1.41}
& \textbf{3.532} & \textbf{2.34$\pm$0.71}
& \textbf{1.211} & \textbf{0.16$\pm$0.09}
& \textbf{6.317} & \textbf{2.21$\pm$0.38} \\

\bottomrule
\end{tabular}%
}
\small
\setlength{\tabcolsep}{3.5pt}
\renewcommand{\arraystretch}{1.08}

\resizebox{\textwidth}{!}{%
\begin{tabular}{lcccccccc}
\toprule

\multirow{2}{*}{\textbf{Method}}
& \multicolumn{2}{c}{\textbf{3D-BPP}}
& \multicolumn{2}{c}{\textbf{FSSP}}
& \multicolumn{2}{c}{\textbf{RCPSP}}
& \multicolumn{2}{c}{\textbf{Max-Cut}} \\
\cmidrule(lr){2-3}
\cmidrule(lr){4-5}
\cmidrule(lr){6-7}
\cmidrule(lr){8-9}

& Obj. ($10^3$) $\downarrow$
& Gap (\%) $\pm$ std $\downarrow$
& Obj. ($10^5$) $\downarrow$
& Gap (\%) $\pm$ std $\downarrow$
& Obj. ($10^3$) $\downarrow$
& Gap (\%) $\pm$ std $\downarrow$
& Obj. ($10^5$) $\uparrow$ & Gap (\%) $\pm$ std $\downarrow$ \\
\midrule

Reference
& 1.917 & --
& 6.353 & --
& 6.167 & --
& 2.901 & -- \\

\midrule
\multicolumn{9}{l}{
    \textit{\textbf{LLM: DeepSeek-V4-Flash}}
} \\

EoH
& 2.007 & 7.90$\pm$1.40
& 7.298 & 19.20$\pm$3.05
& 6.424 & 5.20$\pm$3.70
& 2.745 & 8.67$\pm$0.48 \\

MCTS-AHD
& 2.106 & 12.27$\pm$5.78
& 7.529 & 17.96$\pm$5.72
& 6.539 & 6.82$\pm$4.84
& 2.804 & 4.84$\pm$0.93 \\

MEoH
& 2.020 & 8.73$\pm$1.37
& 7.172 & 14.22$\pm$5.75
& 6.478 & 6.79$\pm$4.24
& 2.766 & 7.27$\pm$1.50 \\

ReEvo
& 1.998 & 6.89$\pm$0.34
& 7.542 & 18.16$\pm$5.94
& 6.334 & 3.38$\pm$1.72
& 2.720 & 10.92$\pm$3.86 \\

\rowcolor{gray!20}
$\mathrm{DGA}_2\mathrm{D}$ (Ours)
& \textbf{1.922} & 0.66$\pm$0.14
& 6.841 & 7.96$\pm$1.77
& \textbf{6.169} & \textbf{0.03$\pm$0.02}
& 2.859 & 1.90$\pm${0.41} \\

\midrule

\multicolumn{9}{l}{
    \textit{\textbf{LLM: GPT-5.6 Sol}}
} \\

EoH
& 2.048 & 9.68$\pm$3.05
& 7.306 & 16.53$\pm$1.00
& 6.500 & 6.59$\pm$3.45
& 2.781 & 6.01$\pm$0.49 \\

MCTS-AHD
& 2.080 & 11.31$\pm$3.90
& 6.897 & 6.80$\pm$6.48
& 6.890 & 12.98$\pm$7.26
& 2.778 & 6.26$\pm$0.76 \\

MEoH
& 2.005 & 7.15$\pm$0.75
& 7.307 & 24.31$\pm$13.84
& 6.429 & 5.16$\pm$2.11
& 2.782 & 6.02$\pm$0.92 \\

ReEvo
& 2.016 & 7.75$\pm$0.75
& 6.742 & 21.40$\pm$21.11
& 6.486 & 6.86$\pm$1.52
& 2.758 & 7.70$\pm$1.13 \\

\rowcolor{gray!20}
$\mathrm{DGA}_2\mathrm{D}$ (Ours)
& 1.931 & \textbf{0.60$\pm$0.21}
& \textbf{6.624} & \textbf{4.25$\pm$2.67}
& 6.180 & 0.22$\pm$0.31
& \textbf{2.872} & \textbf{1.47$\pm$0.23} \\

\bottomrule
\end{tabular}%
}
\end{table*}

The improvements are consistent across problems with substantially different representations, constraints, and objective structures. In particular, $\mathrm{DGA}_{2}\mathrm{D}$ remains effective on both relatively challenging domains, where competing methods exhibit large optimality gaps, and domains where most methods already approach the reference solutions. The comparatively small variations across repeated runs further indicate that the observed improvements are stable rather than being driven by isolated favorable outcomes.

Overall, $\mathrm{DGA}_{2}\mathrm{D}$ ranks first in all $16$ combinations of the eight additional benchmarks and two LLM backbones. Combined with the results on the four benchmarks in the main text, these findings demonstrate that the effectiveness of $\mathrm{DGA}_{2}\mathrm{D}$ generalizes across a broad range of combinatorial optimization domains. Its consistent performance under both DeepSeek-V4-Flash and GPT-5.6 Sol also suggests that the proposed framework is not tied to a specific backbone model or problem representation, but provides a robust and generalizable approach to LLM-based automated heuristic design.

\subsection{Sensitivity Analysis of LLM}
To assess the sensitivity of automated heuristic design to the choice of
language-model backbone, Table~\ref{tab:fjsp_llm_sensitivity} compares four
methods on FJSP across eight models from three LLM families.  The reported
metric is the optimality gap, for which lower values indicate better
solutions.

\begin{table*}[htbp]
    \centering
    \caption{
    FJSP optimality gap (\%) under different LLM backbones.
    Lower values indicate better performance, and the best result
    in each column is shown in bold.
    }
    \label{tab:fjsp_llm_sensitivity}

    \small
    \setlength{\tabcolsep}{7pt}
    \renewcommand{\arraystretch}{1.08}

    \begin{tabular*}{\textwidth}{
        @{\extracolsep{\fill}}
        lcccccccc
    }
        \toprule
        \multirow{2}{*}{\textbf{Method}}
        & \multicolumn{2}{c}{\textbf{DeepSeek}}
        & \multicolumn{3}{c}{\textbf{Claude}}
        & \multicolumn{3}{c}{\textbf{GPT}} \\
        \cmidrule(lr){2-3}
        \cmidrule(lr){4-6}
        \cmidrule(lr){7-9}

        & \textbf{V4-F}
        & \textbf{V4-P}
        & \textbf{4.6 Opus}
        & \textbf{4.7 Opus}
        & \textbf{4.8 Opus}
        & \textbf{5.4-Pro}
        & \textbf{5.5-Pro}
        & \textbf{5.6-Sol} \\
        \midrule

        EoH
        & 28.76 & 26.80
        & 25.90 & 24.80 & 23.70
        & 25.40 & 24.10 & 22.90 \\

        MEoH
        & 26.37 & 26.63
        & 24.60 & 23.50 & 22.40
        & 24.10 & 22.60 & 21.50 \\

        MCTS-AHD
        & 22.20 & 20.10
        & 19.30 & 17.90 & 16.50
        & 18.60 & 16.80 & 15.30 \\

        \textbf{$\mathrm{DGA}_2\mathrm{D}$ (Ours)}
        & \textbf{9.94}
        & \textbf{8.54}
        & \textbf{8.00}
        & \textbf{7.30}
        & \textbf{6.70}
        & \textbf{7.70}
        & \textbf{6.90}
        & \textbf{6.10} \\

        \bottomrule
    \end{tabular*}

    \vspace{2pt}
    \begin{minipage}{\textwidth}
        \footnotesize
        \raggedright
        \textit{Note:}
      V4-F and V4-P denote DeepSeek-V4-Flash and 
DeepSeek-V4-Pro, respectively. 5.4-Pro, 5.5-Pro, and 5.6-Sol denote 
GPT-5.4-Pro, GPT-5.5-Pro, and GPT-5.6 Sol, respectively.
    \end{minipage}
\end{table*}

DGA$_2$D achieves the lowest gap for every backbone. Relative to MCTS-AHD, 
the strongest baseline in each column, it reduces the gap by 
$55.2\%\text{--}60.1\%$. Averaged over all eight backbones, its gap is 
$7.65\%$, compared with $18.34\%$ for MCTS-AHD. Comparing 
DeepSeek-V4-Flash with GPT-5.6 Sol, all four methods obtain lower gaps; 
for DGA$_2$D, the gap decreases from $9.94\%$ to $6.10\%$. These results 
suggest that more capable backbones generally improve solution quality, 
while the advantage of DGA$_2$D remains consistent across LLM families 
rather than depending on a single model.

\subsection{Sensitivity to Evaluation Wall-Clock Budget}
\paragraph{Experimental Setup.}
We evaluate the sensitivity of the generated heuristics to wall-clock budgets
of $0.5$, $1$, $5$, $10$, $30$, $60$, $90$, $180$, and $300$ seconds on
3D-CLP and FJSP. Each budget setting is evaluated through an independent run
from scratch using the same generated heuristic and benchmark instances;
a longer-budget run does not continue from a shorter-budget run. All evaluated
methods terminate within $60$ seconds. Therefore, budgets greater than
$60$ seconds are non-binding and do not provide additional computation after
the corresponding heuristic has terminated.

DGA$_2$D employs stochastic solution construction, and its results may
therefore vary across independent runs even when the nominal wall-clock budget
exceeds its actual termination time. In contrast, the heuristics generated by
EoH, MCTS-AHD, MEoH, and ReEvo are deterministic in this evaluation. Given the
same generated code and problem instance, they return the same solution, and
their results consequently remain unchanged once execution is completed.

\begin{table*}[htbp]
\centering
\caption{
Sensitivity of the normalized gap (\%) to the evaluation 
wall-clock budget on 3D-CLP and FJSP. Lower values are better, and the 
best result under each budget is shown in bold.}
\label{tab:budget_sensitivity}
\label{tab:wall_clock_sensitivity}
\small
\resizebox{\textwidth}{!}{%
\begin{tabular}{lccccccccc}
\toprule
\textbf{Model}
& \multicolumn{9}{c}{\textbf{Evaluation Wall-Clock Budget (s)}} \\
\cmidrule(lr){2-10}
& \textbf{0.5}
& \textbf{1}
& \textbf{5}
& \textbf{10}
& \textbf{30}
& \textbf{60}
& \textbf{90}
& \textbf{180}
& \textbf{300} \\
\midrule

\multicolumn{10}{l}{\textit{\textbf{Benchmark: 3D-CLP}}} \\

EoH
& 2.611
& 1.336
& 0.699
& 0.699
& 0.699
& 0.699
& 0.699
& 0.699
& 0.699 \\

MCTS-AHD
& 2.611
& 1.336
& 0.699
& 0.699
& 0.699
& 0.699
& 0.699
& 0.699
& 0.699 \\

MEoH
& 2.639
& 1.365
& 0.727
& 0.727
& 0.727
& 0.727
& 0.727
& 0.727
& 0.727 \\

ReEvo
& 3.051
& 1.740
& 1.103
& 1.103
& 1.103
& 1.103
& 1.103
& 1.103
& 1.103 \\

$\mathrm{DGA}_{2}\mathrm{D}$ (Ours)
& \textbf{0.393}
& \textbf{0.228}
& \textbf{0.228}
& \textbf{0.166}
& \textbf{0.109}
& \textbf{-0.080}
& \textbf{-0.273}
& \textbf{-0.114}
& \textbf{-0.368} \\

\midrule

\multicolumn{10}{l}{\textit{\textbf{Benchmark: FJSP}}} \\

EoH
& --
& 42.781
& 28.409
& 28.409
& 28.409
& 28.409
& 28.409
& 28.409
& 28.409 \\

MCTS-AHD
& --
& 38.245
& 26.144
& 26.144
& 26.144
& 26.144
& 26.144
& 26.144
& 26.144 \\

MEoH
& --
& 40.113
& 27.313
& 27.313
& 27.313
& 27.313
& 27.313
& 27.313
& 27.313 \\

ReEvo
& --
& 48.719
& 28.750
& 28.750
& 28.750
& 28.750
& 28.750
& 28.750
& 28.750 \\

$\mathrm{DGA}_{2}\mathrm{D}$ (Ours)
& \textbf{18.177}
& \textbf{15.974}
& \textbf{13.312}
& \textbf{12.178}
& \textbf{11.425}
& \textbf{7.791}
& \textbf{7.029}
& \textbf{7.320}
& \textbf{6.369} \\

\bottomrule
\end{tabular}%
}

\vspace{2pt}

\parbox{\textwidth}{%
\footnotesize
\textit{Note:} The symbol \texttt{$-$} indicates that no valid solution
is generated within the specified wall-clock budget.
}

\end{table*}

\paragraph{Sensitivity to the Evaluation Wall-Clock Budget.}
Table~\ref{tab:budget_sensitivity} examines the effect of the evaluation
wall-clock budget on the normalized gap achieved by different methods on
3D-CLP and FJSP. $\mathrm{DGA}_{2}\mathrm{D}$ consistently achieves the
best performance across all evaluated budget settings. Under restrictive
budgets, it produces valid and high-quality solutions substantially faster
than the comparison methods. On FJSP, $\mathrm{DGA}_{2}\mathrm{D}$ generates valid solutions within
both 0.5 and 1 seconds, whereas none of the baselines produces a valid
solution under these budgets.

As the available evaluation budget increases, the performance of
$\mathrm{DGA}_{2}\mathrm{D}$ generally improves up to 60 seconds, by
which point all evaluated methods have completed their computation.
Therefore, the results reported at 90, 180, and 300 seconds do not
represent additional search beyond the 60-second setting. The variations
in the results of $\mathrm{DGA}_{2}\mathrm{D}$ after 60 seconds arise
from stochastic differences across independent runs, whereas the
deterministic constructive baselines return identical results once their
executions have terminated.

Overall, these results demonstrate that $\mathrm{DGA}_{2}\mathrm{D}$
maintains a clear advantage under both restrictive and sufficiently large
evaluation budgets. It can rapidly generate valid and competitive
solutions, while completing its effective computation within 60 seconds
on both benchmarks.
\label{app:wall_clock_sensitivity}

\subsection{Out-of-Distribution Generalization}
\label{app:ood_generalization}
We evaluate the out-of-distribution generalization ability of DGA$_2$D on the Hurink FJSP benchmark using three subsets, namely \textit{edata}, \textit{rdata}, and \textit{vdata}. Each subset contains 66 instances and is characterized by a different flexibility ratio. DGA$_2$D is trained independently on each subset and subsequently evaluated on all three subsets, resulting in nine train--test configurations. The diagonal entries represent in-distribution performance, whereas the off-diagonal entries measure cross-domain generalization under changes in instance flexibility. Each configuration is independently repeated 10 times, and the results are reported as the mean $\pm$ standard deviation.

\begin{table}[htbp]
    \centering
    \caption{In-distribution and cross-domain performance of 
$\mathrm{DGA}_{2}\mathrm{D}$ on the Hurink benchmark. Results are reported 
as the mean $\pm$ standard deviation over 10 independent runs, with lower 
values indicating better performance.}
    \label{tab:id_ood_setting}
    \begin{tabular}{lccc}
        \toprule
        \textbf{Training Domain}
        & \textbf{Test-E}
        & \textbf{Test-R}
        & \textbf{Test-V} \\
        \midrule
        Train-E & 18.97$\pm$0.19 
        & 25.88$\pm$0.15 
        & 25.12$\pm$0.19 \\
        Train-R & 15.92$\pm$0.40 
        & 11.17$\pm$0.20
        & 7.45$\pm$0.35 \\
        Train-V & 18.59$\pm$0.36
        & 20.00$\pm$0.69
        & 19.04$\pm$0.26  \\
        \bottomrule
    \end{tabular}
\end{table}
The reported metric is the normalized optimality gap (\%), for which
lower values indicate better performance. Training on \textit{rdata}
yields the best overall cross-domain performance, achieving the lowest
average gap across the three test domains. The solver trained on
\textit{edata} exhibits a pronounced performance degradation on Test-R
and Test-V, whereas the \textit{vdata}-trained solver shows relatively
stable but less competitive performance across the three domains. These
results suggest that cross-domain generalization is sensitive to the
flexibility distribution of the source domain, with \textit{rdata}
providing the most transferable training distribution among the three
subsets. The strong transfer performance further suggests that
$\mathrm{DGA}_{2}\mathrm{D}$ can discover scheduling strategies that
remain effective under changes in instance flexibility, rather than
merely fitting distribution-specific statistics.
\subsection{Performance at Different Convergence Stages}
To further examine the search efficiency and long-horizon evolutionary behavior of DGA$_2$D, we analyze its convergence from two complementary perspectives. Figure~\ref{fig:convergence_stages} reports the evolution of the best-so-far optimality gap on CVRP across the early, middle, and late stages, while Figure~\ref{fig:budget_convergence} compares DGA$_2$D with representative AHD baselines on FJSP under two computational budgets: output tokens and the number of LLM calls.

\begin{figure}[htbp]
    \centering 
    \includegraphics[width=0.7\columnwidth]{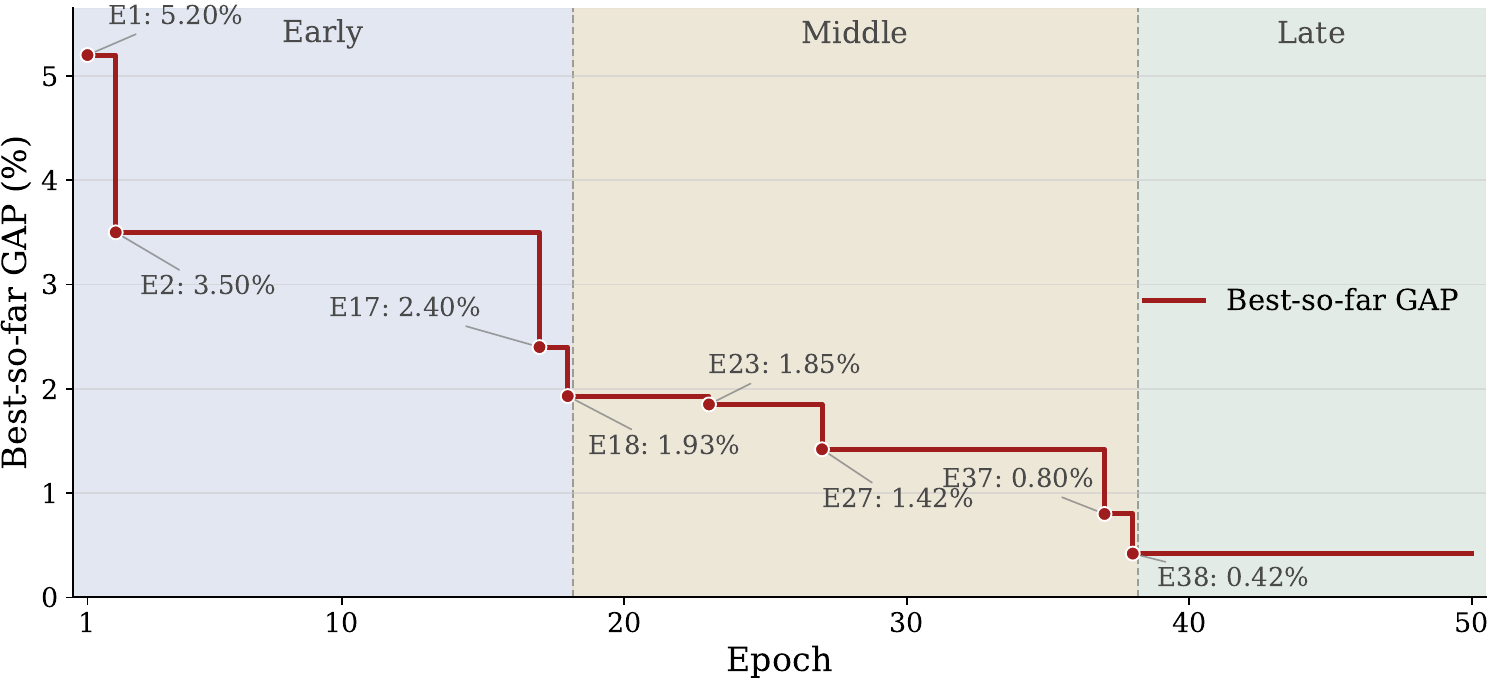}
    \caption{Stage-wise convergence of DGA$_2$D on CVRP, showing the best-so-far optimality gap across the early, middle, and late evolutionary stages.}
\label{fig:convergence_stages}
\end{figure}

As shown in Figure~\ref{fig:convergence_stages}, DGA$_2$D achieves a rapid reduction in the optimality gap during the early stage, decreasing from $5.20\%$ at epoch~1 to $1.93\%$ at epoch~18. This sharp initial improvement suggests that the search can quickly identify promising algorithmic structures. During the middle stage, the gap is progressively reduced from $1.93\%$ to $0.42\%$, with intermediate improvements to $1.85\%$, $1.42\%$, and $0.80\%$ at epochs~23, 27, and 37, respectively, followed by a final decrease to $0.42\%$ at epoch~38. The gap then remains unchanged throughout the late stage, indicating that the search has converged. Overall, the trajectory reflects rapid early-stage exploration, sustained middle-stage refinement, and stable late-stage convergence.
\begin{figure}[htbp]
    \centering 
    \includegraphics[width=0.8\columnwidth]{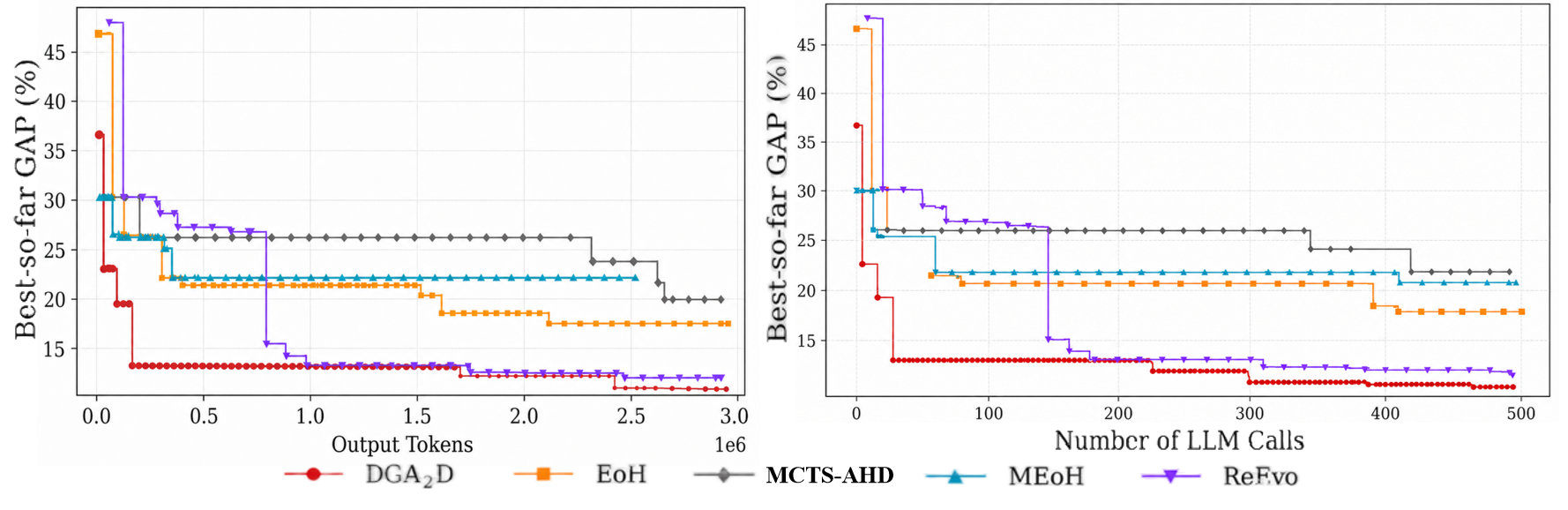}
    \caption{Convergence of the best-so-far optimality gap on FJSP with respect to output tokens (left) and the number of LLM calls (right). Lower values indicate better solution quality. DGA$_2$D exhibits faster convergence and achieves a lower final gap than EoH, MCTS-AHD, MEoH, and ReEvo under both budget measures.}
\label{fig:budget_convergence}
\end{figure}
To complement this stage-wise analysis, Figure~\ref{fig:budget_convergence} presents the full convergence trajectories of different methods with respect to output tokens and LLM calls. DGA$_2$D exhibits a sharp initial decrease under both budget measures and reaches a competitive gap substantially earlier than the baselines. Although ReEvo continues to improve with additional evaluations, it requires a considerably larger budget to approach a similar performance level and still converges to a higher final gap. In contrast, EoH, MCTS-AHD, and MEoH enter relatively high plateaus at earlier stages. These results indicate that DGA$_2$D combines efficient early-stage discovery with sustained long-horizon refinement, yielding both faster convergence and better final solution quality.
\subsection{Resource Consumption Analysis}
Table~\ref{tab:computational_resources} summarizes the wall-clock time and
total token usage required to design heuristics for TSP, FJSP, MIS, and
3D-CLP. Among the evaluated baselines, MCTS-AHD is the only method that
employs an explicit structured search mechanism broadly comparable to that
of DGA$_2$D, making it the most relevant architectural reference for resource
comparison.

\begin{table}[htbp]
\centering
\caption{\textbf{Analysis of computational resources.}
Wall-clock time in hours and total token usage in millions for
heuristic design across four optimization problems.}
\label{tab:computational_resources}
\small
\setlength{\tabcolsep}{8pt}
\renewcommand{\arraystretch}{1.05}
\begin{tabular}{llcccc}
\toprule
\textbf{Method} & \textbf{Metric}
& \textbf{TSP}
& \textbf{FJSP}
& \textbf{MIS}
& \textbf{3D-CLP} \\
\midrule

\multirow{2}{*}{EoH}
& Time (h)    & 0.35 & 0.65 & 0.63 & 0.32 \\
& Tokens (M)  & 1.04 & 2.76 & 1.16 & 0.82 \\
\midrule

\multirow{2}{*}{MCTS-AHD}
& Time (h)    & 2.08 & 3.88 & 4.65 & 3.24 \\
& Tokens (M)  & 1.26 & 1.13 & 1.10 & 1.20 \\
\midrule

\multirow{2}{*}{ReEvo}
& Time (h)    & 0.92 & 0.97 & 1.01 & 0.7 \\
& Tokens (M)  & 2.10 & 3.33 & 1.35 & 1.36 \\
\midrule

\multirow{2}{*}{MEoH}
& Time (h)    & 0.27 & 0.67 & 0.39 & 0.49 \\
& Tokens (M)  & 1.03 & 2.14 & 1.25 & 1.03 \\
\midrule

\multirow{2}{*}{\textbf{DGA2D (Ours)}}
& Time (h)    & 1.43 & 2.90 & 3.38 & 2.35 \\
& Tokens (M)  & 0.56 & 1.74 & 0.81 & 0.45 \\
\bottomrule
\end{tabular}
\end{table}

Compared with MCTS-AHD, DGA$_2$D reduces the total wall-clock time from
$13.85$ to $10.06$ hours and the total token consumption from $4.69$M to
$3.56$M, corresponding to reductions of approximately $27.4\%$ and $24.1\%$,
respectively. DGA$_2$D requires less time on all four problems and consumes
fewer tokens on TSP, MIS, and 3D-CLP, although MCTS-AHD uses fewer tokens on
FJSP.

The relatively low token consumption of DGA$_2$D can be attributed to its
localized optimization strategy. Rather than repeatedly regenerating or
revising an entire heuristic program, DGA$_2$D selectively updates individual
operator implementations and local structural transitions while preserving
the remaining components. In contrast, the other evolutionary AHD frameworks
generally perform search at the whole-heuristic or population level, which
requires broader program synthesis and revision and may consequently incur
more extensive LLM interactions. Localized refinement therefore allows
DGA$_2$D to reuse validated components and avoid unnecessary full-program
regeneration.

Compared with EoH, ReEvo, and MEoH, DGA$_2$D incurs a higher wall-clock cost
because of the additional computation required for structured policy search,
credit assignment, and local program maintenance. Nevertheless, it achieves
the lowest aggregate token consumption among all evaluated methods. Overall,
DGA$_2$D trades additional local computation for reduced LLM usage, providing
a favorable balance between structured search effectiveness and computational
resource consumption.
\section{Extension to Continuous Numerical Tasks}
\label{app:continuous_optimization}

To examine whether the proposed framework is applicable beyond
combinatorial optimization, we further extend
$\mathrm{DGA}_{2}\mathrm{D}$ to three classes of continuous numerical
tasks: ordinary differential equation integration, feedback control,
and numerical root-finding. These tasks differ substantially from the
combinatorial optimization benchmarks in their solution
representations, operator interfaces, objectives, and numerical
stability requirements.

For each benchmark problem within a task class,
$\mathrm{DGA}_{2}\mathrm{D}$ is instantiated with the corresponding task
interface and operator repository, and a separate algorithm is evolved
specifically for that problem. Therefore, these
experiments are intended to evaluate the cross-domain applicability
of the proposed automated design framework, rather than zero-shot
transfer of a single evolved algorithm across different domains.
Unless otherwise specified, higher scores indicate better
performance, while large negative values denote unsuccessful or
invalid solutions.

\subsection{Discovery of Numerical ODE Solvers}
\label{app:solving_ode}

To evaluate the capability of DGA$_2$D in automatically discovering numerical
algorithms for continuous dynamical systems, we construct a comprehensive
benchmark consisting of seven differential-equation systems with diverse
numerical properties, including non-stiff, stiff, chaotic, and
high-dimensional dynamics.

\paragraph{Experimental Setup.}
We evaluate the methods on seven differential-equation systems with diverse
numerical properties: Lotka--Volterra, Van der Pol, Robertson, Lorenz, a stiff
linear system, a two-dimensional heat equation, and a convection--diffusion
equation. The benchmark covers non-stiff, nonlinear, stiff, and chaotic
dynamics, as well as high-dimensional systems obtained from the spatial
discretization of partial differential equations. For each system, one
instance is used for algorithm design and four instances are used for testing.
The nominal step size is \(h=0.05\), with
integration intervals \([0,10]\) for the standard ODE systems and \([0,2]\)
for the PDE-derived systems. The stiff linear system has dimension \(15\),
while Heat2D and ConvDiff use \(10\times10\) and \(100\)-point spatial
discretizations, respectively. High-accuracy terminal states are used as
reference solutions. Following the same evaluation protocol, the performance
score is defined as
$S=-\log_{10}(\mathrm{Error})-\lambda\log_{10}(\mathrm{Cost})$, where
$\mathrm{Error}$ measures the numerical integration error and
$\mathrm{Cost}$ is the number of right-hand-side function evaluations. We set
$\lambda=0.1$ for the standard ODE systems and $\lambda=0.05$ for the two
high-dimensional PDE-derived systems. Higher scores indicate better
performance, and a score of $-20$ is assigned when a solver diverges or fails
to produce a valid solution. The average score is computed over all seven
systems, including failure penalties.
\paragraph{Baselines.}
We compare DGA$_2$D and EoH against five standard numerical integrators covering
different solver families. Euler and fourth-order Runge--Kutta (RK4) represent
classical explicit fixed-step methods. DOP853 is a high-order explicit adaptive
solver, while LSODA automatically switches between non-stiff and stiff
integration schemes. Radau is an implicit adaptive solver designed for stiff
differential equations. Together, these baselines provide comparisons against
both computationally inexpensive explicit methods and robust adaptive solvers
commonly used for challenging stiff systems.
\paragraph{Results Analysis.}
Table~\ref{tab:ode_results} shows that DGA$_2$D achieves the strongest overall
performance across the seven-system benchmark. It obtains the highest score on
five systems and ties with RK4 on Lotka--Volterra. On the remaining Heat2D
problem, DGA$_2$D achieves a score of $14.81$, only slightly below the best
score of $14.89$ obtained by EoH. Its average score reaches $13.79$, substantially
exceeding Radau's $7.56$, the best result among the conventional numerical
solvers. In contrast, the explicit fixed-step methods frequently receive
failure penalties on stiff, chaotic, and spatially discretized systems,
whereas DOP853, LSODA, and Radau exhibit greater robustness. EoH displays more
task-dependent performance: it performs particularly well on Heat2D and
remains competitive on the stiff linear and convection--diffusion systems, but
its performance decreases considerably on Van der Pol, Robertson, and Lorenz.
These results demonstrate DGA$_2$D's ability to discover effective integration
strategies across substantially different numerical regimes.

\begin{table}[htbp]
    \centering
    \caption{
    Performance on ordinary differential equation solving tasks.
    Higher values indicate better performance and $-20$ indicates divergence.
    Failure penalties are included when computing the average.
    }
    \label{tab:ode_results}
    \resizebox{\textwidth}{!}{
    \begin{tabular}{lcccccccc}
        \toprule
        \textbf{Method}
        & \textbf{Lotka--Volterra}
        & \textbf{Van der Pol}
        & \textbf{Robertson}
        & \textbf{Lorenz}
        & \textbf{Linear Stiff}
        & \textbf{Heat2D}
        & \textbf{ConvDiff}
        & \textbf{Average} \\
        \midrule
        Euler
        & 0.76 & -10.00 & -20.00 & -20.00
        & -20.00 & -20.00 & -20.00 & -15.61 \\

        RK4
        & \textbf{6.98} & 3.21 & -20.00 & -20.00
        & -20.00 & -20.00 & -20.00 & -12.83 \\

        DOP853
        & 5.31 & 4.85 & 7.03 & 2.33
        & 8.48 & 9.48 & 6.72 & 6.31 \\

        LSODA
        & 3.95 & 3.40 & 5.69 & 1.33
        & 7.78 & 10.33 & 6.01 & 5.50 \\

        Radau
        & 6.26 & 6.59 & 7.25 & 3.38
        & 9.91 & 12.18 & 7.36 & 7.56 \\

        \midrule
        DGA$_2$D
        & \textbf{6.98} & \textbf{14.74}
        & \textbf{14.79} & \textbf{14.72}
        & \textbf{15.65} & 14.81
        & \textbf{14.84} & \textbf{13.79} \\

        EoH
        & 5.49 & 3.21 & 1.38 & 0.75
        & 8.30 & \textbf{14.89}
        & 8.64 & 6.09 \\
        \bottomrule
    \end{tabular}
    }
\end{table}

\subsection{Discovering Feedback Control Policies}
\label{app:feedback_control}

To evaluate DGA$_2$D's capability in discovering feedback-control policies,
we constructed a benchmark covering dynamical systems with
distinct control characteristics and stability requirements.

\paragraph{Experimental Setup.}
We considered three representative feedback-control tasks: a double integrator,
a pendulum swing-up system, and a cart--pole system. These environments cover
linear motion control, nonlinear stabilization, and coupled underactuated
dynamics, respectively. For each system, one instance is used for algorithm
design and \(30\) instances are used for testing. Trajectories are simulated
with \(\Delta t=0.02\) over a horizon of \(T=100\). Control inputs are bounded
by \([-5,5]\), \([-8,8]\), and \([-20,20]\) for the double integrator,
pendulum, and cart--pole systems, respectively. The test instances vary the
initial states, with additional physical-parameter perturbations of up to
\(20\%\) for the pendulum and cart--pole systems. To jointly evaluate tracking
performance and control effort, we use the cumulative control objective
\[
J=
\sum_t
\left[
\lVert x_t-x^\ast\rVert_Q^2
+
0.01\lVert u_t\rVert_2^2
\right],
\qquad
S=-\log_{10}(1+J),
\]
where \(x_t\) and \(u_t\) denote the system state and control input at time step
\(t\), respectively, \(x^\ast\) is the target state, \(Q\) is the identity
weighting matrix, and the coefficient \(0.01\) balances tracking accuracy
against control effort. Higher scores indicate better performance, with values
closer to zero corresponding to lower cumulative control cost. Invalid or
divergent trajectories receive a fixed failure penalty.

\paragraph{Baselines.}
We compared DGA$_2$D and EoH against a diverse collection of conventional
feedback-control strategies. PID represents a widely used error-feedback
controller that does not require an explicit dynamics model. Linear quadratic
regulation (LQR) provides a model-based optimal controller for approximately
linear systems, while sliding mode control (SMC) is designed to offer robustness
against nonlinearities and disturbances. Bang--bang control represents a simple
switching strategy based on maximum-magnitude control actions. We additionally
included energy shaping, which stabilizes nonlinear mechanical systems by
modifying their effective energy landscape. Together, these baselines cover
classical linear, nonlinear, switching, and energy-based control paradigms.

\begin{table}[htbp]
    \centering
    \caption{
    Performance of feedback control policies.
    Higher values, i.e., values closer to zero, indicate better performance.
    The best result in each column is shown in bold.
    }
    \label{tab:control_results}
    \begin{tabular}{lcccc}
        \toprule
        \textbf{Method}
        & \textbf{Double Integrator}
        & \textbf{Pendulum}
        & \textbf{Cart-Pole}
        & \textbf{Average} \\
        \midrule
        PID
        & -2.37
        & -3.99
        & -11.59
        & -5.98 \\

        LQR
        & -2.15
        & -2.83
        & -11.66
        & -5.55 \\

        SMC
        & -3.16
        & -3.46
        & -4.73
        & -3.78 \\

        Bang--Bang
        & -3.21
        & -3.58
        & -10.27
        & -5.69 \\

        Energy Shaping
        & -3.77
        & -3.21
        & -4.71
        & -3.90 \\

        \midrule
        DGA$_2$D
        & \textbf{-2.13}
        & \textbf{-2.74}
        & \textbf{-4.69}
        & \textbf{-3.19} \\

        EoH
        & -2.25
        & -2.83
        & -4.88
        & -3.32 \\
        \bottomrule
    \end{tabular}
\end{table}

\paragraph{Results Analysis.}
The comparative results in Table~\ref{tab:control_results} demonstrate that
DGA$_2$D achieves consistently strong performance across all three control
systems. It obtains the best score on the double-integrator, pendulum, and
cart--pole tasks, with scores of $-2.13$, $-2.74$, and $-4.69$, respectively.
Its average score is $-3.19$, outperforming EoH at $-3.32$ and SMC at $-3.78$,
the strongest conventional baseline on average. The performance margin is
relatively small on the double integrator, where DGA$_2$D slightly exceeds LQR,
and on cart--pole, where it remains close to energy shaping and SMC. Nevertheless,
DGA$_2$D is the only method that ranks first on every individual task.
Conventional controllers exhibit more pronounced task-specific behavior: LQR
performs well on the double integrator and pendulum but degrades substantially
on cart--pole, whereas SMC and energy shaping are more competitive on cart--pole
but weaker on the other systems. EoH provides comparatively stable performance
and ranks second overall, although it does not achieve the best result on any
single task. These findings demonstrate that DGA$_2$D can discover competitive 
feedback-control policies for systems with substantially different dynamical characteristics.

\subsection{Discovery of Numerical Root-Finding Algorithms}
\label{app:root_finding}

To evaluate DGA$_2$D's capability in discovering numerical algorithms for
continuous nonlinear equations, we constructed a comprehensive benchmark
covering root-finding problems with diverse conditioning properties, function
geometries, and input dimensions.

\paragraph{Experimental Setup.}
The benchmark consists of four representative problem families: a
\emph{Smooth Polynomial} problem with a well-behaved differentiable landscape,
an \emph{Ill-Conditioned} problem characterized by weak or vanishing
derivatives, an \emph{Oscillatory} problem containing multiple local
variations and candidate roots, and a \emph{Two-Dimensional Nonlinear System}
requiring a vector-valued solution. The four families are defined as
\[
\begin{aligned}
f_{\mathrm{smooth}}(x)
&=(x-r)(a_0+a_1x+a_2x^2),\\
f_{\mathrm{ill}}(x)
&=(x-r)^m(a_0+a_1x+a_2x^2),\\
f_{\mathrm{osc}}(x)
&=\sin\!\bigl(\omega(x-r)\bigr)+\alpha(x-r),\\
F_{\mathrm{2D}}(x)
&=
\begin{bmatrix}
x_1^2+x_2-r_1\\
x_1+x_2^2-r_2
\end{bmatrix},
\qquad x=(x_1,x_2)^\top.
\end{aligned}
\]
For the smooth and ill-conditioned families, \(r\in[-5,5]\) and
\(a_j\in[-2,2]\), with \(m\in\{2,3,4\}\) for the latter. For the oscillatory
family, \(r\in[-3,3]\), \(\omega\in[2,10]\), and
\(\alpha\in[0.01,0.5]\); for the two-dimensional family,
\(r_1,r_2\in[-1.5,1.5]\). Initial guesses are generated by applying random
offsets to the known reference roots. For each family, one instance is used
for algorithm design and four instances are used for testing. All methods use
a residual tolerance of \(10^{-8}\). These tasks jointly evaluate convergence
reliability, numerical efficiency, robustness to poor conditioning, and
generalization from scalar equations to coupled nonlinear systems. We define
the following efficiency-aware score:
\[
S=1-0.05\ln(1+N_{\mathrm{fev}})-0.02\ln(1+N_{\mathrm{iter}}),
\]
where \(N_{\mathrm{fev}}\) and \(N_{\mathrm{iter}}\) denote the numbers of
function evaluations and solver iterations, respectively. Thus, the score
rewards methods that require fewer function evaluations and solver iterations.
Higher values indicate better performance.
Unsuccessful or inapplicable solutions are reported as ``--'' in
Table~\ref{tab:root_results}.
\paragraph{Baselines.}
We compare DGA$_2$D and EoH against a broad collection of classical numerical
root-finding methods. For scalar equations, we include robust bracketing
methods such as Bisection and Brent's method; derivative-based methods such as
Newton's method and Halley's method; and derivative-free or accelerated
iterations including the Secant and Steffensen methods. For the
two-dimensional nonlinear system, we consider Damped Newton, Broyden's method,
Anderson acceleration, and SciPy's hybrid solver. These baselines cover
bracketing, derivative-based, quasi-Newton, fixed-point acceleration, and
hybrid solution strategies, thereby providing a comprehensive comparison
across both scalar and multivariate root-finding settings.

\begin{table*}[htbp]
    \centering
    \caption{
    Results on numerical root-finding tasks. Higher is better. ``--'' indicates an unsuccessful or inapplicable solution. The best result in each column is highlighted in bold.
    }
    \label{tab:root_results}
    \begin{tabular}{lcccc}
        \toprule
        \textbf{Method}
        & \textbf{Smooth Polynomial}
        & \textbf{Ill-Conditioned}
        & \textbf{Oscillatory}
        & \textbf{Nonlinear System 2D} \\
        \midrule
        Bisection
        & 0.73
        & 0.04
        & 0.73
        & -- \\

        Brent
        & 0.80
        & 0.10
        & 0.80
        & -- \\

        Newton
        & 0.82
        & 0.77
        & 0.83
        & -- \\

        Secant
        & 0.83
        & 0.78
        & \textbf{0.86}
        & -- \\

        Steffensen
        & 0.81
        & 0.78
        & 0.78
        & -- \\

        Halley
        & 0.83
        & 0.78
        & 0.83
        & -- \\

        Newton (Damped)
        & --
        & --
        & --
        & 0.83 \\

        Broyden
        & --
        & --
        & --
        & 0.84 \\

        Anderson
        & --
        & --
        & --
        & 0.83 \\

        SciPy (hybr)
        & --
        & --
        & --
        & 0.82 \\

        \midrule
        DGA$_2$D
        & \textbf{0.86}
        & \textbf{0.85}
        & 0.84
        & \textbf{0.88} \\

        EoH
        & --
        & -0.58
        & -1.29
        & \textbf{0.88} \\
        \bottomrule
    \end{tabular}
\end{table*}

\paragraph{Results Analysis.}
Table~\ref{tab:root_results} demonstrates that DGA$_2$D achieves the most
consistent performance across all four root-finding settings. It obtains scores
of $0.86$, $0.85$, $0.84$, and $0.88$ on the Smooth Polynomial,
Ill-Conditioned, Oscillatory, and Two-Dimensional Nonlinear System tasks,
respectively. DGA$_2$D ranks first on the Smooth Polynomial and
Ill-Conditioned problems, ranks second only to the Secant method on the
Oscillatory problem, and ties with EoH for the best result on the
two-dimensional system.

The conventional solvers exhibit a clear separation between scalar and
multivariate applicability. Scalar methods perform strongly on the first three
problems, with derivative-based approaches generally offering high efficiency,
but they are inapplicable to the coupled two-dimensional system under the
unified evaluation setting. Conversely, Damped Newton, Broyden, Anderson
acceleration, and the hybrid solver successfully address the multivariate
problem but do not transfer to the scalar benchmark tasks. EoH also displays
strong task dependence: it achieves the best score on the two-dimensional
system but fails on the Smooth Polynomial problem and obtains negative scores
on the Ill-Conditioned and Oscillatory tasks. In contrast, DGA$_2$D consistently
discovers effective root-finding algorithms across variations in conditioning,
function geometry, and dimensionality.

Overall, these experiments demonstrate the cross-domain applicability of
DGA$_2$D as an automated algorithm-design framework. Although a separate
algorithm is evolved for each benchmark problem within each task class,
DGA$_2$D consistently discovers competitive task-specific methods across 
ODE integration, feedback control, and numerical root finding. These results 
therefore support framework-level transferability across continuous numerical 
domains, rather than zero-shot transfer of a single evolved algorithm.

\end{document}